\documentclass{article}

\usepackage[accepted]{icml2016}
\usepackage{amsfonts,amsmath,amssymb,amscd,dsfont,mathrsfs}
\usepackage{etoolbox}
\usepackage{array}

\usepackage{times}
\usepackage{graphicx} 
\usepackage[labelformat=simple]{subcaption}

\usepackage{natbib}

\usepackage{algorithm}
\usepackage{algorithmic}
\usepackage{hyperref}

\usepackage{bm}

\usepackage{siunitx}
\usepackage{xargs}                      
\usepackage[pdftex,dvipsnames]{xcolor}  
\usepackage{xspace}
\usepackage{enumitem}

\icmltitlerunning{Discrete Distribution Estimation under Local Privacy}

\newcommand{\BlackBox}{\rule{1.5ex}{1.5ex}}  
\newenvironment{proof}{\par\noindent{\bf Proof\ }}{\hfill\BlackBox\\[2mm]}

\newtheorem{theorem}{Theorem}

\newtheorem{proposition}[theorem]{Proposition}

\DeclareMathAlphabet{\mathpzc}{OT1}{pzc}{m}{it}

\DeclareMathOperator*{\argmin}{argmin}

\DeclareMathOperator*{\hash}{\textsc{hash}}
\DeclareMathOperator*{\bloom}{\textsc{bloom}}

\newcommand*{\unity}{\textrm{{\usefont{U}{fplmbb}{m}{n}1}}}
\newcommand*{\one}{\mathbf{1}}

\newcommand{\X}{X}
\newcommand{\x}{x}
\newcommand{\Y}{Y}
\newcommand{\y}{y}
\newcommand{\kk}{k}
\newcommand{\n}{n}

\newcommand{\s}{S}

\newcommand{\naturals}{{\mathbb N}}

\newcommand{\simplex}{{\mathbb S}}
\newcommand{\cC}{{\cal C}}
\newcommand{\cX}{{\cal \X}}
\newcommand{\cY}{{\cal \Y}}
\newcommand{\cS}{{\cal \s}}

\newcommand{\prob}{\mathbb P}

\newcommand{\matr}[1]{\bm{#1}}  

\newcommand{\cScalar}{c}  
\newcommand{\hScalar}{h}  
\newcommand{\iScalar}{i}  
\newcommand{\jScalar}{j}  
\newcommand{\kScalar}{k}  
\newcommand{\nScalar}{n}  
\newcommand{\pScalar}{p}  
\newcommand{\PScalar}{P}  
\newcommand{\sScalar}{s}  
\newcommand{\xScalar}{x}  
\newcommand{\yScalar}{y}  

\newcommand{\CScalar}{C}  
\newcommand{\SScalar}{S}  
\newcommand{\PTScalar}{P_t}  

\newcommand{\cSet}{\cC}   
\newcommand{\sSet}{\cS}   
\newcommand{\xSet}{\cX}   

\newcommand{\mVector}{\matr{m}}   
\newcommand{\pVector}{\matr{p}}   
\newcommand{\sVector}{\matr{s}}   
\newcommand{\xVector}{\matr{x}}   

\newcommand{\hMatrix}{\matr{H}}  
\newcommand{\qMatrix}{\matr{Q}}  
\newcommand{\wMatrix}{\matr{W}}  

\newcommand{\qKrrMatrix}{\qMatrix_\text{KRR}}  
\newcommand{\qOrrMatrix}{\qMatrix_\text{ORR}}  
\newcommand{\qWrrMatrix}{\qMatrix_\text{WRR}}  

\newcommand{\KRR}{$\kScalar$-RR\xspace}
\newcommand{\ORR}{O-RR\xspace}
\newcommand{\WRR}{W-RR\xspace}
\newcommand{\RAPPOR}{\textsc{Rappor}\xspace}
\newcommand{\KRAPPOR}{$\kScalar$-\textsc{Rappor}\xspace}
\newcommand{\ORAPPOR}{O-\textsc{Rappor}\xspace}
\newcommand{\LL}{\ell}
\newcommand{\PP}{\pVector}
\newcommand{\pp}{p}
\newcommand{\MM}{\mVector}

\newcommand{\Q}{\qMatrix}

\begin{document}

\twocolumn[
\icmltitle{Discrete Distribution Estimation under Local Privacy}

\icmlauthor{Peter Kairouz $\ast \dagger$}{kairouz2@illinois.edu}
\icmlauthor{Keith Bonawitz $\ast$}{bonawitz@google.com}
\icmlauthor{Daniel Ramage $\ast$}{dramage@google.com}
\icmladdress{$\ast$ Google,
            1600 Amphitheatre Parkway, Mountain View, CA 94043, \\
            $\dagger$ University of Illinois, Urbana-Champaign,
            1308 W Main St, Urbana, IL 61801}

\icmlkeywords{local differential privacy, privacy-preserving machine learning algorithms, statistics, distribution learning}

\vskip 0.3in
]

\begin{abstract}

The collection and analysis of user data drives improvements in the app and web ecosystems, but comes with risks to privacy. This paper examines discrete distribution estimation under local privacy, a setting wherein service providers can learn the distribution of a categorical statistic of interest without collecting the underlying data. We present new mechanisms, including hashed $k$-ary Randomized Response (\KRR), that empirically meet or exceed the utility of existing mechanisms at all privacy levels. New theoretical results demonstrate the order-optimality of \KRR and the existing \RAPPOR mechanism at different privacy regimes.

\end{abstract}

%
%
%
%
%
%
%
%
%
%
%

\section{Introduction}
\label{sec:intro}
Software and service providers increasingly see the collection and analysis of user data as key to improving their services. Datasets of user interactions give insight to analysts and provide training data for machine learning models. But the collection of these datasets comes with risk---can the service provider keep the data secure from unauthorized access? Misuse of data can violate the privacy of users and substantially tarnish the provider's reputation.

One way to minimize risk is to store less data: providers can methodically consider what data to collect and how long to store it.
However, even a carefully processed dataset can compromise user privacy.
In a now famous study, \cite{narayanan2008robust} showed how to de-anonymize watch histories released in the Netflix Prize, a public recommender system competition.
While most providers do not intentionally release anonymized datasets, security breaches can mean that even internal, anonymized datasets have the potential to become privacy problems.

Fortunately, mathematical formulations exist that can give the benefits of population-level statistics without the collection of raw data. Local differential privacy \cite{duchi2013locala, duchi2013local} is one such formulation, requiring each device (or session for a cloud service) to share only a noised version of its raw data with the service provider's logging mechanism. No matter what computation is done to the noised output of a locally differentially private mechanism, any attempt to impute properties of a single record will have a significant probability of error. But not all differentially private mechanisms are equal when it comes to utility: some mechanisms have better accuracy than others for a given analysis, amount of data, and desired privacy level.

\textbf{Private distribution estimation.} This paper investigates the fundamental problem of discrete distribution estimation under local differential privacy. We focus on discrete distribution estimation because it enables a variety of useful capabilities, including usage statistics breakdowns and count-based machine learning models, e.g. naive Bayes \cite{mccallum1998comparison}. We consider empirical, maximum likelihood, and minimax distribution estimation, and study the price of local differential privacy under a variety of loss functions and privacy regimes. In particular, we compare the performance of two recent local privacy mechanisms: (a) the Randomized Aggregatable Privacy-Preserving Ordinal Response (\RAPPOR) \cite{erlingsson2014rappor}, and (b) the $\kScalar$-ary Randomized Response (\KRR) \cite{kairouz2014extremal} from a theoretical and empirical perspective.

\textbf{Our contributions} are:

\vspace{-1em}
\begin{enumerate}[leftmargin=*]
\item For binary alphabets, we prove that Warner's randomized response model \cite{warner1965randomized} is globally optimal for any loss function and any privacy level (Section~\ref{sec:bin_alphabets}).
\item For $\kScalar$-ary alphabets, we show that \RAPPOR is order optimal in the high privacy regime and strictly sub-optimal in the low privacy regime for $\ell_1$ and $\ell_2$ losses using an empirical estimator.  Conversely, \KRR is order optimal in the low privacy regime and strictly sub-optimal in the high privacy regime (Section~\ref{sec:krr}).
\item Large scale simulations show that the optimal decoding algorithm for both \KRR and \RAPPOR depends on the shape of the true underlying distribution. For skewed distributions, the \textit{projected estimator} (introduced here) offers the best utility across a wide variety of privacy levels and sample sizes (Section~\ref{sec:sim_results}).
\item For open alphabets in which the set of input symbols is not enumerable \textit{a priori} we construct the \ORR mechanism (an extension to \KRR using hash functions and cohorts) and provide empirical evidence that the performance of \ORR meets or exceeds that of \RAPPOR over a wide range of privacy settings (Section~\ref{sec:open_alphabets}).
\item We apply the \ORR mechanism to closed $\kScalar$-ary alphabets, replacing hash functions with permutations.  We provide empirical evidence that the performance of \ORR meets or exceeds that of \KRR and \RAPPOR in both low and high privacy regimes (Section~\ref{sec:krr_with_hashing}).
\end{enumerate}

\textbf{Related work.} There is a rich literature on distribution estimation under local privacy \cite{chan2012differentially, hsu2012distributed,bassily2015local}, of which several works are particularly relevant herein.  \cite{warner1965randomized} was the first to study the local privacy setting and propose the randomized response model that will be detailed in Section~\ref{sec:bin_alphabets}.  \cite{kairouz2014extremal} introduced \KRR and showed that it is optimal in the low privacy regime for a rich class of information theoretic utility functions. \KRR will be extended to open alphabets in Section \ref{sec:orr}. \cite{duchi2013locala,duchi2013local} was the first to apply differential privacy to the local setting, to study the fundamental trade-off between privacy and minimax distribution estimation in the high privacy regime, and to introduce the core of \KRAPPOR.  \cite{erlingsson2014rappor} proposed \RAPPOR, systematically addressing a variety of practical issues for private distribution estimation, including robustness to attackers with access to multiple reports over time, and estimating distributions over open alphabets. \RAPPOR has been deployed in the Chrome browser to allow Google to privately monitor the impact of malware on homepage settings.  \RAPPOR will be investigated in Sections~\ref{sec:rappor} and~\ref{sec:orappor}.

Private distribution estimation also appears in the global privacy context where a trusted service provider releases randomized data (e.g., NIH releasing medical records) to protect sensitive user information \cite{Dwo06,DMNS06,DL09,dwork2008differential, diakonikolas2015differentially, blocki2016differentially}.


\section{Preliminaries}



\subsection{Local differential privacy}
\label{sec:local_dp}
%

Let $\X$ be a private source of information defined on a discrete, finite input alphabet $\cX=\{\x_1,...,\x_{\kk}\}$. A statistical privatization mechanism is a family of distributions $\Q$ that map $\X = \x$ to $\Y = \y$ with probability $\Q\left(\y|\x\right)$. $\Y$, the privatized version of $\X$, is defined on an output alphabet $\cY=\{\y_1,...,\y_l\}$ that need not be identical to the input alphabet $\cX$. In this paper, we will represent a privatization mechanism $\Q$ via a $\kk \times l$ row-stochastic matrix. A conditional distribution $\Q$
is said to be $\varepsilon$-locally differentially private if for all $\x$,
$\x'\in\mathcal{\X}$ and all $E \subset \mathcal{\Y}$, we have that
\begin{equation}
\Q\left(E|\x\right) \leq  e^{\varepsilon} \Q\left(E|\x'\right),
\end{equation}
where $\Q\left(E|\x\right) = \prob(\Y \in E | \X = \x)$ and $\varepsilon \in[0,\infty)$ \cite{duchi2013locala} .
In other words, by observing $\Y \in E$, the adversary cannot reliably infer whether $\X = \x$ or $\X = \x'$ (for any pair $\x$ and $\x'$). Indeed, the smaller the $\varepsilon$ is, the closer the likelihood ratio of $\X = \x$ to $\X =\x'$ is to 1. Therefore, when $\varepsilon$ is small, the adversary cannot recover the true value of $\X$ reliably.

\subsection{Private distribution estimation}
\label{sec:private_multinomial_estimation}

The private multinomial estimation problem is defined as follows. Given a vector $\PP = (\pp_1, ..., \pp_\kk)$ on the probability simplex $\mathbb{S}^\kk$, samples $X_1, ..., X_n$ are drawn i.i.d.~according to $\PP$. An $\varepsilon$-locally differentially private mechanism $\Q$ is then applied independently to each sample $\X_i$ to produce $Y^n = (Y_1, \cdots, Y_n)$, the sequence of private observations. Observe that the $\Y_i$'s are distributed according to $\MM = \PP\Q$ and not $\PP$. Our goal is to estimate the distribution vector $\PP$ from $Y^n$.

\textbf{Privacy vs. utility. }
There is a fundamental trade-off between utility and privacy. The more private you want to be, the less utility you can get. To formally analyze the privacy-utility trade-off, we study the following constrained minimization problem
\begin{equation}
\label{opt_privacy_utility}
r_{\LL,\varepsilon, \kk, \n} = \underset{\Q\in \mathcal{D}_\varepsilon}{\inf}~r_{\LL,\varepsilon, \kk, \n}(\Q),
\end{equation}
where
\begin{eqnarray*}
r_{\LL,\varepsilon, \kk, \n}(\Q) & = &  \inf_{\hat{\PP}} \sup_{\PP} \underset{Y^n \sim \PP\Q}{\mathbb{E}} \LL(\PP,  \hat{\PP})
\end{eqnarray*}
is the minimax risk under $\Q$,  $\LL$ is an application dependent loss function, and  $\mathcal{D}_\varepsilon$ is the set of all $\varepsilon$-locally differentially private mechanisms.

This problem, though of great value, is intractable in general. Indeed, finding minimax estimators in the non-private setting is already hard for several loss functions. For instance, the minimax estimator under $\ell_1$ loss is unknown even until today.  However, in the high privacy regime, we are able to bound the minimax risk of any differentially private mechanism $\Q$.

\begin{proposition}
\label{prop:opt_lb_lp}
For the private distribution estimation problem in \eqref{opt_privacy_utility}, for any $\varepsilon$-locally differentially private mechanism $\Q$, there exist universal constants $0 < c_l \leq c_u < 5$ such that for all $\varepsilon \in [0,1]$,
\begin{equation*}
c_l \min \left\{1, \frac{1}{\sqrt{n\varepsilon^2}}, \frac{k}{n\varepsilon^2} \right\} \leq r_{\ell_2^2,\varepsilon, \kk, \n}  \leq c_u \min \left\{1, \frac{k}{n\varepsilon^2} \right\},
\end{equation*}
and
\begin{equation*}
c_l \min \left\{1, \frac{\kk}{\sqrt{n\varepsilon^2}} \right\} \leq r_{\ell_1,\varepsilon, \kk, n} \leq c_u \min \left\{1, \frac{\kk}{\sqrt{n\varepsilon^2}} \right\}
\end{equation*}
\end{proposition}
\begin{proof}
See \cite{duchi2013local}.
\end{proof}
This result shows that in the high privacy regime ($\varepsilon \leq 1$), the effective sample size of a dataset decreases from $n$ to $n\varepsilon^2/ \kk$. In other words, a factor of $k/\varepsilon^2$ extra samples are needed to achieve the same minimax risk. This is problematic for large alphabets. Our work shows that (a) this problem can be (partially) circumvented using a combination of cohort-style hashing and \KRR (Section~\ref{sec:open_alphabets}), and (b) the dependence on the alphabet size vanishes in the moderate to low privacy regime (Section \ref{sec:theo_ana}).

\section{Binary Alphabets}
\label{sec:bin_alphabets}

In this section, we study the problem of private distribution estimation under binary alphabets. In particular, we show that Warner's randomized response model (\WRR) is optimal for binary distribution minimax estimation \cite{warner1965randomized}.
In \WRR, interviewees flip a biased coin (that only they can see the result of), such that a fraction $\eta$ of participants answer the question ``Is the predicate $\PScalar$ true (of you)?'' while the remaining particants answer the negation (``Is $\neg\PScalar$ true?''), without revealing which question they answered.
For $\eta = e^\varepsilon$ ($\varepsilon \geq 0$), \WRR can be described by the following $2\times2$ row-stochastic matrix
\begin{equation}
\label{eq:warner_response}
\Q_{\text{WRR}}=\frac{1}{e^{\varepsilon}+1}\left[
                                \begin{array}{cc}
                                  e^{\varepsilon} & 1 \\
                                  1 & e^{\varepsilon} \\
                                \end{array}
                              \right].
\end{equation}
It is easy to check that the above mechanism satisfies the constraints imposed by local differential privacy.

\begin{theorem}
\label{thm:opt_bin}
For all binary distributions $\PP$, all loss functions $\LL$, and all privacy levels $\varepsilon$, $\Q_{\text{WRR}}$ is the optimal solution to the private minimax distribution estimation problem in (\ref{opt_privacy_utility}).
\end{theorem}

\textbf{Proof sketch.} \cite{kairouz2014extremal} showed that \WRR dominates all other differentially private mechanisms in a strong Markovian sense: for any binary differentially private mechanism $\Q$, there exists a $2\times2$ stochastic mapping $\wMatrix$ such that $\qMatrix = \wMatrix \circ \qWrrMatrix$. Therefore, for any risk function $r(\cdot)$ that obeys the data processing inequality ($r(\qMatrix) \leq r(\qMatrix \circ \wMatrix)$ for any stochastic mappings $\qMatrix$ and $\wMatrix$), we have that $r(\qWrrMatrix) \leq r(\qMatrix)$ for any binary differentially private mechanism $\qMatrix$. In Supplementary Section~\ref{proof_opt_bin}, we prove that $r_{\LL,\varepsilon, \kk, \n}(\qMatrix)$ obeys the data processing inequality, thus \WRR achieves the optimal privacy-utility trade-off under minimax distribution estimation.

\section{$\kScalar$-ary Alphabets}
\label{sec:kary_alphabets}
Above, we saw that \WRR is optimal for all privacy levels and all loss functions. However, it can only be applied to binary alphabets. In this section, we study optimal privacy mechanisms for $k$-ary alphabets.
We show that under $\ell_1$ and $\ell_2$ losses, \KRAPPOR is order optimal in the high privacy regime and sub-optimal in the low privacy regime.  Conversely, \KRR is order optimal in the low privacy regime and sub-optimal in the high privacy regime.

\subsection{The $\kScalar$-ary Randomized Response}
\label{sec:krr}
The {\em $\kScalar$-ary randomized response} (\KRR) mechanism is a locally differentially private mechanism that maps $\cX$ stochastically onto itself
(i.e., $\cY=\cX$), given by
\begin{eqnarray}
	\qKrrMatrix(\y| \x) \,=\, \dfrac{1}{k-1+e^{\varepsilon}} \left\{
\begin{array}{rl}
	e^{\varepsilon}& \text{ if } \y=\x,\\
	1 & \text{ if } \y \neq \x.\\
\end{array}
\right.
	\label{eq:rr}
\end{eqnarray}
\KRR can be viewed as a multiple choice generalization of the \WRR mechanism (note that \KRR reduces to \WRR for $\kScalar=2$).  In \cite{kairouz2014extremal}, the \KRR mechanism was shown to be optimal in the low privacy regime for a large class of information theoretic utility functions. 

%
%
%

\textbf{Empirical estimation under \KRR. }
It is easy to see that under $\qKrrMatrix$, outputs are distributed according to:
\begin{equation}
\MM =  \frac{e^{\varepsilon} - 1}{e^{\varepsilon} + \kk - 1} \PP + \frac{1}{e^{\varepsilon} + \kk - 1}
\end{equation}

The empirical estimate of $\PP$ under $\qKrrMatrix$ is given by
\begin{eqnarray}
\label{eq:rr_empirical}
\hat{\PP} & = & \hat{\MM} \qKrrMatrix^{-1} \\ \nonumber
 & = &  \frac{e^{\varepsilon} + \kk - 1}{e^{\varepsilon} - 1}\hat{\MM} - \frac{1}{e^{\varepsilon} - 1},
 \end{eqnarray}
where $\hat{\MM}$ is the empirical estimate of $\MM$ and
\begin{eqnarray}
  \qKrrMatrix^{-1}( \y|  \x) \,=\, \dfrac{1}{e^{\varepsilon} - 1} \left\{
\begin{array}{rl}
  e^{\varepsilon} + \kk - 2& \text{ if } \y=\x,\\
  -1 & \text{ if } \y \neq \x.\\
\end{array}
\right.
  \label{eq:inv_rr}
\end{eqnarray}
via the Sherman-Morrison formula.  Observe that because $\hat{\MM} \rightarrow \MM$ almost surely, $ \hat{\PP} \rightarrow \PP$
almost surely.
\begin{proposition}
\label{prop:rr_perf}
For the private distribution estimation problem under \KRR and its empirical estimator given in \eqref{eq:rr_empirical}, for all $\varepsilon$, $n$, and $\kScalar$, we have that
\begin{equation*}
\mathbb{E}~\ell_2^2(\hat{\PP}, \PP) = \frac{1 - \sum_{i = 1}^ {\kk} \pp_i^2}{n} + \frac{\kk - 1}{n} \left(\frac{\kk + 2(e^{\varepsilon} - 1)}{(e^{\varepsilon} - 1)^2} \right) ,
\end{equation*}
and for large n, $\mathbb{E}~\ell_1(\hat{\PP}, \PP) \approx$
\begin{equation*}
\sum_{i = 1}^{\kk} \sqrt{\dfrac{2((e^{\varepsilon} - 1)\pp_i +1)((e^{\varepsilon} - 1)(1-\pp_i) + \kk - 1)}{\pi n (e^{\varepsilon} - 1)^2}},
\end{equation*}
where $a_n \approx b_n$ means $\lim_{n \rightarrow \infty} a_n/b_n = 1$.
\end{proposition}
\begin{proof}
See Supplementary Section \ref{proof_rr_perf}.
\end{proof}
Observe that for $\PP_{\text{U}} = \left(\frac{1}{\kk}, \cdots,\frac{1}{k}\right)$, we have that
\begin{eqnarray}
\label{eq:rr_l2_performance}
\mathbb{E}~\ell_2^2(\hat{\PP}, \PP)&\leq&  \mathbb{E}~\ell_2^2(\hat{\PP}, \PP_{\text{U}}) \\
&=& \left( 1 + \frac{\kk + 2(e^{\varepsilon} - 1)}{(e^{\varepsilon} - 1)^2}\kk \right) \frac{1 - \frac{1}{\kk}}{n}, \nonumber
\end{eqnarray}
and
\begin{eqnarray}
\label{eq:rr_l1_performance}
\mathbb{E}~\ell_1(\hat{\PP}, \PP) &\leq&  \mathbb{E}~\ell_1(\hat{\PP}, \PP_{\text{U}})   \\
&\approx& \left(\frac{e^{\varepsilon} + \kk - 1}{e^{\varepsilon} - 1}\right) \sqrt{\frac{2(\kk-1)}{\pi n}}. \nonumber
\end{eqnarray}

\textbf{Constraining empirical estimates to $\simplex^\kScalar$. }
\label{sec:rr_normalization}
It is easy to see that $||\hat{\PP}_{\text{KRR}}||_1  = 1$. However, some of the entries of $\hat{\PP}_{\text{KRR}}$ can be negative (especially for small values of $n$).  Several remedies are available, including (a) truncating the negative entries to zero and renormalizing the entire vector to sum to 1, or (b) projecting $\hat{\PP}_{\text{KRR}}$ onto the probability simplex.  We evaluate both approaches in Section \ref{sec:sim_results}.

\subsection{\KRAPPOR}
\label{sec:rappor}
The randomized aggregatable privacy-preserving ordinal response (\RAPPOR) is an open source Google technology for collecting aggregate statistics from end-users with strong local differential privacy guarantees \cite{erlingsson2014rappor}. The simplest version of \RAPPOR, called the basic one-time \RAPPOR and referred to herein as \KRAPPOR, first appeared in \cite{duchi2013locala,duchi2013local}. \KRAPPOR maps the input alphabet $\cX$ of size $\kk$ to an output alphabet $\cY$ of size $2^\kk$. In \KRAPPOR, we first map $\cX$ deterministically to $\tilde{\cX} = \mathbb{R}^\kk$, the $\kScalar$-dimensional Euclidean space. Precisely, $\X = \x_i$ is mapped to $\tilde{\X} = e_i$, the $i^{th}$ standard basis vector in $\mathbb{R}^\kk$. We then randomize the coordinates of $\tilde{\X}$ independently to obtain the private vector $\Y \in \{0,1\}^\kk$. Formally, the $j^{th}$ coordinate of $Y$ is given by: $\Y^{(j)}  = \tilde{X}^{(j)}$ with probability $e^{\varepsilon/2}/(1+e^{\varepsilon/2})$ and $1 - \tilde{X}^{(j)}$ with probability $1/(1+e^{\varepsilon/2})$. The randomization in $\Q_{\text{\KRAPPOR}}$ is $\varepsilon$-locally differentially private \cite{duchi2013locala,erlingsson2014rappor}.

Under \KRAPPOR, $Y_i = [ Y_i^{(1)}, \cdots, Y^{(\kk)}_i]$ is a $\kk$-dimensional binary vector, which implies that
\begin{equation}
\prob(Y_i^{(j)} = 1) =  \left(\frac{e^{\varepsilon/2} - 1}{e^{\varepsilon/2} + 1}\right) \pp_j  + \frac{1}{e^{\varepsilon/2} + 1},
\end{equation}
for all $i \in \{1, \cdots, n\}$ and $j \in \{1, \cdots, \kk\}$.

\textbf{Empirical estimation under \KRAPPOR. }
Let $Y^n$ be the $n \times \kk$ matrix formed by stacking the row vectors $Y_1, \cdots, Y_n$ on top of each other. The empirical estimator of $\PP$ under \KRAPPOR is:
\begin{equation}
\label{eq:rappor_empirical}
\hat{\pp}_j = \left(\frac{e^{\varepsilon/2} + 1}{e^{\varepsilon/2} - 1}\right) \frac{T_j}{n}- \frac{1}{e^{\varepsilon/2} - 1},
\end{equation}
where $T_j = \sum_{i = 1}^{n} Y_i^{(j)}$. Because $T_j/n$ converges to $m_j$ almost surely, $\hat{\pp}_j$ converges to $\pp_j$ almost surely.
As with \KRR, we can constrain $\hat{\pVector}$ to $\simplex^\kScalar$ through truncation and normalization or through projection (described in Section~\ref{sec:rr_normalization}), both of
which will be evaluated in Section \ref{sec:sim_results}.

\begin{proposition}
\label{prop:rappor_perf}
For the private distribution estimation problem under \KRAPPOR and its empirical estimator given in \eqref{eq:rappor_empirical}, for all $\varepsilon$, $n$, and $\kScalar$, we have that
\begin{equation*}
\mathbb{E}~\ell_2^2(\hat{\PP}, \PP) = \frac{1-\sum_{i=1}^{\kk}\pp_i^2}{n} + \frac{\kk e^{\varepsilon/2}}{n(e^{\varepsilon/2} - 1)^2},
\end{equation*}
and for large n, $\mathbb{E}~\ell_1(\hat{\PP}, \PP) \approx$
\begin{equation*}
\sum_{i = 1}^{\kk} \sqrt{\dfrac{2((e^{\varepsilon/2} - 1)\pp_i +1)((e^{\varepsilon/2} - 1)(1-\pp_i) + 1)}{\pi n (e^{\varepsilon/2} - 1)^2}},
\end{equation*}
where $a_n \approx b_n$ means $\lim_{n \rightarrow \infty} a_n/b_n = 1$.
\end{proposition}
\begin{proof}
See Supplementary Section \ref{proof_rappor_perf}.
\end{proof}
Observe that for $\PP_{\text{U}} = \left(\frac{1}{\kk}, \cdots,\frac{1}{k}\right)$, we have that
\begin{eqnarray}
\label{eq:rappor_l2_performance}
\mathbb{E}~\ell_2^2(\hat{\PP}, \PP) &\leq&  \mathbb{E}~\ell_2^2(\hat{\PP}, \PP_{\text{U}})\\
&=& \left(1 +  \frac{\kk^2e^{\varepsilon/2}}{(\kk - 1)(e^{\varepsilon/2} - 1)^2}\right)\frac{1 - \frac{1}{\kk}}{n},  \nonumber
\end{eqnarray}
and
\begin{eqnarray}
\label{eq:rappor_l1_performance}
\mathbb{E}~\ell_1(\hat{\PP}, \PP) &\leq& \mathbb{E}~\ell_1(\hat{\PP}, \PP_{\text{U}}) \\
  &\hspace{-2cm} \approx & \hspace{-1.2cm}  \sqrt{\frac{(e^{\varepsilon/2} + \kk  - 1) (e^{\varepsilon/2}(\kk-1)+1)}{(e^{\varepsilon/2}-1)^2 (\kk  - 1)}} \sqrt{\frac{2(\kk  - 1)}{\pi n}}. \nonumber
\end{eqnarray}

\subsection{Theoretical Analysis}
\label{sec:theo_ana}
We now analyze the performance of \KRR and \KRAPPOR relative to maximum likelihood estimation (which is equivalent to empirical estimation) on the non-privatized data $X^n$. In the non-private setting, the maximum likelihood estimator has a worst case risk of $\sqrt{\frac{2(\kk-1)}{\pi n}}$ under the $\ell_1$ loss, and a worst case risk of $\frac{1 - \frac{1}{k}}{n}$ under the $\ell_2^2$ loss \cite{lehmann1998theory, kamath2015learning}.

\textbf{Performance under \KRR. }
Comparing Equation \eqref{eq:rr_l2_performance} to the observation above, we can see that an extra factor of $\left( 1 + \frac{\kk + 2(e^{\varepsilon} - 1)}{(e^{\varepsilon} - 1)^2}\kk \right)$ samples is needed to achieve the same $\ell_2^2$ loss as in the non-private setting.
Similarly, from Equation \eqref{eq:rr_l1_performance}, a factor of $\left(\frac{e^{\varepsilon} + \kk - 1}{e^{\varepsilon} - 1}\right)^2$ samples is needed under the $\ell_1$ loss. For small $\varepsilon$, the sample size $n$ is effectively reduced to $n\varepsilon^2/\kk^2$ (under both losses). When compared to Proposition \ref{prop:opt_lb_lp}, this result implies that \KRR is not optimal in the high privacy regime. However, for $\varepsilon \approx \ln \kk$, the sample size $n$ is reduced to $n/4$ (under both losses). This result suggests that, while \KRR is not optimal for small values of $\varepsilon$, it is ``order'' optimal for $\varepsilon$ on the order of $\ln \kScalar$. Note that \KRR provides a natural interpretation of this low privacy regime: specifically, setting $\varepsilon=\ln \kScalar$ translates to telling the truth with probability $\frac{1}{2}$ and lying uniformly over the remainder of the alphabet with probability $\frac{1}{2}$; an intuitively reasonably notion of plausible deniability.

\textbf{Performance under \KRAPPOR. }
Comparing Equation \eqref{eq:rappor_l2_performance} to the observation at the beginning of this subsection, we can see that an extra factor of $\left(1 +  \frac{\kk^2e^{\varepsilon/2}}{(\kk - 1)(e^{\varepsilon/2} - 1)^2}\right)$ samples is needed to achieve the same $\ell_2^2$ as in the non-private case. Similarly, from Equation \eqref{eq:rappor_l1_performance}, an extra factor of $\frac{(e^{\varepsilon/2} + \kk  - 1) (e^{\varepsilon/2}(\kk-1)+1)}{(e^{\varepsilon/2}-1)^2 (\kk  - 1)}$ samples is needed under the $\ell_1$ loss. For small $\varepsilon$, $n$ is effectively reduced to $n\varepsilon^2/4\kk$ (under both losses). When compared to Proposition \ref{prop:opt_lb_lp}, this result implies that \KRAPPOR is ``order'' optimal in the high privacy regime. However, for $\varepsilon \approx \ln \kk$, $n$ is reduced to $n /\sqrt{\kk}$ (under both losses). This suggests that \KRAPPOR is strictly sub-optimal in the moderate to low privacy regime.

\begin{proposition}
\label{prop:rr_high_eps_l2_win}
For all $\PP \in \mathbb{S}^\kk$ and all $\varepsilon \geq \ln(\kk/2)$,
\begin{equation}
\mathbb{E}\left|\left|\hat{\PP}_{\textrm{KRR}} - \PP\right|\right|_2^2  \leq \mathbb{E}\left|\left|\hat{\PP}_{\textrm{\RAPPOR}} - \PP\right|\right|_2^2,
\end{equation}
where $\hat{\PP}_{\text{KRR}}$ is the empirical estimate of $\PP$ under \KRR, $\hat{\PP}_{\textrm{\RAPPOR}}$ is the empirical estimate of $\PP$ under \KRAPPOR,  and $\hat{\PP}$ is the empirical estimator under \KRAPPOR.
\end{proposition}
\begin{proof}
See Supplementary Section \ref{proof_rr_high_eps_l2_win}.
\end{proof}
\subsection{Simulation Analysis}
\label{sec:sim_results}

To complement the theoretical analysis above, we ran simulations of \KRR and \KRAPPOR varying the alphabet size $\kScalar$, the privacy level $\varepsilon$, the number of users $n$, and the true distribution $p$ from which the samples were drawn. In all cases, we report the mean over \num{10000} evaluations of $\|\hat{\PP} - \hat{\PP}_{\text{decoded}}\|_1$ where $\hat{\PP}$ is the ground truth sample drawn from the true distribution and $\hat{\PP}_{\text{decoded}}$ is the decoded \KRR or \KRAPPOR distribution. We vary $\varepsilon$ over a range that corresponds to the moderate-to-low privacy regimes in our theoretical analysis above, observing that even large values of $\varepsilon$ can provide plausible deniability impossible under un-noised logging.

We compare using the $\ell_1$ distance of the two distributions because in most applications we want to estimate all values well, emphasizing neither very large values (as an $\ell_2$ or higher metric might) nor very small values (as information theoretic metrics might). Supplementary Figures \ref{fig:l2_decoders} and \ref{fig:l2_krr_vs_rappor}, analogous to the ones in this section, demonstrate that the choice of distance metric does not qualitatively affect our conclusions on the decoding strategies for \KRR or \KRAPPOR nor on the regimes in which each is superior.

The distributions we considered in simulation were binomial distributions with parameter in $\{.1, .2, .3, .4, .5\}$ , Zipf distribution with parameter in $\{1, 2, 3, 4, 5\}$, multinomial distributions drawn from a symmetric Dirichlet distribution with parameter $\vec{\mathbf{1}}$, and the geometric distribution with mean $\kScalar/5$. The geometric distribution is shown in Supplementary Figure~\ref{fig:geometric_ground_truth}. We focus primarily on the geometric distribution here because qualitatively it shows the same patterns for decoding as the full set of binomial and Zipf distributions and it is sufficiently skewed to represent many real-world datasets. It is also the distribution for which \KRAPPOR does the best relative to \KRR over the largest range of $\kScalar$ and $\varepsilon$ in our simulations.

\subsubsection{Decoding}
\label{sec:rr_and_rappor_decoding}

We first consider the impact of the choice of decoding mechanism used for \KRR and \KRAPPOR. We find that the best decoder in practice for both \KRR and \KRAPPOR on skewed distributions is the \textit{projected decoder} which projects the $\hat{\PP}_{\mathrm{KRR}}$ or $\hat{\PP}_{\mathrm{\RAPPOR}}$ onto the probability simplex $\simplex^\kScalar$ using the method described in Algorithm~1 of \cite{DBLP:journals/corr/WangC13a}. For \KRR, we compare the projected empirical decoder to the normalized empirical decoder (which truncates negative values and renormalizes) and to the maximum likelihood decoder (see Supplementary Section~\ref{sec:rr_ml_decoder}).  For \KRAPPOR, we compare the standard decoder, normalized decoder, and projected decoder. Figure~\ref{fig:decoders} shows that the projected decoder is substantially better than the other decoders for both \KRR and \KRAPPOR for the whole range of $\kScalar$ and $\varepsilon$ for the geometric distribution. We find this result holds as we vary the number of users from $30$ to $10^6$ and for all distributions we evaluated except for the Dirichlet distribution, which is the least skewed. For the Dirichlet distribution, the normalized decoder variant is best for both \KRR and \KRAPPOR. Because the projected decoder is best on all the skewed distributions we expect to see in practice, we use it exclusively for the open-alphabet experiments in Section~\ref{sec:open_alphabets}.

\subsubsection{\KRR vs \KRAPPOR}
\label{sec:rr_vs_rappor}

\begin{figure*}
\vspace{-.25in}
\centering
\begin{tabular}{m{.2in}cc}
&
\begin{subfigure}[b]{.35\linewidth}
\includegraphics[width=\linewidth]{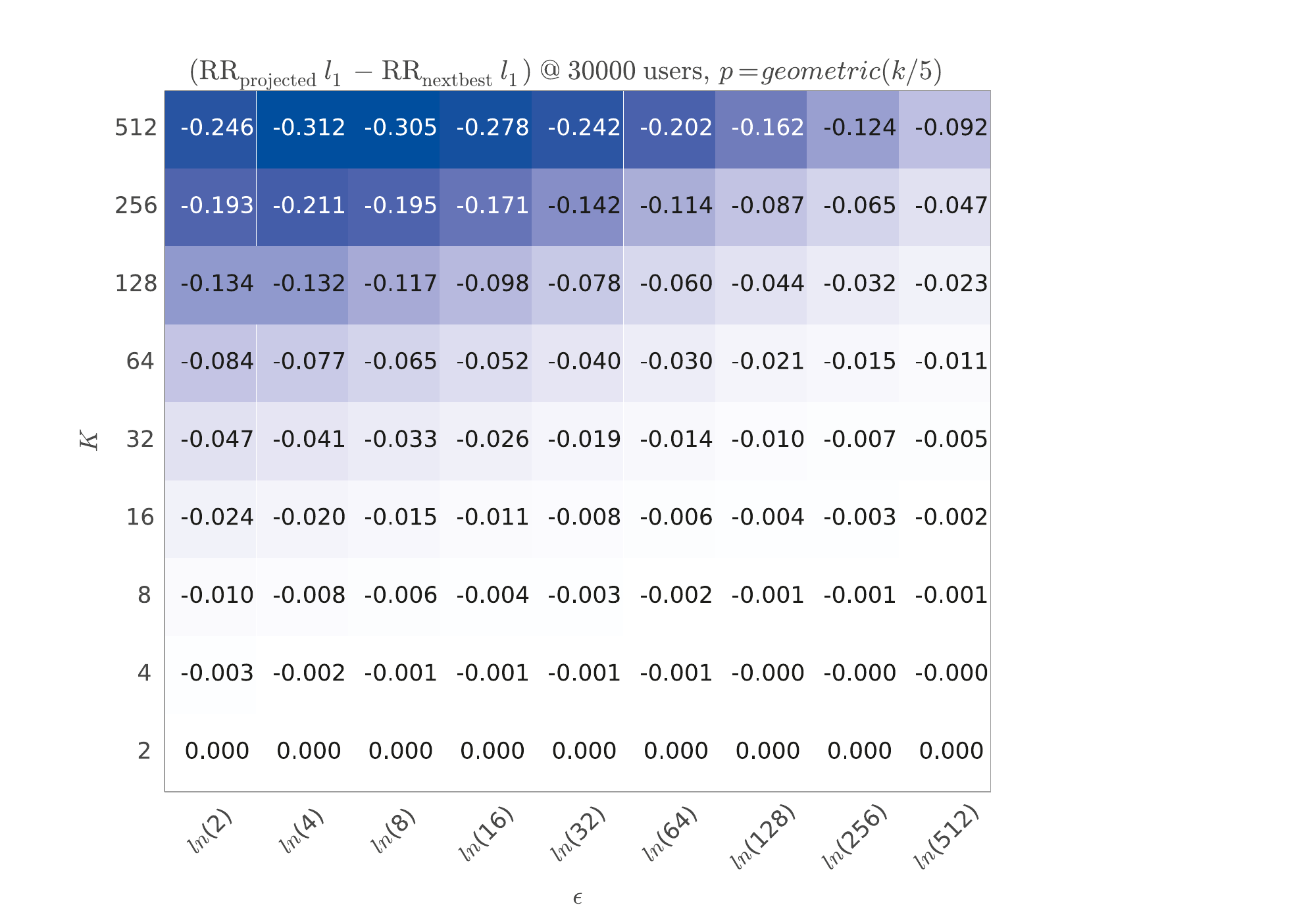}
\end{subfigure}
&
\begin{subfigure}[b]{.35\linewidth}
\includegraphics[width=\linewidth]{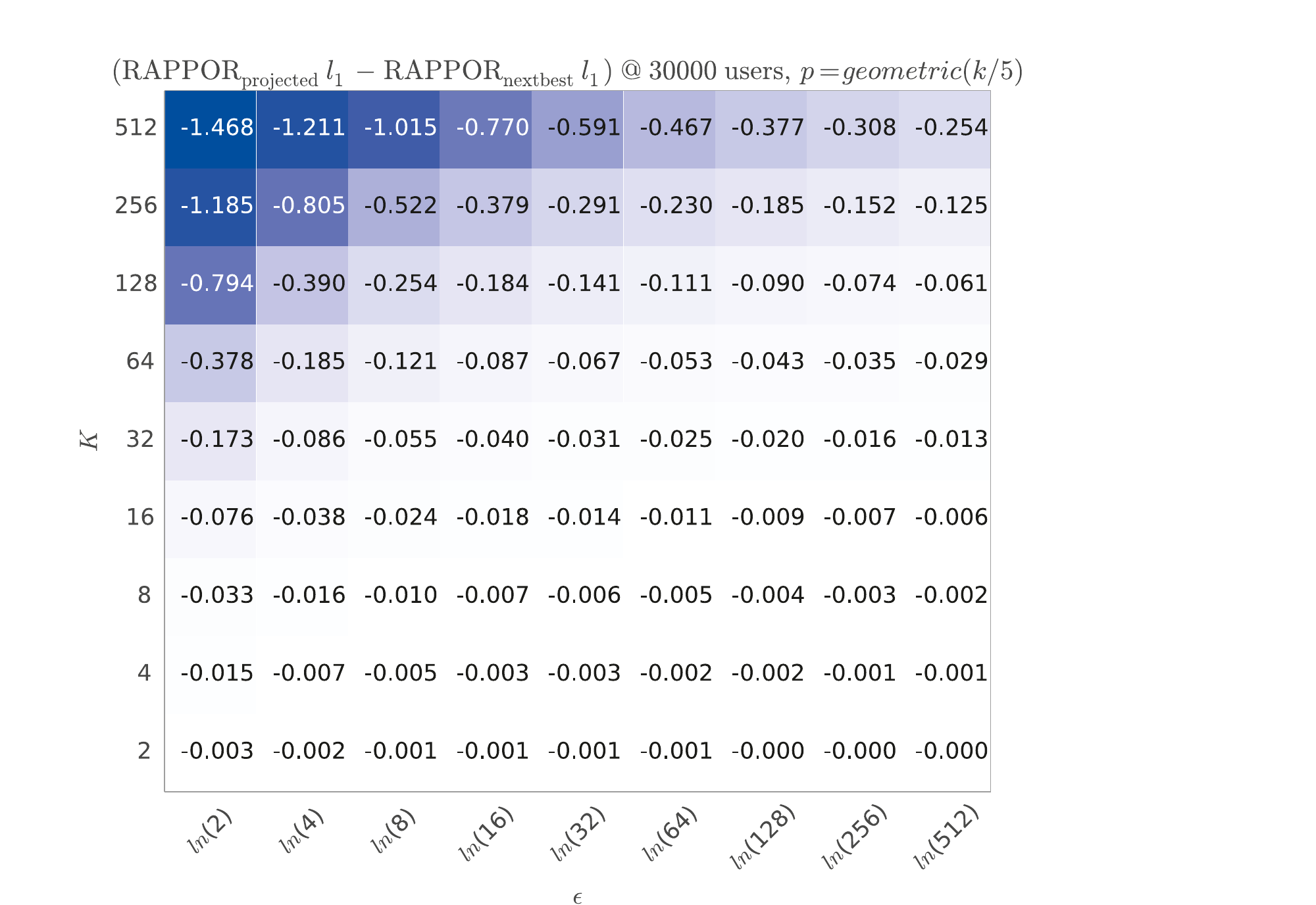}
\end{subfigure}
\end{tabular}
\vspace{-1em}
\caption{
The improvement in $\ell_1$ decoding of the projected \KRR decoder (left) and projected \KRAPPOR decoder (right). Each grid varies the size of the alphabet $\kScalar$ (rows) and privacy parameter $\varepsilon$ (columns). Each cell shows the difference in $\ell_1$ magnitude that the projected decoder has over the ML and normalized \KRR decoders (left) or the standard and normalized \KRAPPOR decoders (right). Negative values mean improvement of the projected decoder over the next best alternative.}
\label{fig:decoders}
\end{figure*}

\begin{figure*}
\vspace{-.36in}
\centering
\begin{tabular}{m{.2in}cc}
\parbox[t]{2mm}{\rotatebox{90}{\rlap{\hspace{.6in}Geometric}}}
&
\begin{subfigure}[b]{.35\linewidth}
\includegraphics[width=\linewidth]{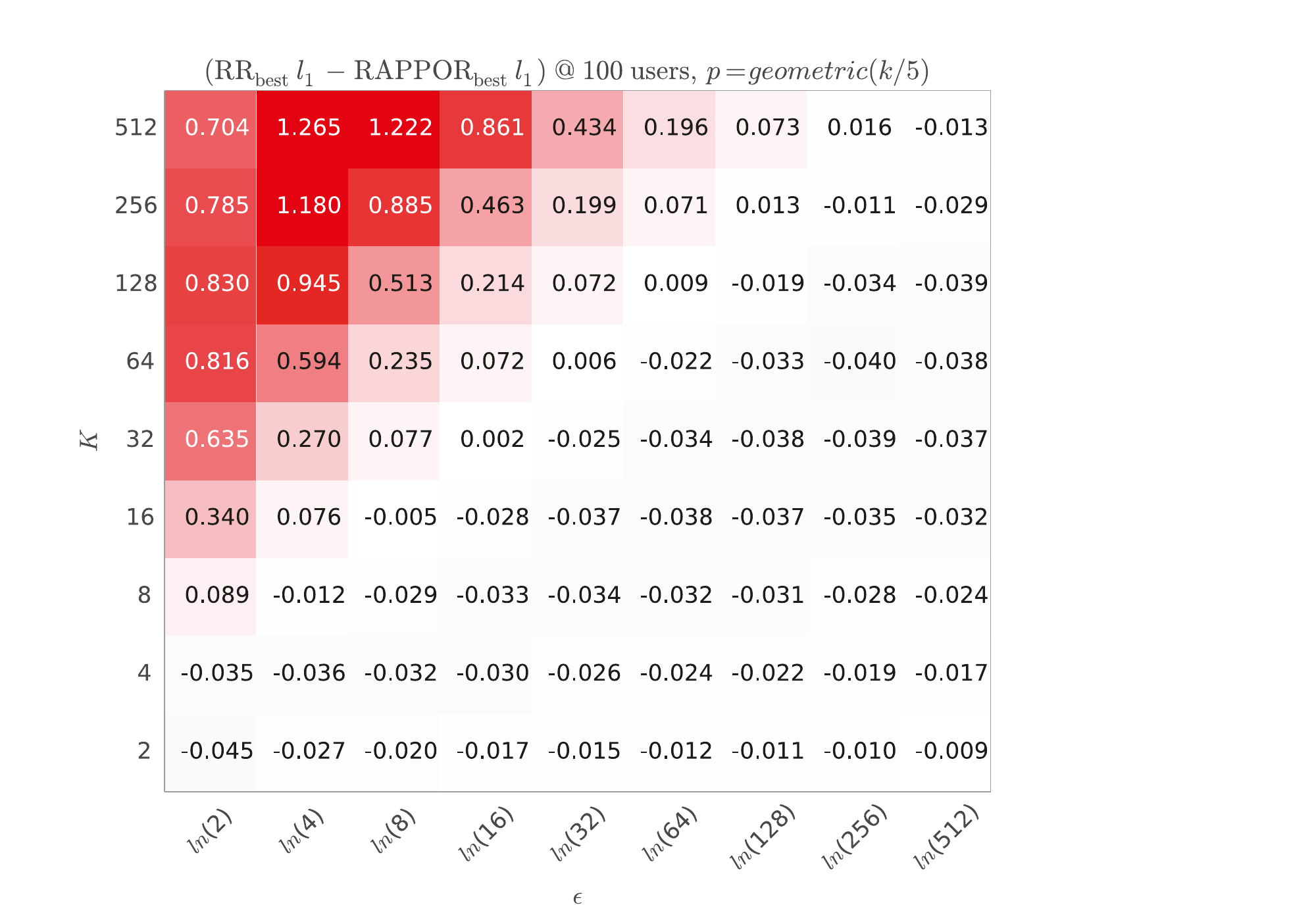}
\end{subfigure}
&
\begin{subfigure}[b]{.35\linewidth}
\includegraphics[width=\linewidth]{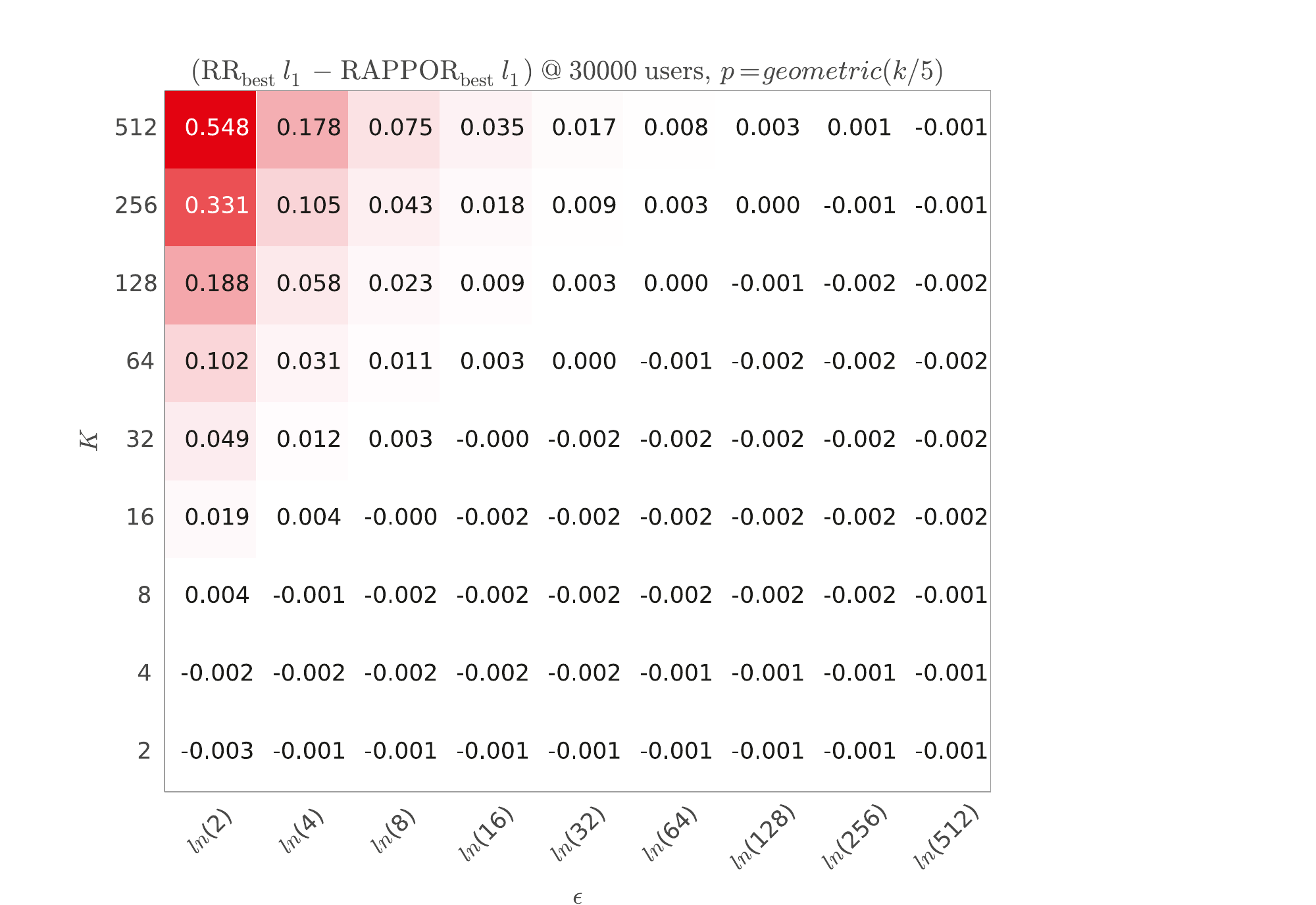}
\end{subfigure}
\\
\parbox[t]{2mm}{\rotatebox{90}{\rlap{\hspace{.6in}Dirichlet}}}
&
\begin{subfigure}[b]{.35\linewidth}
\vspace{-1em}
\includegraphics[width=\linewidth]{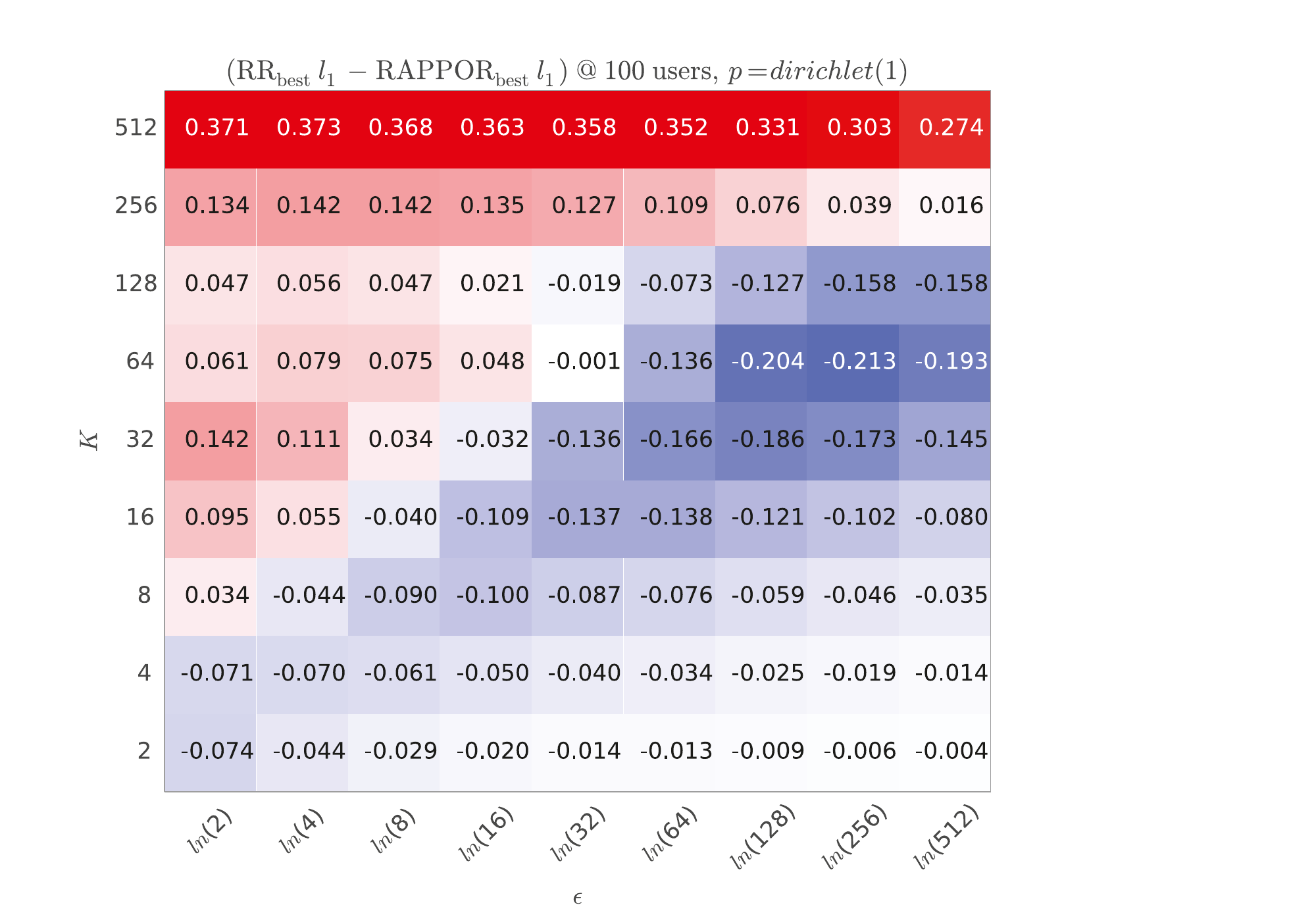}
\end{subfigure}
&
\begin{subfigure}[b]{.35\linewidth}
\vspace{-1em}
\includegraphics[width=\linewidth]{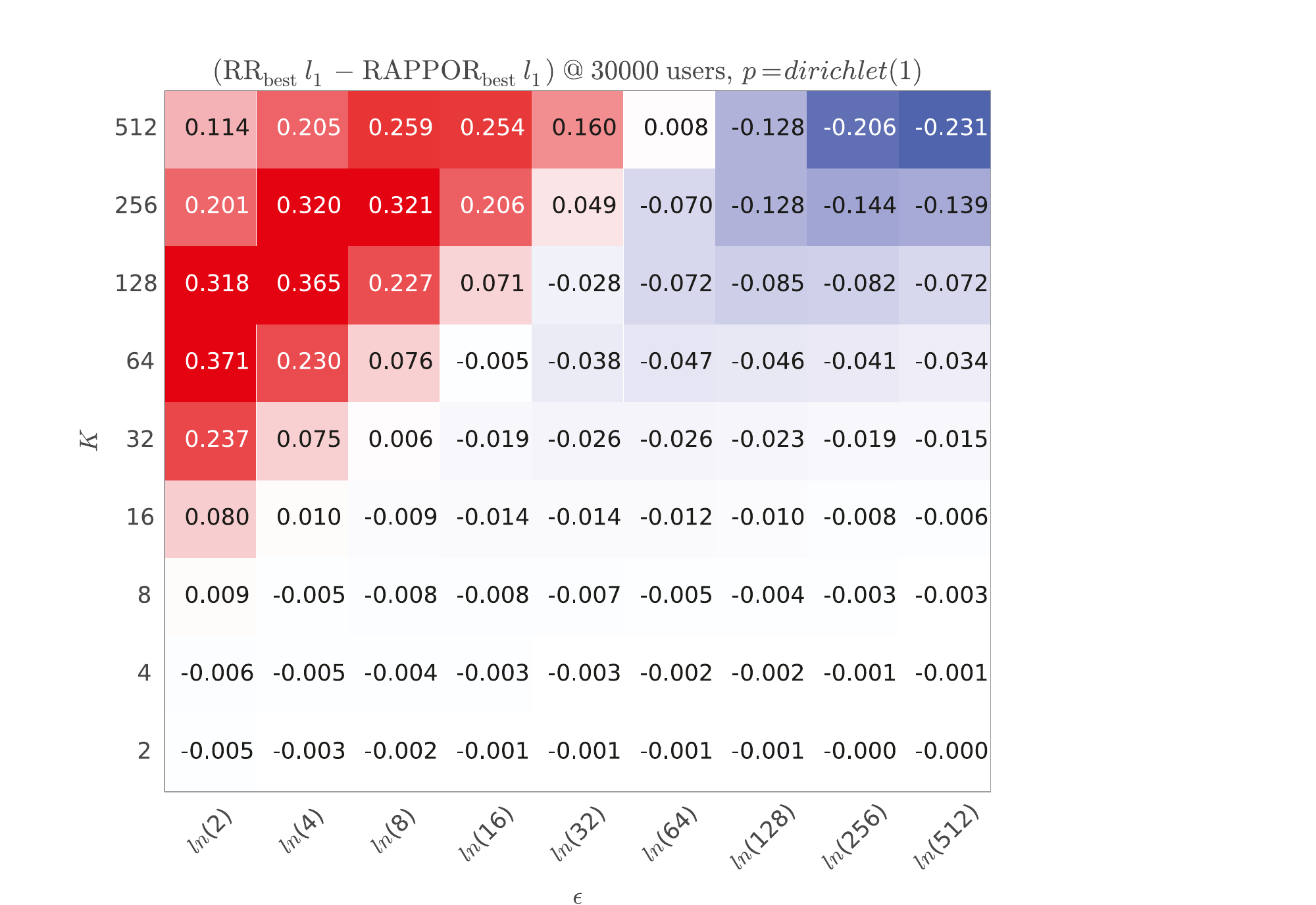}
\end{subfigure}

\end{tabular}
\vspace{-1em}
\caption{The improvement (negative values, blue) of the best \KRR decoder over the best \KRAPPOR decoder varying the size of the alphabet $\kScalar$ (rows) and privacy parameter $\varepsilon$ (columns). The left charts focus on small numbers of users (100); the right charts show a large number of users (30000, also representative of larger numbers of users). The top charts show the geometric distribution (skewed) and the bottom charts show the Dirichlet distribution (flat).}
\label{fig:krr_vs_rappor}
\end{figure*}

\begin{figure*}
\vspace{-.36in}
\captionsetup[subfigure]{aboveskip=-5pt,belowskip=-5pt}
\centering
\begin{tabular}{cc}
\begin{subfigure}[b]{.4\linewidth}
\includegraphics[width=\linewidth]{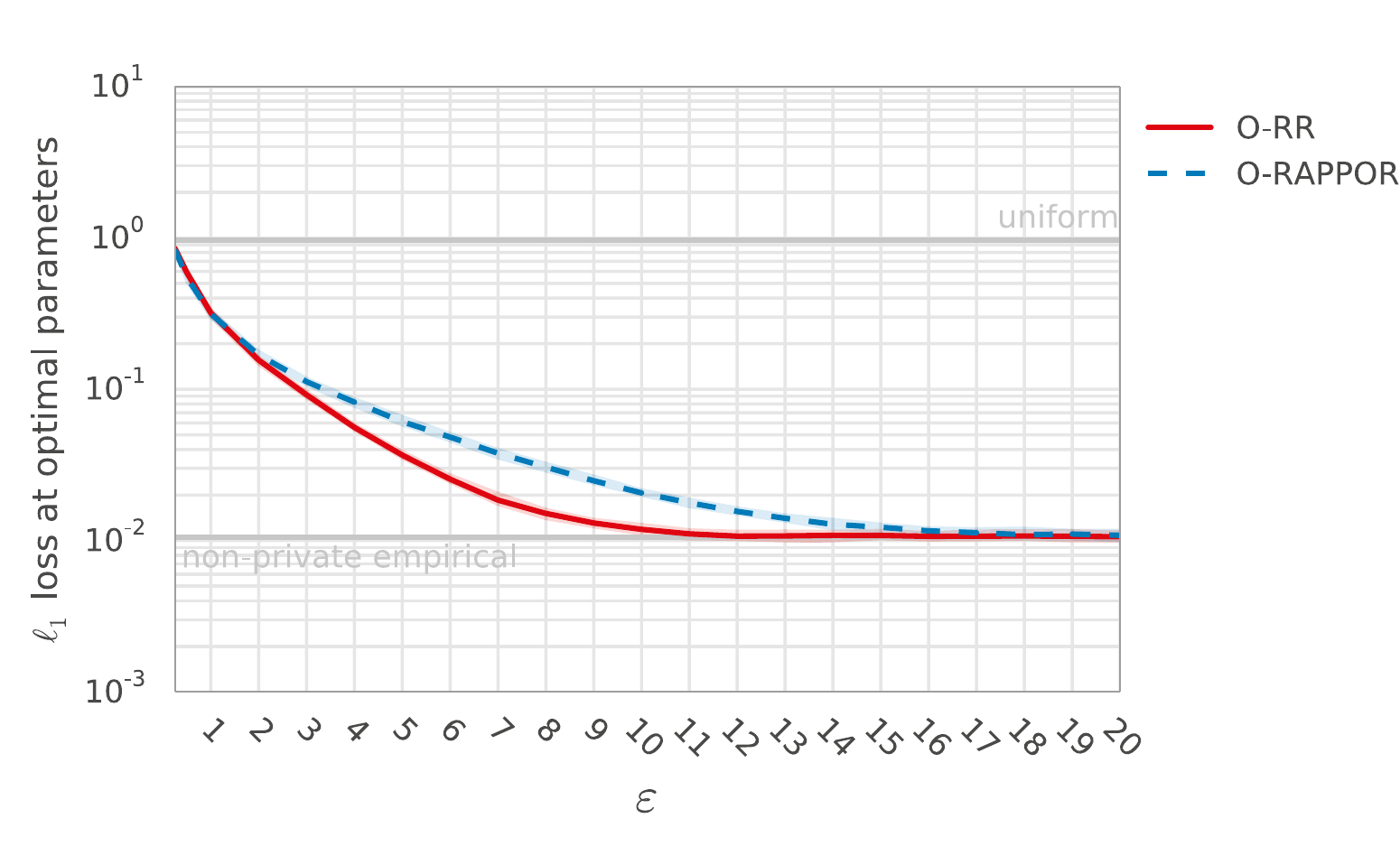}
\caption{Open alphabets.}
\label{fig:orr_vs_rappor:open}
\end{subfigure}
&
\begin{subfigure}[b]{.4\linewidth}
\includegraphics[width=\linewidth]{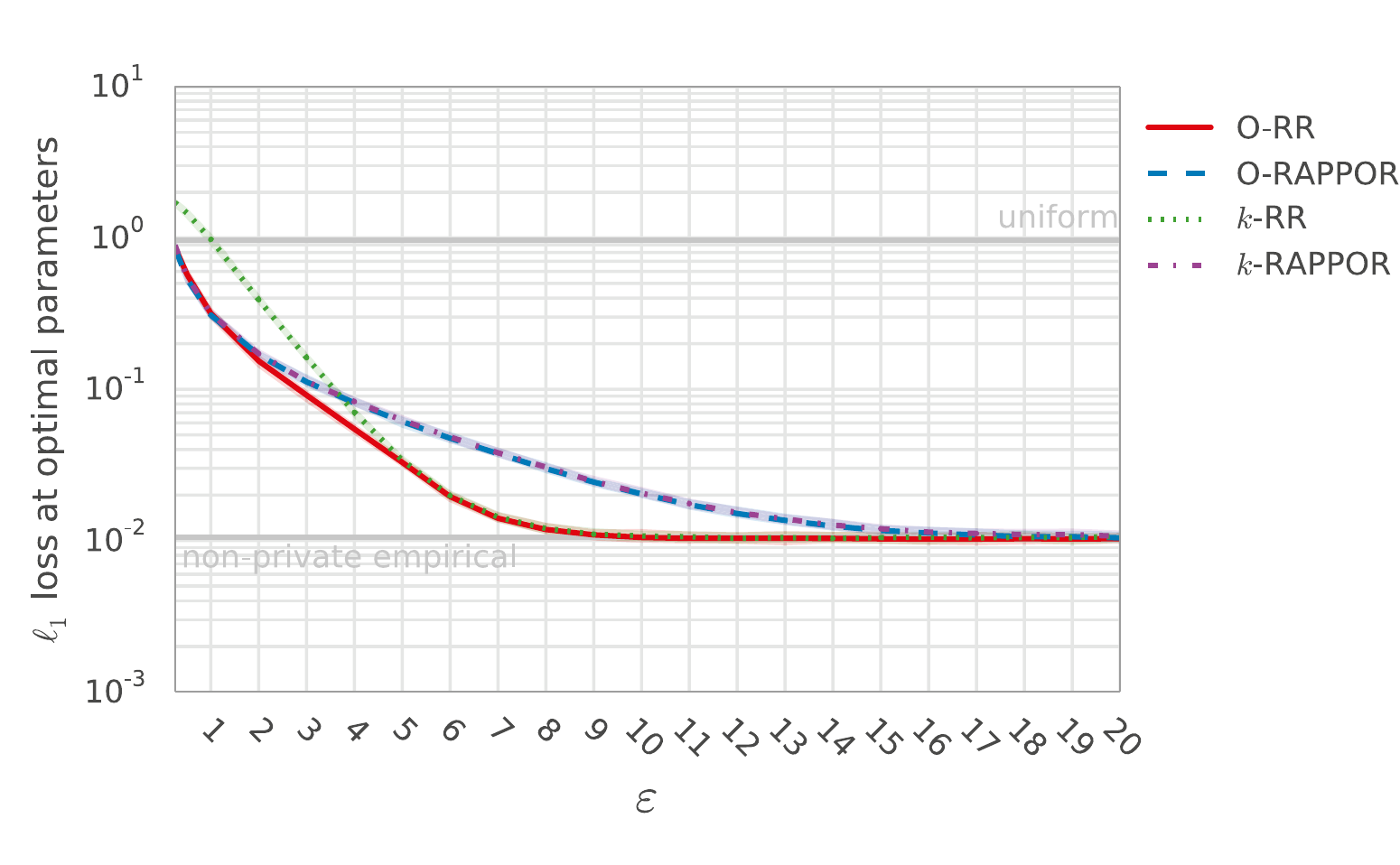}
\caption{Closed alphabets.}
\label{fig:orr_vs_rappor:closed}
\end{subfigure}
\end{tabular}
\vspace{-.5em}
\caption{
$\ell_1$ loss of \ORR and \ORAPPOR for $n=10^6$ on the geometric distribution
when applied to unknown input alphabets (via hash functions, \subref*{fig:orr_vs_rappor:open})
and to known input alphabets (via perfect hashing, \subref*{fig:orr_vs_rappor:closed}).
Lines show median $\ell_1$ loss with 90\% confidence intervals over 50 samples.
Free parameters are set via grid search over
$\kScalar \in [2, 4, 8, \ldots, 2048, 4096]$,
$\cScalar \in [1, 2, 4, \ldots, 512, 1024]$,
$\hScalar \in [1, 2, 4, 8, 16]$ for each $\varepsilon$.
Note that the \KRAPPOR and \ORAPPOR lines in
\protect\subref*{fig:orr_vs_rappor:closed} are nearly indistinguishable.
Baselines indicate expected loss from (1) using an empirical estimator
directly on the input $\sVector$ and (2) using the uniform distribution
as the $\hat{\pVector}$ estimate.
}
\label{fig:orr_vs_rappor}
\end{figure*}

To construct a fair, empirical comparison of \KRR and \KRAPPOR, we employ the same methodology used above in selecting decoders. Figure~\ref{fig:krr_vs_rappor} shows the difference between the best \KRR decoder and the best \KRAPPOR decoder (for a particular $\kScalar$ and $\varepsilon$). For most cells, the best decoder is the projected decoder described above.

Note that the best \KRAPPOR decoder is consistently better than the best \KRR decoder for relatively large $\kScalar$ and low $\varepsilon$. However, \KRR is slightly better than \KRAPPOR in all conditions where $k < e^\varepsilon$ (bottom-right triangle), an empirical result for $\ell_1$ that complements Proposition~\ref{prop:rr_high_eps_l2_win}'s statement about ML decoders in $\ell_2$. All of the skewed distributions manifest the same pattern as the geometric distribution. As the number of users increases, \KRR's advantage over \KRAPPOR in the low privacy environment shrinks. In the next sections, we will examine the use of cohorts to improve decoding and to handle larger, open alphabets.

\section{Open Alphabets, Hashing, and Cohorts}

\label{sec:open_alphabets}
In practice, the set of values that may need to be collected may not be easily enumerable in advance, preventing a direct application of the binary and $\kScalar$-ary formulations of private distribution estimation. Consider a population of $\nScalar$ users, where each user $\iScalar$ possesses a symbol $\sScalar_\iScalar$ drawn from a large set of symbols $\sSet$ whose membership is not known in advance.  This scenario is common in practice; for example, in Chrome's estimation of the distribution of home page settings \cite{erlingsson2014rappor}.  Building on this intuitive example, we assume for the remainder of the paper that symbols $\sScalar_\iScalar$ are strings, but we note that the methods described are applicable to any hashable structures.

\subsection{\ORR: \KRR with hashing and cohorts}
\label{sec:orr}
\KRR is effective for privatizing over known alphabets.  Inspired by \cite{erlingsson2014rappor}, we extend \KRR to open alphabets by combining two primary intuitions: hashing and cohorts. Let $\hash(\sScalar)$ be a function mapping $\sSet \rightarrow \naturals$ with a low collision rate, i.e. $\hash(\sScalar) = \hash(\sScalar')$ with very low probability for $\sScalar' \ne \sScalar$. With hashing, we could use \KRR to guarantee local privacy over an alphabet of size $\kScalar$ by having each client report $\qKrrMatrix(\hash(\sScalar)\mod{\kScalar})$. However, as we will see, hashing alone is not enough to provide high utility because of the increased rate of collisions introduced by the modulus.

Complementing hashing, we also apply the idea of hash \textit{cohorts}: each user $\iScalar$ is assigned to a cohort $\cScalar_\iScalar$ sampled i.i.d. from the uniform distribution over $\cSet = \{1, ..., \CScalar\}$. Each cohort $\cScalar \in \cSet$ provides an independent view of the underlying distribution of strings by projecting the space of strings $\sSet$ onto a smaller space of symbols $\xSet$ using an independent hash function $\hash_\cScalar$.  The users in a cohort use their cohort's hash function to partition $\sSet$ into $\kScalar$ disjoint subsets by computing $\xScalar_\iScalar  = \hash_{\cScalar_\iScalar} (\sScalar_\iScalar) \mod{\kScalar} = \hash_{\cScalar_\iScalar}^{(\kScalar)}(\sScalar_\iScalar)$.  Each subset contains approximately the same number of strings, and because each cohort uses a different hash function, the induced partitions for different cohorts are orthogonal:
$\prob(\xScalar_\iScalar = \x_j | \cScalar_\iScalar \neq \cScalar_j) \approx \frac{1}{\kScalar}$ even when $\sScalar_\iScalar = \sScalar_j$.

\subsubsection{Encoding and Decoding}
For encoding, the \ORR privatization mechanism can
be viewed as a sampling distribution independent of $\cSet$. Therefore, $\qOrrMatrix (\yScalar, \cScalar | \sScalar)$ is given by
\begin{equation}
 \frac{1}{\CScalar(e^{\varepsilon} + \kScalar - 1)}
\left\{
\begin{array}{rl}
  e^{\varepsilon} & \text{ if } \hash_\cScalar^{(\kScalar)}(\sScalar) = \yScalar,\\
  1               & \text{ if } \hash_\cScalar^{(\kScalar)}(\sScalar) \neq \yScalar.\\
\end{array}
\right.
\end{equation}
For decoding, fix candidate set $\sSet$ and interpret the
privatization mechanism $\qOrrMatrix$ as a $\kScalar \CScalar \times \SScalar$
row-stochastic matrix:




\begin{equation}
\qOrrMatrix = \frac{1}{\CScalar} \frac{1}{e^{\varepsilon} + \kScalar - 1} \left(\one + (e^{\varepsilon} - 1)\hMatrix \right)
\end{equation}
where:
\begin{equation}
\hMatrix(\yScalar, \cScalar | \sScalar) = \unity_{\{\hash_{\cScalar}^{(\kScalar)}(\sScalar) = \yScalar\}}
\end{equation}
Note that $\hMatrix$ is a $\kScalar \CScalar \times \SScalar$ sparse binary matrix
encoding the hashed outputs for each cohort, wherein each column of
$\hMatrix$ has exactly $\CScalar$ non-zero entries.


Now $\mVector = \pVector \qOrrMatrix$ is the expected output distribution for true probability vector $\pVector$,
allowing us to form an empirical estimator by using standard least-squares
techniques to solve the linear system:
%
\begin{equation}
\label{eq:rr_strings_regression}
\hat{\pVector}_\text{ORR} \hMatrix = \frac{1}{e^{\varepsilon}-1} \left( \CScalar (e^{\varepsilon}+ \kScalar - 1) \hat{\mVector} - \one \right).
\end{equation}

Note that when $\CScalar=1$ and $\hMatrix$ is the identity matrix,
\eqref{eq:rr_strings_regression}
reduces to standard \KRR empirical estimator as seen
in \eqref{eq:rr_empirical}.

As with the \KRR empirical estimator, $\hat{\pVector}_\text{ORR}$ may have negative entries.  Section~\ref{sec:rr_normalization} describes methods for constraining $\hat{\pVector}_\text{ORR}$ to $\simplex^\kScalar$, of which simplex projection is demonstrated to offer superior performance in Section \ref{sec:sim_results}.  The remainder of the paper assumes that \ORR uses the simplex projection strategy.

\subsection{\ORAPPOR}
\label{sec:orappor}
\RAPPOR also extends from $\kScalar$-ary alphabets to open alphabets using
hashing and cohorts \cite{erlingsson2014rappor}; we refer to this extension
herein as \ORAPPOR.  However, the \KRAPPOR
mechanism uses a size $|\tilde{\xSet}| = 2^\kScalar$ input representation as
opposed to \KRR's size $|\xSet| = \kScalar$ representation.
Taking advantage of the larger input space, \ORAPPOR uses an
independent $\hScalar$-hash Bloom filter $\bloom_\cScalar^{(\kScalar)}$ for
each cohort before applying the \KRAPPOR mechanism---i.e. the
$\jScalar$-th bit of $\xScalar_\iScalar$ is 1 if
$\hash_{\cScalar, \hScalar'}^{(k)}(\sScalar_\iScalar) = \jScalar$
for any $\hScalar' \in [1 \ldots \hScalar]$, where
$\hash_{\cScalar, \hScalar'}^{(\kScalar)}$ are a set of $\hScalar\CScalar$
mutually independent hash functions modulo $\kScalar$.

Decoding for \ORAPPOR is described in \cite{erlingsson2014rappor}
and follows a similar strategy as for \ORR.  However, because this paper
focuses on distribution estimation rather than heavy hitter detection,
we eliminate both the Lasso regression stage and filtering of imputed
frequencies relative to Bonferroni corrected thresholds, retaining just
the regular least-squares regression.

\subsection{Simulation Analysis}
\label{sec:open_simulation_analysis}
We ran simulations of \ORR and \ORAPPOR for $\nScalar=10^6$
users with input drawn from an alphabet of $\SScalar=256$ symbols under a
geometric distribution with mean=$\SScalar/5$ (see
Supplementary Figure~\protect\ref{fig:geometric_ground_truth}).
As described in Section~\protect\ref{sec:sim_results}, the geometric distribution is
representative of actual data and relatively easy for \KRAPPOR and challenging
for \KRR.  Free parameters were set to minimize the median $\ell_1$ loss.
Similar results for $\SScalar=4096$ and $\nScalar=10^6$ and $10^8$ are included
in the Supplementary Material.


In Figure~\ref{fig:orr_vs_rappor:open}, we see that under these conditions,
\ORR matches the utility of \ORAPPOR in both the very low and high privacy
regimes and exceeds the utility of \ORAPPOR over midrange privacy settings.

For \ORR, we find that the optimal $\kScalar$ depends directly on $\varepsilon$,
that increasing $\CScalar$ consistently improves performance in the low-to-mid
privacy regime, and that $\CScalar=1$ noticably underperforms across
the range of privacy levels.  For \ORAPPOR, we find that performance improves as
$\kScalar$ increases (with $\kScalar=4096$ near the asymptotic limit), that
$\CScalar=1$ noticably underperforms across the range of privacy values, but
with all $\CScalar \ge 2$ performing indistinguishably.  Finally, we find that
the optimal value for $\hScalar$ is consistently 1, indicating that Bloom filters
provide no utility improvement beyond simple hashing.
See Supplementary Figure~\ref{fig:open_set_params} for details.

\subsection{Improved Utility for Closed Alphabets}
\label{sec:krr_with_hashing}

\ORR and \ORAPPOR extend $\kScalar$-ary mechanisms to open alphabets
through the use of hash functions and cohorts. These same mechanisms may also
be applied to closed alphabets known \textit{a priori}.  While direct application is
possible, the reliance on hash functions exposes both mechanism to
unnecessary risk of hash collision.

Instead, we modify the \ORR and \ORAPPOR mechanisms, replacing
each cohort's generic hash functions with minimal perfect hash functions mapping
$\sSet$ to $[0 \ldots \SScalar-1]$ before applying the modulo $\kScalar$ operation.
In most closed-alphabet applications, $\sSet = [0 \ldots \SScalar-1]$, in which case
these minimal perfect hash functions are simply permutations.  Also note that
in this setting, \ORR and and \ORAPPOR reduce to exactly their $\kScalar$-ary
counterparts when $\CScalar$ and $\hScalar$ are both 1 except that the
output symbols are permuted.

In Figure~\ref{fig:orr_vs_rappor:closed}, we evaluate these modified
mechanisms using the same method described in
Section~\ref{sec:open_simulation_analysis}
(note that the utilities of \KRAPPOR and \ORAPPOR are nearly
indistinguishable).
\ORAPPOR benefits
little from the introduction of minimal perfect hash functions.
In contrast, \ORR's utility improves significantly, meeting or exceeding the
utility of all other mechanisms at all considered $\varepsilon$.

\section{Conclusion}

Data improves products, services, and our understanding of the world. But its collection comes with risks to the individuals represented in the data as well as to the institutions responsible for the data's stewardship. This paper's focus on distribution estimation under local privacy takes one step toward a world where the benefits of data-driven insights are decoupled from the collection of raw data. Our new theoretical and empirical results show that combining cohort-style hashing with the $\kScalar$-ary extension of the classical randomized response mechanism admits practical, state of the art results for locally private logging.

In many applications, data is collected
to enable the making of a specific decision.  In such settings, the
nature of the decision frequently determines the required level
of utility, and the number of reports to be collected $\nScalar$ is
pre-determined by the size of the existing user base.  Thus, the differential
privacy practitioner's role is often to offer users as much privacy as
possible while still extracting sufficient utility at the given $\nScalar$.
Our results suggest that \ORR may play a crucial role for such a
practitioner, offering a single mechanism that provides maximal privacy
at any desired utility level simply by adjusting the mechanism's parameters.

In future work, we plan to examine estimation of non-stationary distributions as they change over time, a common scenario in data logged from user interactions. We will also consider what utility improvements may be possible when some responses need more privacy than others, another common scenario in practice. Much more work remains before we can dispel the collection of un-noised data altogether.

\textbf{Acknowledgements.} Thanks to \'{U}lfar Erlingsson, Ilya Mironov, and Andrey Zhmoginov for their comments on drafts of this paper.

\bibliography{refs}
\bibliographystyle{icml2016}

\newpage
\onecolumn
\appendix

\begin{center}
{\Large Supplementary Material: Discrete Distribution Estimation Under Local Privacy}
\end{center}

\section{Proof of Theorem \ref{thm:opt_bin}}
\label{proof_opt_bin}
As argued in the proof sketch of Theorem \ref{thm:opt_bin}, it suffices to show that  $r_{\LL,\varepsilon, \kk, \n}(\Q)$ obeys the data processing inequality. Precisely, we need to show that for any row stochastic matrix $\mathbf{W}$, $r_{\LL,\varepsilon, \kk, \n}(\mathbf{W}\Q) \geq r_{\LL,\varepsilon, \kk, \n}(\Q)$. Observe that this is equivalent to showing that $r_{\LL,\varepsilon, \kk, \n}(\Q) \geq r_{\LL, \kk, \n}$, where $r_{\LL, \kk, \n}$ is the minimax risk in the non-private setting.

Consider the set of all randomized estimators $\hat{\PP}$. Under randomized estimators, the minimax risk is given by
\begin{equation*}
r_{\LL, k,n} = \inf_{\hat{\PP}} \sup_{\PP \in \mathbb{S}^\kk}  \underset{X^n \sim \PP, \hat{ \PP}}{\mathbb{E}} \LL(\PP,  \hat{\PP}),
\end{equation*}
where the expectation is taken over the randomness in the observations $X_1, \cdots, X_n$ and the randomness in $\hat{\PP}$. Under a differentially private mechanism $\Q$, the minimax risk is given by
\begin{equation*}
r_{\LL,\varepsilon, \kk, \n}(\Q)  = \inf_{\hat{\PP}_{\Q}} \sup_{\PP \in \mathbb{S}^\kk}  \underset{Y^n \sim \PP\Q, \hat{\PP}_{\Q}}{\mathbb{E}} \LL(\PP,  \hat{\PP}_{\Q}),
\end{equation*}
where the expectation is taken over the randomness in the private observations $Y_1, \cdots, Y_n$ and the randomness in $\hat{\PP}_{\Q}$.

Assume that there exists a (potentially randomized) estimator $\hat{\PP}^{*}_{\Q}$ that achieves $r_{\LL,\varepsilon, \kk, \n}(\Q)$. Consider the following randomized estimator: $\Q$ is first applied to $X_1, \cdots, X_n$ individually and $\hat{\PP}^{*}_{\Q}$ is then jointly applied to the outputs of $\Q$. This estimator achieves a risk of $r_{\LL,\varepsilon, \kk, \n}(\Q)$. Therefore, $r_{\LL, k,n} \leq r_{\LL,\varepsilon, \kk, \n}(\Q)$.

If there is no estimator that can achieve $r_{\LL,\varepsilon, \kk, \n}(\Q)$, then there exists a sequence of (potentially randomized) estimators $\{\hat{\PP}^{i}_{\Q}\}$ such that $\lim_{i \rightarrow \infty} \hat{\PP}^{i}_{\Q}$ achieves the minimax risk. In other words, if $r^{i}_{\LL,\varepsilon, \kk, \n}(\Q)$ represents the risk under $\hat{\PP}^{i}_{\Q}$, then $\lim_{i \rightarrow \infty} r^{i}_{\LL,\varepsilon, \kk, \n}(\Q) =  r_{\LL,\varepsilon, \kk, \n}(\Q)$. Using an argument similar to the one presented above, we get that  $r_{\LL, k,n} \leq r^{i}_{\LL,\varepsilon, \kk, \n}(\Q)$. Taking the limit as $i$ goes to infinity on both sides, we get that  $r_{\LL, k,n} \leq r_{\LL,\varepsilon, \kk, \n}(\Q)$. This finishes the proof.
\vfill

\section{Proof of Proposition \ref{prop:rr_perf}}
\label{proof_rr_perf}

Fix $\Q$ to $\qKrrMatrix$ and $\hat{\PP}$ to be the empirical estimator given in \eqref{eq:rr_empirical}. In this case, we have that
\begin{eqnarray*}
\underset{Y^n \sim \MM\left(\qKrrMatrix\right)}{\mathbb{E}} \left|\left|\hat{\PP} - \PP\right|\right|_2^2 &=& \underset{Y^n \sim \MM\left(\qKrrMatrix\right)}{\mathbb{E}}  \left|\left| \frac{e^{\varepsilon} + \kk - 1}{e^{\varepsilon} - 1}\hat{\MM} - \frac{1}{e^{\varepsilon} - 1}  - \PP\right|\right|_2^2 \\
&=& \underset{Y^n \sim \MM\left(\qKrrMatrix\right)}{\mathbb{E}}  \left|\left| \frac{e^{\varepsilon} + \kk - 1}{e^{\varepsilon} - 1}\left(\hat{\MM} - \MM\right)\right|\right|_2^2 \\
&=& \left(\frac{e^{\varepsilon} + \kk - 1}{e^{\varepsilon} - 1}\right)^2 \underset{Y^n \sim \MM\left(\qKrrMatrix\right)}{\mathbb{E}}   \left|\left|\hat{\MM} - \MM\right|\right|_2^2 \\
& = & \left(\frac{e^{\varepsilon} + \kk - 1}{e^{\varepsilon} - 1}\right)^2 \frac{1 - \sum_{i=1}^{k} m_i^2}{n} \\
& = & \frac{1}{n}\left(\frac{e^{\varepsilon} + \kk - 1}{e^{\varepsilon} - 1}\right)^2 \left(1 - \dfrac{\sum_{i=1}^{\kk} \left\{ (e^{\varepsilon} - 1)^2 \pp_i^2 + 2 (e^{\varepsilon}-1)\pp_i + 1 \right\}}{(e^{\varepsilon} + \kk - 1 )^2} \right) \\
&=& \frac{ \left(e^{\varepsilon} + \kk - 1\right)^2  - 2(e^{\varepsilon} -1)   -\kk -(e^{\varepsilon} - 1)^2\sum_{i=1}^{\kk}\pp_i^2}{n(e^{\varepsilon} - 1)^2} \\
&=& \frac{ \left( \left(e^{\varepsilon} - 1\right) + \kk \right)^2  - 2(e^{\varepsilon} -1)   -\kk}{n(e^{\varepsilon} - 1)^2}
    -\frac{\left(e^{\varepsilon} - 1\right)^2}{n\left(e^{\varepsilon} - 1\right)^2}
    +\frac{1}{n}
    -\frac{\sum_{i=1}^{\kk}\pp_i^2}{n} \\
&=& \frac{ \left(e^{\varepsilon} - 1\right)^2 + 2 \kk\left(e^{\varepsilon} - 1\right) + k^2 - 2\left(e^{\varepsilon} - 1\right) - k - \left(e^{\varepsilon} - 1\right)^2}{n(e^{\varepsilon} - 1)^2}
    +\frac{1 - \sum_{i=1}^{\kk}\pp_i^2}{n} \\
&=& \frac{ 2 \left(\kk - 1\right)\left(e^{\varepsilon} - 1\right) + k\left(k-1\right)}{n(e^{\varepsilon} - 1)^2}
    +\frac{1 - \sum_{i=1}^{\kk}\pp_i^2}{n} \\
&=& \frac{k-1}{n}\left(\frac{2\left(e^{\varepsilon} - 1\right) + k}{\left(e^{\varepsilon} - 1\right)^2}\right)
    +\frac{1 - \sum_{i=1}^{\kk}\pp_i^2}{n},
\end{eqnarray*}
and
\begin{eqnarray*}
\underset{Y^n \sim \MM\left(\qKrrMatrix\right)}{\mathbb{E}} \left|\left|\hat{\PP} - \PP\right|\right|_1 &=&
\left(\frac{e^{\varepsilon} + \kk - 1}{e^{\varepsilon} - 1}\right) \underset{Y^n \sim \MM\left(\qKrrMatrix\right)}{\mathbb{E}}   \left|\left|\hat{\MM} - \MM\right|\right|_1 \\
& = & \left(\frac{e^{\varepsilon} + \kk - 1}{e^{\varepsilon} - 1}\right) \sum_{i=1}^{\kk} \mathbb{E} \left| m_i - \hat{m_i} \right|\\
& \approx & \left(\frac{e^{\varepsilon} + \kk - 1}{e^{\varepsilon} - 1}\right) \sum_{i = 1}^{\kk} \sqrt{\frac{2m_i(1-m_i)}{\pi n}} \\
& = & \frac{1}{e^{\varepsilon} - 1}\sum_{i = 1}^{\kk} \sqrt{\frac{2((e^{\varepsilon} - 1)\pp_i +1)((e^{\varepsilon} - 1)(1-\pp_i) + \kk - 1)}{\pi n}}.
\end{eqnarray*}
\vfill

\section{Proof of Proposition \ref{prop:rappor_perf}}
\label{proof_rappor_perf}

Fix $\Q$ to $\Q_{\text{\KRAPPOR}}$ and $\hat{\PP}$ to be the empirical estimator given in \eqref{eq:rappor_empirical}, and let $C = \frac{e^{\varepsilon/2} - 1}{e^{\varepsilon/2} + 1}$, $B = \frac{1}{e^{\varepsilon/2} + 1}$, and $A = e^{\varepsilon/2} - 1$. Then $C = BA$,  $1 - B = e^{\varepsilon/2}B$, and from Section \ref{sec:rappor} $m_i = \pp_i C + B$. Using this notation, we have that
 \begin{eqnarray*}
 \underset{Y^n \sim \MM\left(\Q_{\text{\KRAPPOR}}\right)}{\mathbb{E}} \left|\left|\hat{\PP} - \PP\right|\right|_2^2 &=& \underset{Y^n \sim \MM\left(\Q_{\text{\KRAPPOR}}\right)}{\mathbb{E}}  \left|\left| \frac{e^{\varepsilon/2} + 1}{e^{\varepsilon/2} - 1}\hat{\MM} - \frac{1}{e^{\varepsilon/2} - 1}  - \PP\right|\right|_2^2 \\
 &=& \underset{Y^n \sim \MM\left(\Q_{\text{\KRAPPOR}}\right)}{\mathbb{E}}  \left|\left| \frac{e^{\varepsilon/2} + 1}{e^{\varepsilon/2} - 1}\left(\hat{\MM} - \MM\right)\right|\right|_2^2 \\
 &=& \left(\frac{e^{\varepsilon/2} + 1}{e^{\varepsilon/2} - 1}\right)^2 \underset{Y^n \sim \MM\left(\Q_{\text{\KRAPPOR}}\right)}{\mathbb{E}}   \left|\left|\hat{\MM} - \MM\right|\right|_2^2 \\
 & = & \frac{1}{nC^2}\left(C + \kk B - \sum_{i = 1}^{\kk} (\pp_i C + B)^2 \right) \\
  & = & \frac{1}{n}\left(1 - \sum_{i=1}^{\kk} \pp_i^2 \right) + \frac{1}{nC^2}\left(C - C^2 + \kk B - \kk B^2 - 2CB\right) \\
& = &   \frac{1}{n}\left(1 - \sum_{i=1}^{\kk} \pp_i^2 \right) + \frac{1}{nBA^2} \left( A - BA^2 + \kk(1 -B) - 2BA \right) \\
& = & \frac{1}{n}\left(1 - \sum_{i=1}^{\kk} \pp_i^2 \right) + \frac{1}{n} \frac{\kk e^{\varepsilon/2}}{(e^{\varepsilon/2} - 1)^2},
\end{eqnarray*}
and
\begin{eqnarray*}
\underset{Y^n \sim \MM\left(\qKrrMatrix\right)}{\mathbb{E}} \left|\left|\hat{\PP} - \PP\right|\right|_1 &=& \left(\frac{e^{\varepsilon/2} + 1}{e^{\varepsilon/2} - 1}\right) \underset{Y^n \sim \MM\left(\Q_{\text{\KRAPPOR}}\right)}{\mathbb{E}}   \left|\left|\hat{\MM} - \MM\right|\right|_1 \\
&=& \left(\frac{e^{\varepsilon/2} + 1}{e^{\varepsilon/2} - 1}\right)  \sum_{i=1}^{\kk} \mathbb{E} \left| m_i - \hat{m_i} \right| \\
 & \approx & \left(\frac{e^{\varepsilon/2} + 1}{e^{\varepsilon/2} - 1}\right) \sum_{i = 1}^{\kk} \sqrt{\frac{2m_i(1-m_i)}{\pi n}} \\
  & = &  \sum_{i = 1}^{\kk} \sqrt{\dfrac{2((e^{\varepsilon/2} - 1)\pp_i +1)((e^{\varepsilon/2} - 1)(1-\pp_i) + 1)}{\pi n (e^{\varepsilon/2} - 1)^2}}.
\end{eqnarray*}

\section{Proof of Proposition \ref{prop:rr_high_eps_l2_win}}
\label{proof_rr_high_eps_l2_win}

We want to show that for all $\PP \in \mathbb{S}^\kk$ and all $\varepsilon \geq \ln \kk$,
\begin{equation}
\mathbb{E}\left|\left|\hat{\PP}_{\textrm{KRR}} - \PP\right|\right|_2^2  \leq \mathbb{E}\left|\left|\hat{\PP}_{\textrm{\RAPPOR}} - \PP\right|\right|_2^2,
\end{equation}
where $\hat{\PP}_{\text{KRR}}$ is the empirical estimate of $\PP$ under \KRR, $\hat{\PP}_{\textrm{\RAPPOR}}$ is the empirical estimate of $\PP$ under \KRAPPOR,  and $\hat{\PP}$ is the empirical estimator under \KRAPPOR.

From propositions \ref{prop:rr_perf} and \ref{prop:rappor_perf}, we have that
\begin{equation*}
\mathbb{E}\left|\left|\hat{\PP}_{\textrm{KRR}} - \PP\right|\right|_2^2 = \frac{1 - \sum_{i = 1}^ {\kk} \pp_i^2}{n} + \frac{\kk - 1}{n}\left( \frac{2}{e^{\varepsilon} - 1} + \frac{\kk}{(e^{\varepsilon} - 1)^2} \right) ,
\end{equation*}
and
\begin{equation*}
\mathbb{E}\left|\left|\hat{\PP}_{\textrm{\RAPPOR}} - \PP\right|\right|_2^2 = \frac{1-\sum_{i=1}^{\kk}\pp_i^2}{n} + \frac{\kk e^{\varepsilon/2}}{n(e^{\varepsilon/2} - 1)^2}.
\end{equation*}
Therefore, we just have to prove that
\begin{equation*}
(\kk - 1)\left( \frac{2}{e^{\varepsilon} - 1} + \frac{\kk}{(e^{\varepsilon} - 1)^2} \right)  \leq \frac{\kk e^{\varepsilon/2}}{(e^{\varepsilon/2} - 1)^2},
\end{equation*}
for $ \varepsilon \geq \ln \kk$. Alternatively, we can show that
\begin{equation*}
f(\varepsilon, \kk) = \frac{\kk}{\kk - 1} \left(\frac{e^\varepsilon -1}{e^{\varepsilon/2}-1}\right)^2\frac{e^{\varepsilon/2}}{2e^{\varepsilon}+k -2} \geq 1,
\end{equation*}
for $ \varepsilon \geq \ln \kk$. Observe that $f(\varepsilon, \kk)$ is an increasing function of $\varepsilon$ and therefore, it suffices to show that
\begin{equation}
f(\ln \kk, \kk) = \frac{\kk}{\kk-1} \left(\frac{\kk-1}{\sqrt{\kk}-1}\right)^2 \frac{\sqrt{\kk}}{3\kk-2} = \frac{\kk}{3\kk - 2} \frac{\sqrt{\kk}(\kk - 1)}{(\sqrt{\kk} - 1)^2} \geq 1.
\end{equation}
As a discrete function of $\kk \in \{2, 3, .... \}$, $f(\ln \kk, \kk)$ admits a unique minimum at $\kk = 7$. Therefore, we just need to verify that $f(\ln 7, 7)  > 1$. Indeed, $f(\ln 7, 7) = 3.1559 > 1$.

\section{Discrete Distribution Estimation}
\label{sec:disc_dist_est}
Consider the $(\kk-1)$-dimensional probability simplex
\begin{equation*}
\mathbb{S}^\kk = \{ \pVector=(\pp_1, ..., \pp_\kk)| \pp_i \geq 0, \sum_{i=1}^{\kk}\pp_i = 1\}.
\end{equation*}
The discrete distribution estimation problem is defined as follows. Given a vector $\pVector \in \mathbb{S}^\kk$, samples $X_1, ..., X_n$ are drawn i.i.d according to $\pVector$.
Our goal is to estimate the probability vector $\pVector$ from the observation vector $X^n = \left(X_1, ..., X_n\right)$.

An estimator $\hat{\pVector}$ is a mapping from $X^n$ to a point in $\mathbb{S}^\kk$. The performance of $\hat{\pVector}$ may be measured via a loss function $\LL$ that computes a distance-like metric between $\hat{\pVector}$ and $\pVector$.  Common loss functions include, among others, the absolute error loss $\ell_1(\pVector,\hat{\pVector}) = \sum_{i = 1}^{k} |\pp_i - \hat{\pp}_i|$ and the quadratic loss $\ell_2^2(\pVector,\hat{\pVector}) = \sum_{i = 1}^{k} (\pp_i - \hat{\pp}_i)^2$.
The choice of the loss function depends on the application; for example, $\ell_1$ loss is commonly used in classification and other machine learning applications.
Given a loss function $\LL$, the expected loss under $\hat{\pVector}$ after observing $n$ i.i.d samples is given by
\begin{equation}
r_{\LL, k,n} (\pVector,\hat{\pVector}) = \underset{X^n \sim \text{Multimial}(n, \pVector)}{\mathbb{E}} \LL(\pVector,  \hat{\pVector}) .
\end{equation}


\subsection{Maximum likelihood and empirical estimation}
\label{sec:ml_and_empirical}
In the absence of a prior on $\pVector$, a natural and commonly used estimator of $\pVector$ is the maximum likelihood (ML) estimator. The maximum likelihood estimate $\hat{\pVector}_{\text{ML}}$ of $\pVector$ is defined as
\begin{equation*}
\hat{\pVector}_{\text{ML}} = \underset{\pVector \in \mathbb{S}^\kk}{\text{argmax}}~\prob\left(X_1, ..., X_n|\pVector\right)
\end{equation*}
In this setting, it is easy to show that the maximum likelihood estimate is equivalent to the empirical estimator of $\pVector$, given by
$\hat{p}_i = T_i/n$
where $T_i$ is the frequency of element $i$. Observe that the empirical estimator is an unbiased estimator for $\pVector$ because $\mathbb{E}[\hat{\pp}_i] = \pp_i$ for any $\kk, n$, and $i$. Under maximum likelihood estimation, the $\ell_2^2$ loss is the most tractable and simplest to analyze loss function. Because $T_i \sim \text{Binomial}(\pp_i,n)$, we have $\mathbb{E}[T_i] = n\pp_i$, $\text{Var}(T_i) = n\pp_i (1 - \pp_i)$, and the expected $\ell_2^2$ loss of the empirical estimator is given by
\begin{eqnarray*}
r_{\ell_2^2,k,n}(\PP,\hat{\PP}_{\text{ML}}) &=& \mathbb{E} ||\hat{\PP}_{\text{ML}} - \PP||_2^2 = \sum_{i=1}^{\kk} \mathbb{E} \left(\frac{T_i}{n}-\pp_i\right)^2 \nonumber \\
&=& \sum_{i=1}^{\kk} \frac{\text{Var}(T_i)}{n^2} =  \frac{1 - \sum_{i=1}^{k}\pp_i^2}{n}.
\end{eqnarray*}
Let $\PP_{\text{U}} = \left(\frac{1}{\kk}, \cdots,\frac{1}{k}\right)$ and observe that
\begin{equation}
\label{eq:l2_empirical_risk}
r_{\ell_2^2, k,n}(\PP,\hat{\PP}_{\text{ML}}) \leq r_{\ell_2^2, k,n}(\PP_{\text{U}},\hat{\PP}_{\text{ML}}) = \frac{1 - \frac{1}{k}}{n}.
\end{equation}
In other words, the uniform distribution is the worst distribution for the empirical estimator under the $\ell_2^2$ loss.
From \cite{kamath2015learning}, the asymptotic performance of the empirical estimator under the $\ell_1$ loss functions is given by
\begin{equation*}
r_{\ell_1, k,n} (\PP,\hat{\PP}_{\text{ML}}) \approx  \sum_{i=1}^{\kk} \sqrt{\frac{2\pp_i(1-\pp_i)}{\pi n}},
\end{equation*}
where $a_n \approx b_n$ means $\lim_{n \rightarrow \infty} a_n/b_n = 1$. As in the $\ell_2^2$ case, notice that
  \begin{equation}
  \label{eq:l1_empirical_risk}
r_{\ell_1, k,n} (\PP,\hat{\PP}_{\text{ML}}) \leq r_{k,n}^{\ell_1} (\PP_{\text{U}},\hat{\PP}_{\text{ML}}) \approx \sqrt{\frac{2(\kk-1)}{\pi n}},
\end{equation}
for any $\PP \in \mathbb{S}^\kk$. In other words, the uniform distribution is the worst distribution for the empirical estimator under the $\ell_1$ loss as well. Observe that the $\ell_1$ loss scales as $\sqrt{k/n}$ whereas the $\ell_2^2$ loss scales as $1/n$.

\subsection{Minimax estimation}
Another popular estimator that is widely studied in the absence of a prior is the minimax estimator $\hat{\PP}_{\text{MM}}$. The minimax estimator minimizes the expected loss under the worst distribution $\PP$:
\begin{equation}
\hat{\PP}_{\text{MM}} = \argmin_{\hat{\PP}} \max_{\PP \in \mathbb{S}^\kk} \underset{X^n \sim \PP}{\mathbb{E}} \LL(\PP,  \hat{\PP}).
\end{equation}
The minimax risk is therefore defined as
\begin{equation*}
r_{\LL, k,n} = \min_{\hat{\PP}} \max_{\PP \in \mathbb{S}^\kk} \underset{X^n \sim \PP}{\mathbb{E}} \LL(\PP,  \hat{\PP}).
\end{equation*}
For the $\ell_2^2$ loss, it is shown in \cite{lehmann1998theory} that
\begin{equation}
\hat{\pp}_i = \frac{\frac{\sqrt{n}}{k} + \sum_{j=1}^{n}\unity_{\{X_j = i\}}}{\sqrt{n} + n} = \frac{\frac{\sqrt{n}}{k} + T_i}{\sqrt{n} + n},
\end{equation}
is the minimax estimator, and that the minimax risk is
\begin{equation}
\label{eq:l2_minimax_risk}
r_{\ell_2^2, k,n} = \frac{1 - \frac{1}{k}}{(\sqrt{n}+1)^2}.
\end{equation}
Observe that unlike the empirical estimator, the minimax estimator is not even asymptotically unbiased. Moreover, it improves on the empirical estimator only slightly (compare Equations \eqref{eq:l2_empirical_risk} to \eqref{eq:l2_minimax_risk}), increasing the the denominator from $n$ to $n + 2\sqrt{n} + 1$ under the worst case distribution (the uniform distribution). This explains why the minimax estimator is almost never used in practice.

The minimax estimator under $\ell_1$ loss is not known. However, the minimax risk is known for the case when $\kScalar$ is fixed and $n$ is increased. In  this case, it is shown in \cite{kamath2015learning} that
\begin{equation}
  \label{eq:l1_minimax_risk}
r_{\ell_1,k,n}  = \sqrt{\frac{2(\kk-1)}{\pi n}} + O\left(\frac{1}{n^{3/4}}\right).
\end{equation}
Comparing Equations \eqref{eq:l1_empirical_risk} to \eqref{eq:l1_minimax_risk}, we see that the worst case loss under the empirical estimator is again roughly as good as the minimax risk.

\section{Maximum Likelihood Estimation for $\kScalar$-ary Mechanisms}

\subsection{\KRR}
\label{sec:rr_ml_decoder}
\begin{proposition}
\label{prop:max_est_krr}
The maximum likelihood estimator of $\PP$ under \KRR is given by
\begin{equation}
\hat{\pp}_i = \left[\frac{T_i}{\lambda} - \frac{1}{e^{\varepsilon} - 1}\right]^{+},
\end{equation}
where $[x]^+ = \max(0,x)$, $T_i$ is the frequency of element $i$ calculated from $Y^n$, and $\lambda$ is chosen so that
\begin{equation}
\sum_{i=1}^{\kk} \left[\frac{T_i}{\lambda} - \frac{1}{e^{\varepsilon} - 1}\right]^{+}  = 1.
\end{equation}
Moreover, finding $\lambda$ can be done in $O(\kk\log\kk)$ steps.
\end{proposition}
The proof of the above proposition is provided in Supplementary Section \ref{proof_max_est_krr}.

\subsection{Proof of Proposition \ref{prop:max_est_krr}}
\label{proof_max_est_krr}
The maximum likelihood estimator under \KRR is the solution to
\begin{equation*}
\hat{\pVector}_{\text{ML}} = \underset{\pVector \in \mathbb{S}^\kk}{\text{argmax}}~\prob\left(Y_1, ..., Y_n|\pVector\right),
\end{equation*}
where the $Y_i$'s are the outputs of \KRR. Since the $\log(.)$ function is a monotonic function, the above maximum likelihood estimation problem is equivalent to
\begin{equation*}
\hat{\pVector}_{\text{ML}} = \underset{\pVector \in \mathbb{S}^\kk}{\text{argmax}}~ \log \prob\left(Y_1, ..., Y_n|\pVector\right).
\end{equation*}
Given that
\begin{eqnarray*}
\prob\left(Y_1, ..., Y_n|\pVector\right) &=& \prod_{i=1}^{n} \prob \left(Y_i | \pVector \right) \\
&=&  \prod_{i=1}^{n} \left( \sum_{j = 1}^{\kk} \qKrrMatrix(Y_i|X_i = j)\pp_j \right),
\end{eqnarray*}
we have that
\begin{equation*}
\log \prob\left(Y_1, ..., Y_n|\pVector\right) = \sum_{i=1}^{n} \log \left( \sum_{j = 1}^{\kk} \qKrrMatrix(Y_i|X_i = j)\pp_j \right).
\end{equation*}
Observe that
\begin{eqnarray}
\sum_{j = 1}^{\kk} \qKrrMatrix(Y_i|X_i = j)\pp_j & = & \qKrrMatrix(Y_i|X_i = Y_i)\pp_{Y_i} + \sum_{j \neq Y_i} \qKrrMatrix(Y_i|X_i = j)\pp_j \\
& =& \frac{e^\varepsilon}{e^\varepsilon+\kk - 1}\pp_{Y_i} + \frac{1}{e^\varepsilon+\kk - 1}(1 - \pp_{Y_i}) \\
& =& \frac{1}{e^\varepsilon+\kk - 1}\left( (e^\varepsilon - 1)\pp_{Y_i}  +1 \right),
\end{eqnarray}
and therefore,
\begin{equation*}
\sum_{i=1}^{n} \log \left( \sum_{j = 1}^{\kk} \qKrrMatrix(Y_i|X_i = j)\pp_j \right) = \sum_{i = 1}^{k} T_i \log \left( \frac{1}{e^\varepsilon+\kk - 1}\left( (e^\varepsilon - 1)\pp_{i}  +1 \right) \right),
\end{equation*}
where $T_i$ is the number of $Y$'s that are equal to $i$ (i.e., the frequency of element $i$ in the observed sequence $Y^n$). Thus, the maximum likelihood estimation problem under \KRR is equivalent to
\begin{equation*}
\hat{\pVector}_{\text{ML}} = \underset{\pVector \in \mathbb{S}^\kk}{\text{argmax}}~\sum_{i = 1}^{k} T_i \log \left( (e^\varepsilon - 1)\pp_{i}  +1 \right).
\end{equation*}
The above constrained optimization problem is a convex optimization problem that is well studied in the literature under the rubric of water-filling algorithms. From \cite{boyd2004convex}, the solution to this problem is given by
\begin{equation*}
\hat{\pp}_i = \left[\frac{T_i}{\lambda} - \frac{1}{e^{\varepsilon} - 1}\right]^{+},
\end{equation*}
where $[x]^+ = \max(0,x)$ and $\lambda$ is chosen so that
\begin{equation*}
\sum_{i=1}^{\kk} \left[\frac{T_i}{\lambda} - \frac{1}{e^{\varepsilon} - 1}\right]^{+}  = 1.
\end{equation*}
Given the $T_i$'s, $\pVector$ is computed according to the empirical estimator. If all the $\hat{\pp}_i$'s are non-negative, then the maximum likelihood estimate is the same as the empirical estimate. If not, $\hat{\pVector}$ is sorted, its negative entries are zeroed out, and lambda is computed according to the above equation. Given lambda, a new $\hat{\pVector}$  can be computed and the above process can be repeated until all the entries of $\hat{\pVector}$ are non-negative. Notice that sorting happens once and the process is repeated at most $\kk - 1$ times. Therefore, the computational complexity of this algorithm is upper bounded by $\kk \log \kk + \kk$ which is $O(\kk\log\kk)$.
\subsection{\KRAPPOR}
\begin{proposition}
\label{prop:max_est_rappor}
The maximum likelihood estimator of $\PP$ under \KRAPPOR is
\begin{equation*}
\begin{aligned}
 \underset{\PP \in \mathbb{S}^\kk}{\text{argmax}} \sum_{j = 1}^{\kk} (n - T_j) \log\left( (1- \delta) - (1 -2\delta)\pp_j\right)  \\
~~~~~~~~~~~~~~~~~~~~~~~~~~~~~  + T_j \log\left((1-2\delta)\pp_j + \delta\right)
\end{aligned}
\end{equation*}
where $T_j = \sum_{i = 1}^{n} Y_i^{(j)}$ and $\delta = 1/(e^{\varepsilon/2} + 1)$.
\end{proposition}
The proof of the above proposition is provided in Supplementary Section \ref{proof_max_est_rappor}. Observe that unlike \KRR, a $\kScalar$-dimensional convex program has to be solved in this case to determine the maximum likelihood estimate of $\PP$.

\subsection{Proof of Proposition \ref{prop:max_est_rappor}}
\label{proof_max_est_rappor}
The maximum likelihood estimator under \KRAPPOR is the solution to
\begin{equation*}
\hat{\pVector}_{\text{ML}} = \underset{\pVector \in \mathbb{S}^\kk}{\text{argmax}}~\prob\left(Y_1, ..., Y_n|\pVector\right),
\end{equation*}
where the $Y_i$'s are the outputs of \KRAPPOR. Since the $\log(.)$ function is a monotonic function, the above maximum likelihood estimation problem is equivalent to
\begin{equation*}
\hat{\pVector}_{\text{ML}} = \underset{\pVector \in \mathbb{S}^\kk}{\text{argmax}}~ \log \prob\left(Y_1, ..., Y_n|\pVector\right).
\end{equation*}
Recall that under \KRAPPOR,  $Y_i = [ Y^{(1)}_i, \cdots, Y^{(\kk)}_i]$ is a $\kk$-dimensional binary vector, which implies that
\begin{equation}
\prob(Y_i^{(j)} = 1) =  \left(\frac{e^{\varepsilon/2} - 1}{e^{\varepsilon/2} + 1}\right) \pp_j  + \frac{1}{e^{\varepsilon/2} + 1},
\end{equation}
for all $i \in \{1, \cdots, n\}$ and $j \in \{1, \cdots, \kk\}$. Therefore,
\begin{eqnarray*}
\log \prob\left(Y_1, ..., Y_n|\pVector\right) & =& \log \prod_{i=1}^{n}\prod_{j=1}^{\kk} \left(Y_i^{(j)}(\pp_j(1 - \delta) + (1 -\pp_j)\delta) + (1 - Y_i^{(j)})(\pp_j\delta + (1-\pp_j)(1 - \delta)) \right) \\
& = & \sum_{i=1}^{n}\sum_{j=1}^{\kk} \log  \left(Y_i^{(j)}(\pp_j(1 - \delta) + (1 -\pp_j)\delta) + (1 - Y_i^{(j)})(\pp_j\delta + (1-\pp_j)(1 - \delta)) \right) \\
& = & \sum_{i=1}^{n}\sum_{j=1}^{\kk} \log  \left( (1 - 2\delta)(2Y_i^{(j)} -1) \pp_j - Y_i^{(j)}(1-2\delta) + (1-\delta) \right),
\end{eqnarray*}
where $\delta = 1/(1+ e^{\varepsilon/2})$. Therefore, under \KRAPPOR, the maximum likelihood estimation problem is given by
\begin{equation*}
 \underset{\PP \in \mathbb{S}^\kk}{\text{argmax}} \sum_{j = 1}^{\kk} (n - T_j) \log\left( (1- \delta) - (1 -2\delta)\pp_j\right)+ T_j \log\left((1-2\delta)\pp_j + \delta\right)
\end{equation*}
where $T_j = \sum_{i = 1}^{n} Y_i^{(j)}$.

\section{Conditions for Accurate Decoding under \KRR}

For accurate decoding, we must satisfy three criteria:
(i) $\kScalar$ and $\CScalar$ must be large enough that the input strings to be distinguishable,
(ii) $\kScalar$ and $\CScalar$ must be large enough that the linear system in \eqref{eq:rr_strings_regression} is not underconstrained, and
(iii) $\nScalar$ must be large enough that the variance on estimated probability vector $\hat{\pVector}$ is small.

Let us first consider string distinguishability.  Each string $\sScalar \in \sSet$
is associated with a $\CScalar$-tuple of hashes it can produces in the various cohorts:
$\hash_\cSet^{(\kScalar)}(\sScalar) = \langle \hash_1^{(\kScalar)}(\sScalar), \hash_2^{(\kScalar)}(\sScalar), \cdots, \hash_\CScalar^{(\kScalar)}(\sScalar) \rangle \in \xSet ^ \CScalar$.
Two strings $\sScalar_\iScalar \in \sSet$ and $\sScalar_\jScalar \in \sSet$ are distinguishable from one
another under the encoding scheme if $\hash_\cSet^{(\kScalar)}(\sScalar_\iScalar) \neq \hash_\cSet^{(\kScalar)}(\sScalar_\jScalar)$,
and a string $\sScalar$ is distinguishable within the set $\sSet$ if
$\hash_\cSet^{(\kScalar)}(\sScalar) \neq \hash_\cSet^{(\kScalar)}(\sScalar_\jScalar) \forall \sScalar_\jScalar \in \sSet \backslash {\sScalar}$.

Because $\hash_\cSet^{(\kScalar)}(\sScalar)$ is distributed uniformly over $\xSet^\CScalar$,
$\prob(\hash_\cSet^{(\kScalar)}(\sScalar) = \xVector_\cSet) \approx \frac{1}{\kScalar^\CScalar}$
for all $\xVector_\cSet \in \xSet^\CScalar$.  It follows that the probability of two strings being distinguishable is also $\frac{1}{\kScalar^\CScalar}$. Furthermore, the probability that exactly one string from $\sSet$ produces the hash tuple $\xVector_\cSet$ is:

\begin{equation*}
\text{Binomial}(1; \frac{1}{\kScalar^\CScalar}, \SScalar)
=  \frac{\SScalar (\kScalar^\CScalar-1)^{\SScalar-1}}{(\kScalar^\CScalar)^{\SScalar}}
\end{equation*}

Thus, the expected number of $\xVector_\cSet \in \xSet^\CScalar$ associated with exactly one
string in $\sSet$, which is also the expected number of distinguishable strings
in a set $\sSet$ is:
\begin{equation}
\sum_{\xVector_\cSet \in \cY^\CScalar} \left(\frac{\SScalar(\kScalar^\CScalar-1)^{\SScalar-1}}{(\kScalar^\CScalar)^{\SScalar}}\right) = \SScalar \left(\frac{\kScalar^\CScalar-1}{\kScalar^\CScalar}\right)^{\SScalar-1}
\end{equation}
and the probability that a string $\sScalar$ is distinguishable within the set $\sSet$ is $\left(\frac{\kScalar^\CScalar-1}{\kScalar^\CScalar}\right)^{\SScalar-1}$.

Consider a probability distribution $\pVector \in \simplex^{\SScalar}$.
The expected recoverable probability mass is the the mass associated with the distinguishable
strings within the set $\sSet$ is
$
\sum_{\sScalar \in \sSet} \pScalar_\sScalar \left(\frac{\kScalar^\CScalar-1}{\kScalar^\CScalar}\right)^{\SScalar-1}
=  \left(\frac{\kScalar^\CScalar-1}{\kScalar^\CScalar}\right)^{\SScalar-1}
$
Therefore, if we hope to recover at least $\PTScalar$ of the probability mass,
we require $\left(\frac{\kScalar^\CScalar-1}{\kScalar^\CScalar}\right)^{\SScalar-1} \ge \PTScalar$, or equivalently, $\kScalar^\CScalar \ge \frac{1}{1 - \PTScalar^\frac{1}{\SScalar-1}}$.

Now consider ensuring that the linear system in
\eqref{eq:rr_strings_regression} is not underconstrained.  The system has
$\SScalar$ variables and $\kScalar \CScalar$ independent equations.  Thus,
the system is not underconstrained so long as $\kScalar \CScalar \ge \SScalar$.



\newpage
\section{Supplementary Figures}

\begin{figure}[h!]
\centering
\includegraphics[width=0.7\linewidth]{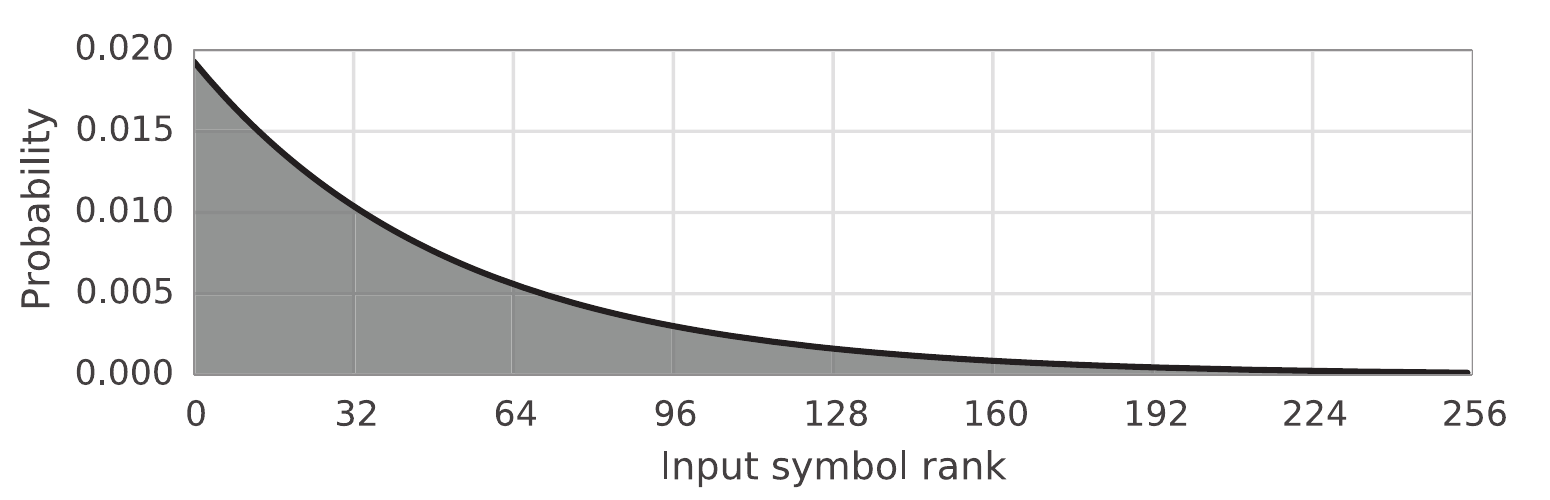}
\caption{The true input distribution $\pVector$ for open-set and closed-set
experiments in sections~\ref{sec:sim_results} and \ref{sec:open_alphabets} is
the geometric distribution with mean at $|\text{input alphabet}|/5$,
truncated and renormalized.  In the $\kScalar$-ary experiments of
Section~\ref{sec:sim_results}, the input alphabet is size $\kScalar$;
in the open alphabet experiments of Section~\ref{sec:open_alphabets},
the input alphabet is size $\SScalar = 256$.}
\label{fig:geometric_ground_truth}
\end{figure}

\begin{figure*}[h]
\centering
\begin{tabular}{m{.2in}cc}
&
\begin{subfigure}[b]{.45\linewidth}
\includegraphics[width=\linewidth]{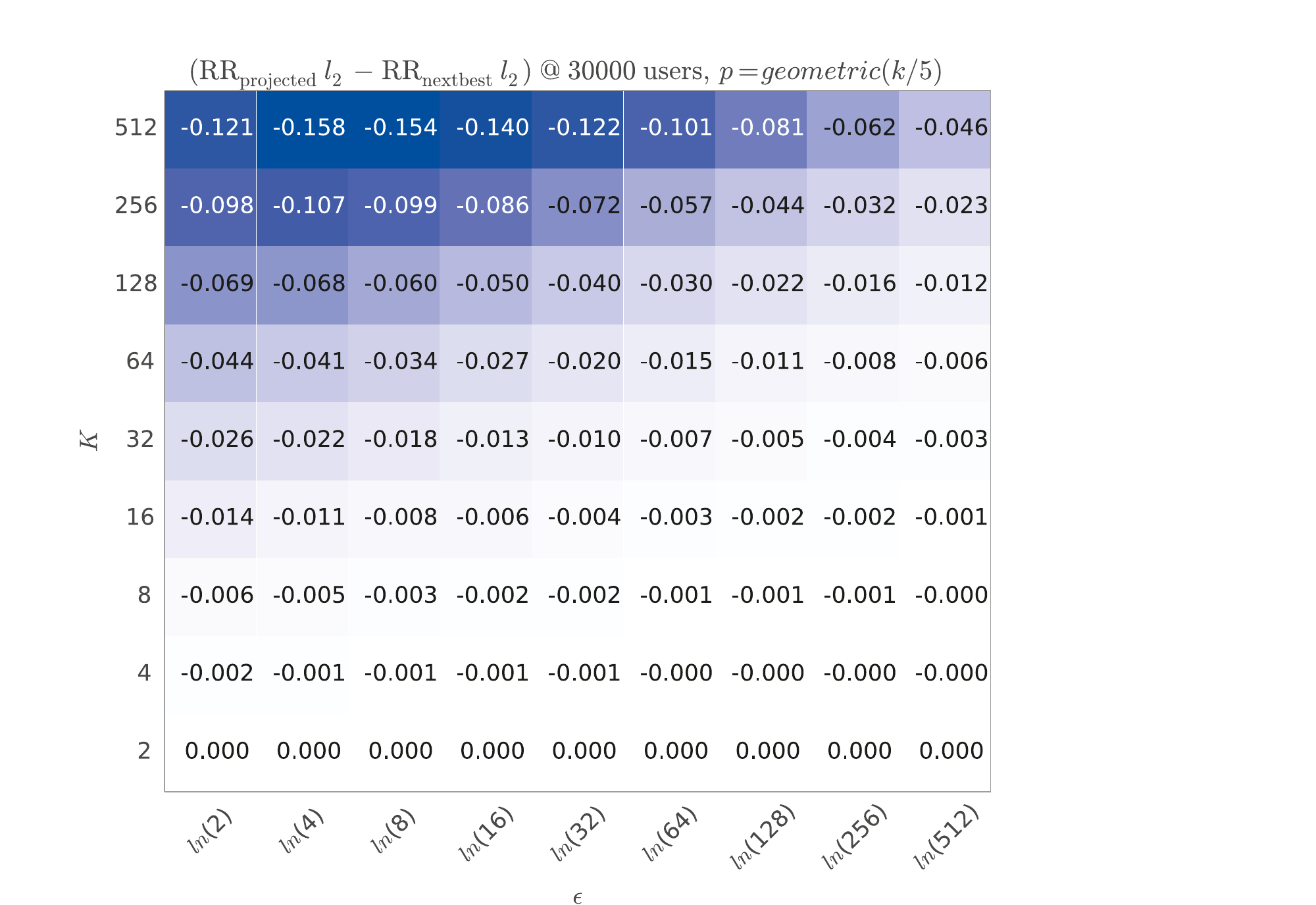}
\end{subfigure}
&
\begin{subfigure}[b]{.45\linewidth}
\includegraphics[width=\linewidth]{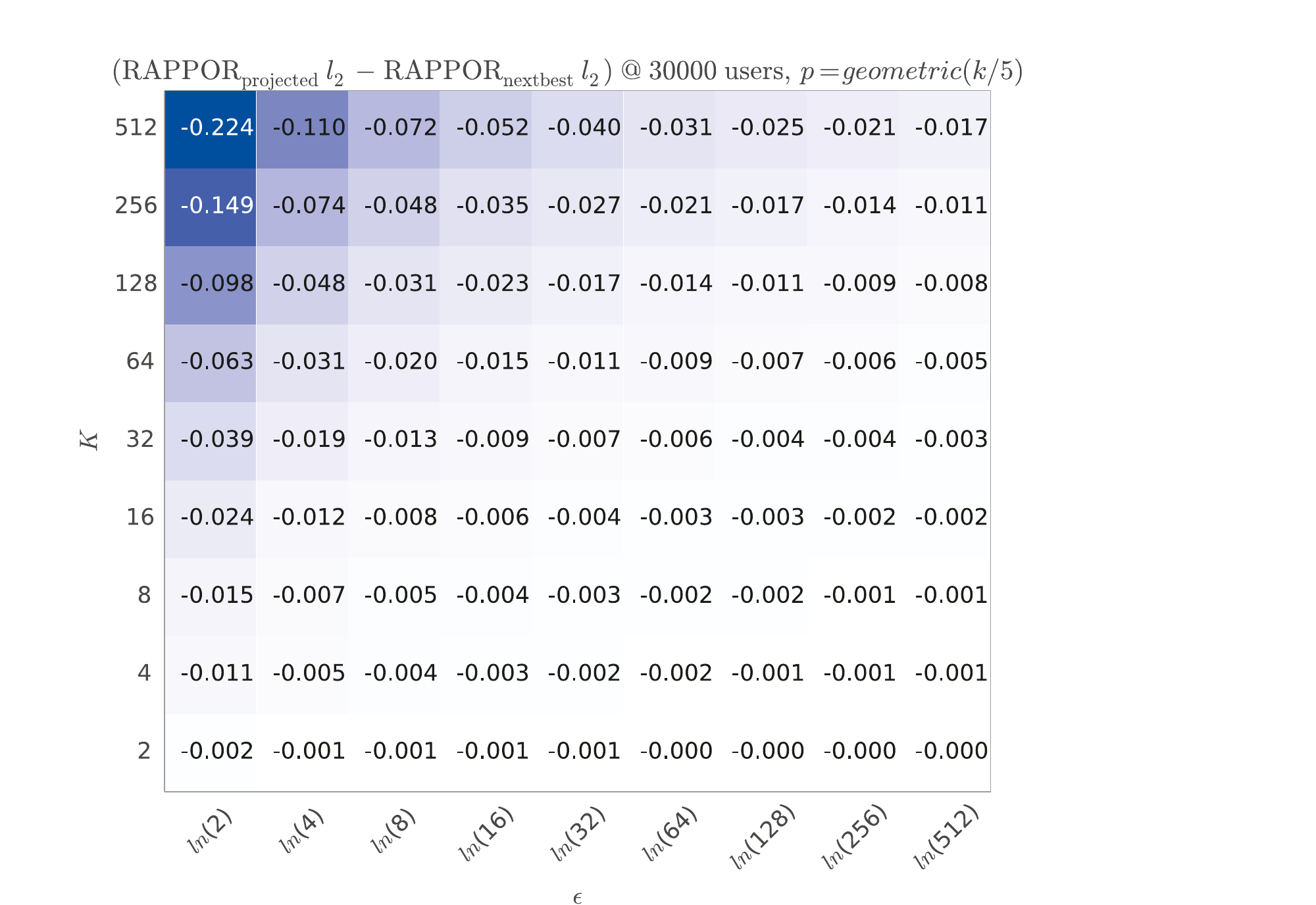}
\end{subfigure}
\end{tabular}
\caption{The improvement in $\ell_2$ decoding of the projected \KRR decoder (left) and projected \KRAPPOR decoder (right). This figure demonstrates that the same patterns hold in $\ell_2$ as in $\ell_1$ for the conditions shown in Figure~\ref{fig:decoders}.}
\label{fig:l2_decoders}
\end{figure*}

\begin{figure*}[h]
\centering
\begin{tabular}{m{.2in}cc}
& 100 users & 30000 users
\\
\parbox[t]{2mm}{\rotatebox{90}{\rlap{\hspace{.6in}geometric}}}
&
\begin{subfigure}[b]{.45\linewidth}
\includegraphics[width=\linewidth]{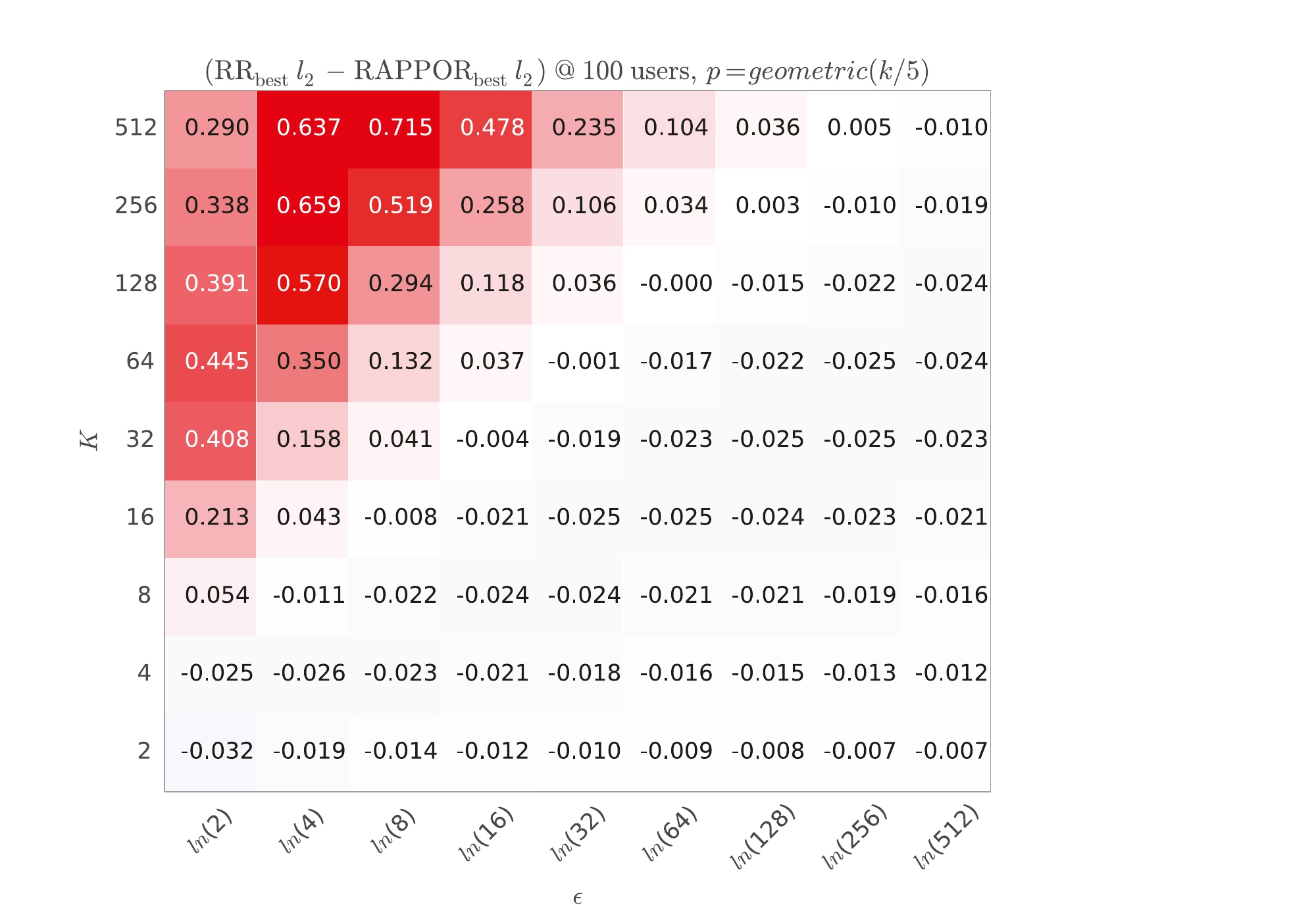}
\end{subfigure}
&
\begin{subfigure}[b]{.45\linewidth}
\includegraphics[width=\linewidth]{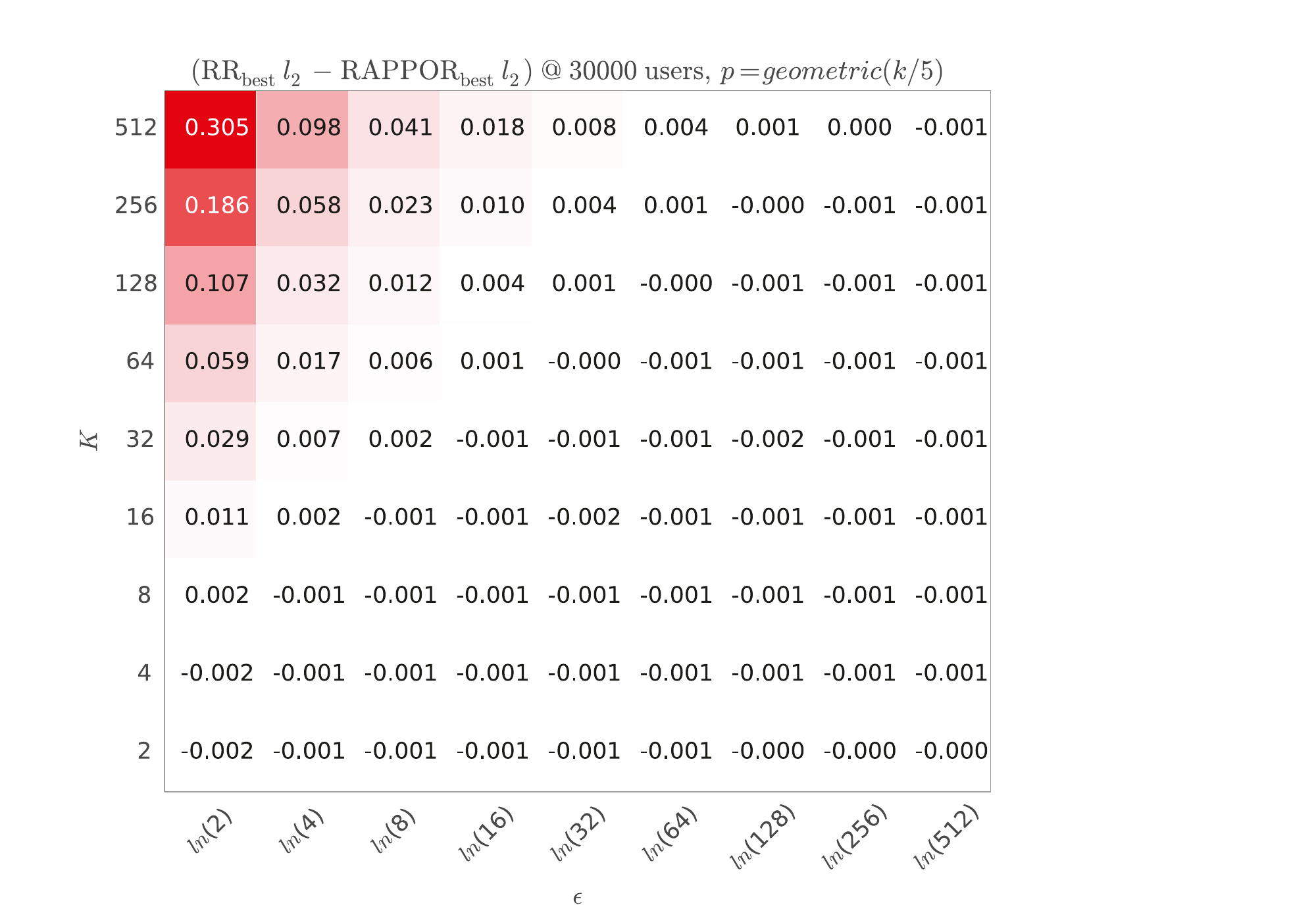}
\end{subfigure}
\\
\parbox[t]{2mm}{\rotatebox{90}{\rlap{\hspace{.6in}dirichlet}}}
&
\begin{subfigure}[b]{.45\linewidth}
\includegraphics[width=\linewidth]{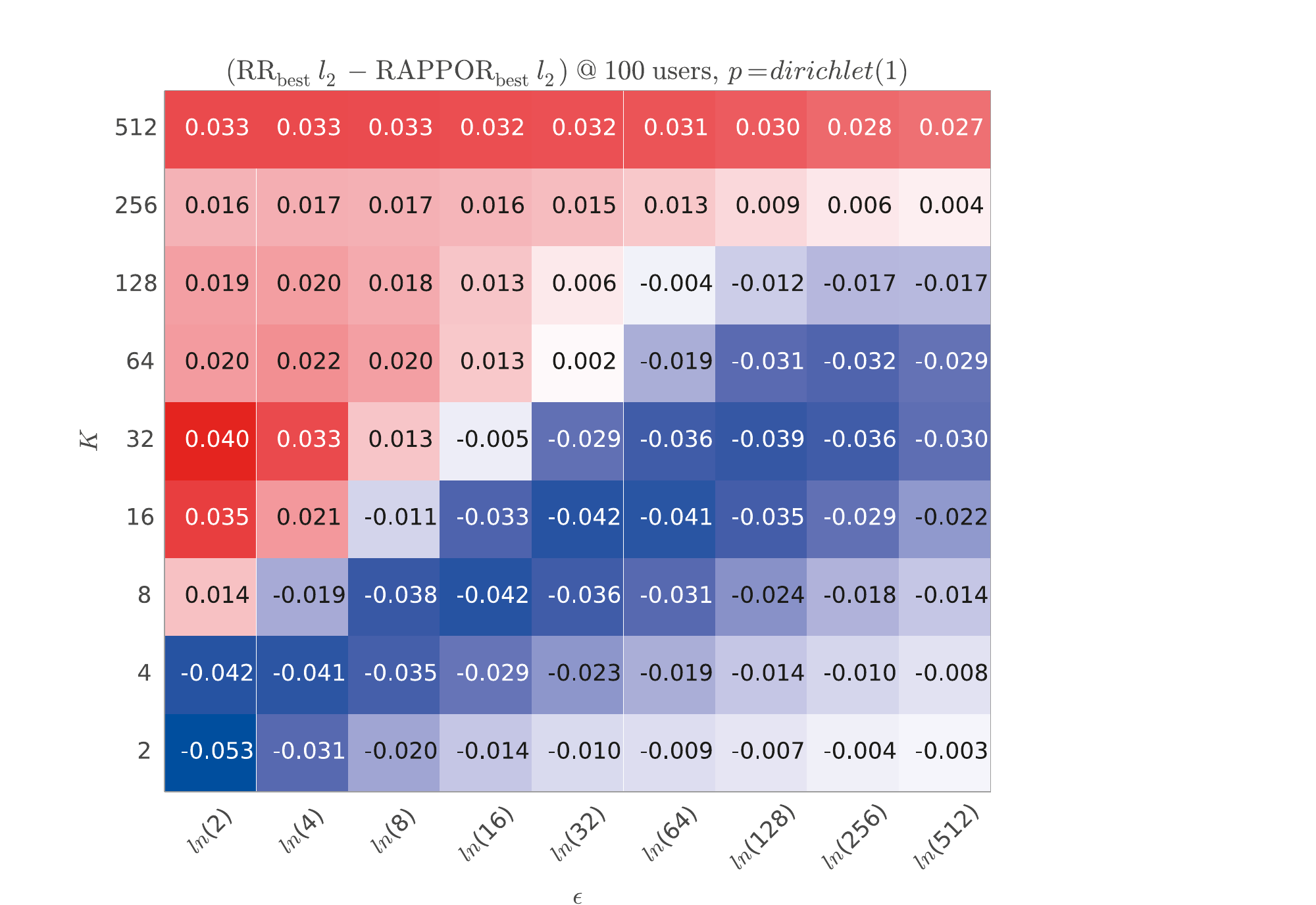}
\end{subfigure}
&
\begin{subfigure}[b]{.45\linewidth}
\includegraphics[width=\linewidth]{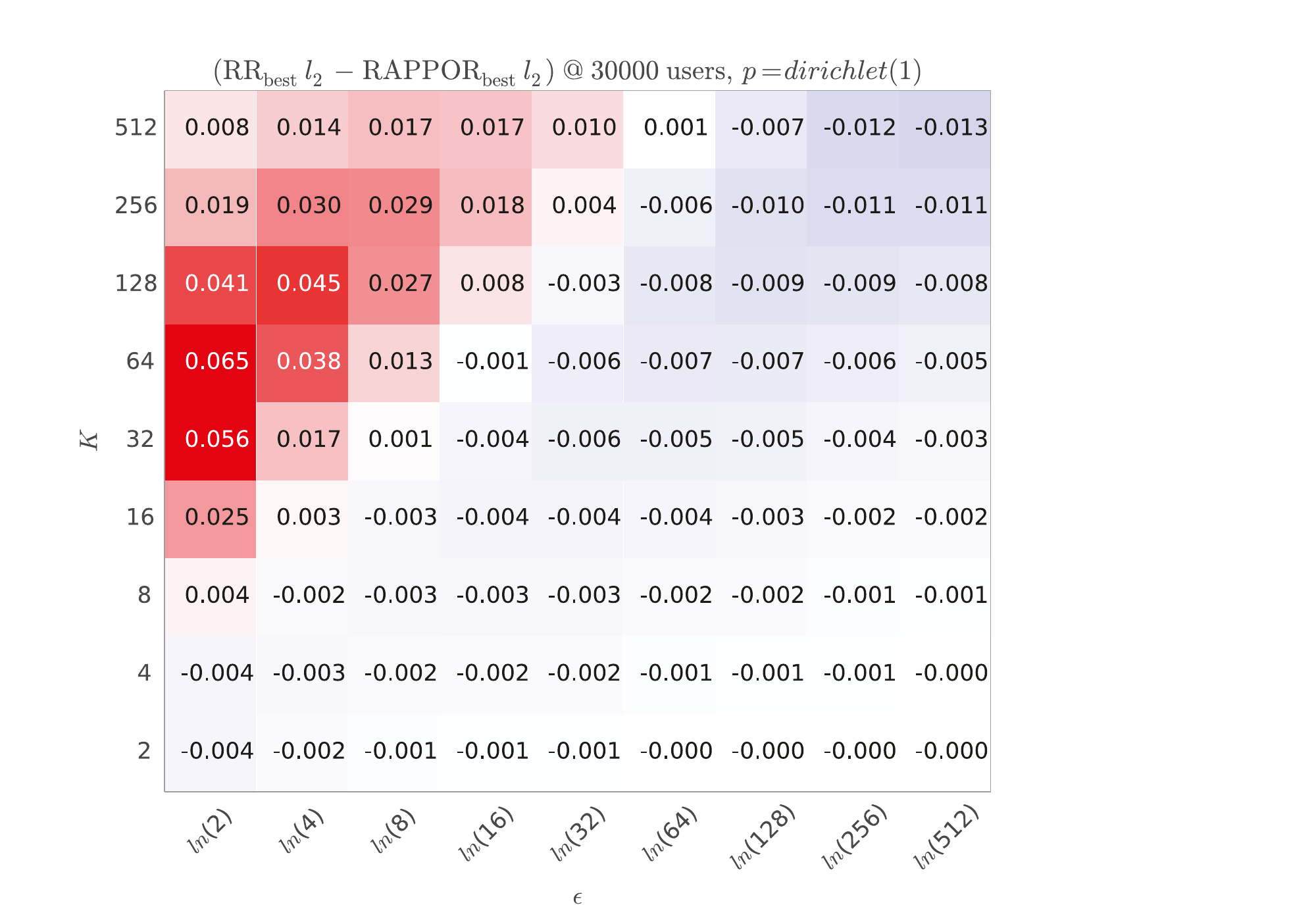}
\end{subfigure}
\end{tabular}
\caption{The improvement (negative values, blue) of the best \KRR decoder over the best \KRAPPOR decoder varying the size of the alphabet $\kScalar$ (rows) and privacy parameter $\varepsilon$ (columns). This figure demonstrates that the same patterns hold in $\ell_2$ as in $\ell_1$ for the conditions shown in Figure~\ref{fig:krr_vs_rappor}.}
\label{fig:l2_krr_vs_rappor}
\end{figure*}
\clearpage

\begin{figure*}
\centering

\begin{tabular}{cc}

\begin{subfigure}[b]{.45\linewidth}
\includegraphics[width=\linewidth]{figures/open_set.pdf}
\caption{Full $\varepsilon$ range.}
\end{subfigure}
&
\begin{subfigure}[b]{.45\linewidth}
\includegraphics[width=\linewidth]{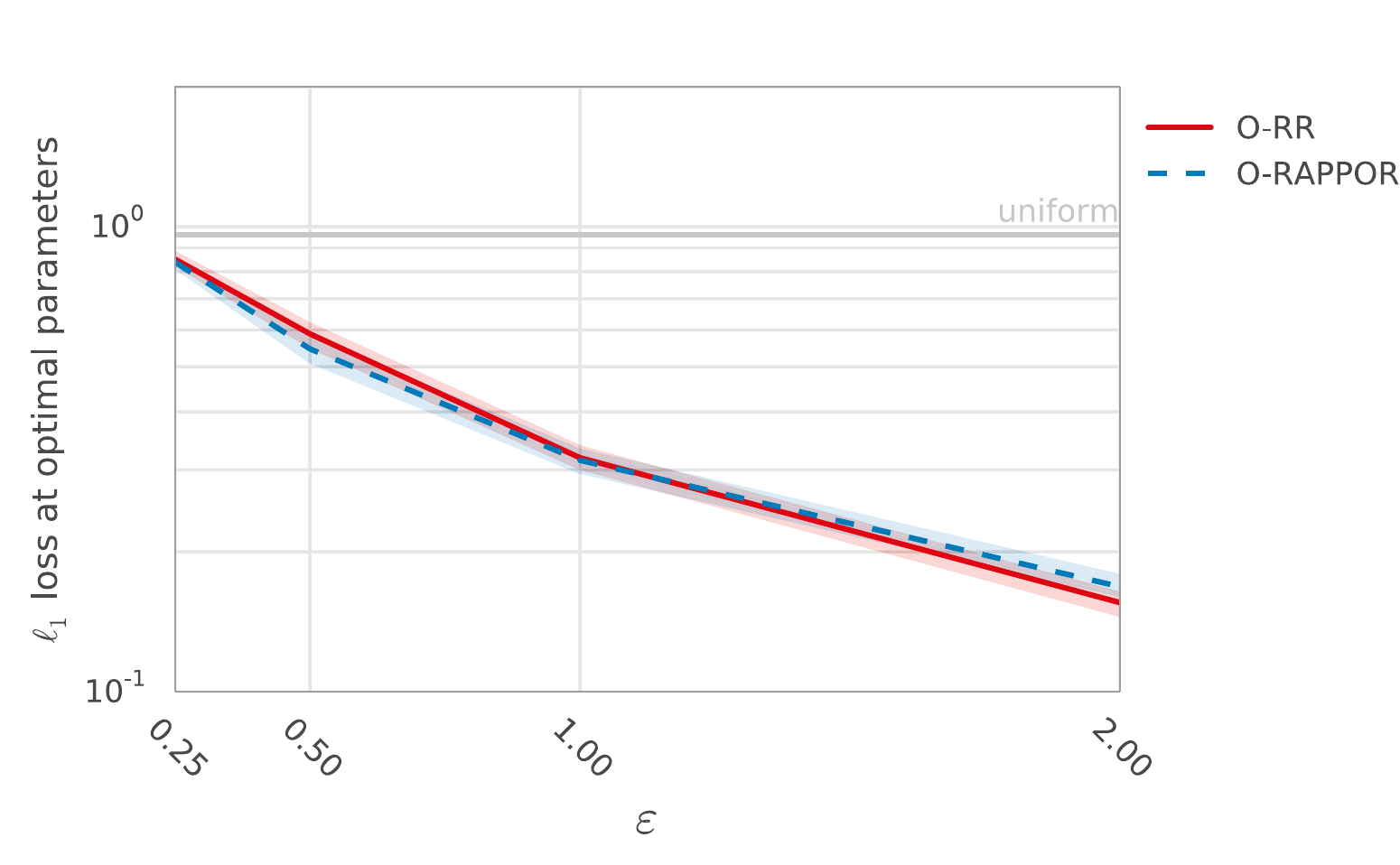}
\caption{Low $\varepsilon$ range.}
\end{subfigure}
\end{tabular}
\caption{
$\ell_1$ loss when decoding open alphabets using the \ORR and \ORAPPOR
for $n=10^6$ users with input drawn from an
alphabet of $S=256$ symbols under a geometric distribution with mean=$S/5$,
as depicted in Figure~\protect\ref{fig:geometric_ground_truth}.
Free parameters are set via grid search over
$\kScalar \in [2, 4, 8, \ldots, 2048, 4096]$,
$\cScalar \in [1, 2, 4, \ldots, 512, 1024]$,
$\hScalar \in [1, 2, 4, 8, 16]$
to minimize the median loss over 50 samples at the given $\varepsilon$
value.  Lines show median $\ell_1$ loss while the (narrow) shaded regions indicate 90\%
confidence intervals (over 50 samples).
Baselines indicate expected loss from (1) using an empirical estimator
directly on the input $\sVector$ and (2) using the uniform distribution
as the $\hat{\pVector}$ estimate.
}

\label{fig:open_set_zoom}
\end{figure*}

\begin{figure*}
\centering

\begin{tabular}{cc}

\begin{subfigure}[b]{.45\linewidth}
\includegraphics[width=\linewidth]{figures/closed_set.pdf}
\caption{Full $\varepsilon$ range.}
\end{subfigure}
&
\begin{subfigure}[b]{.45\linewidth}
\includegraphics[width=\linewidth]{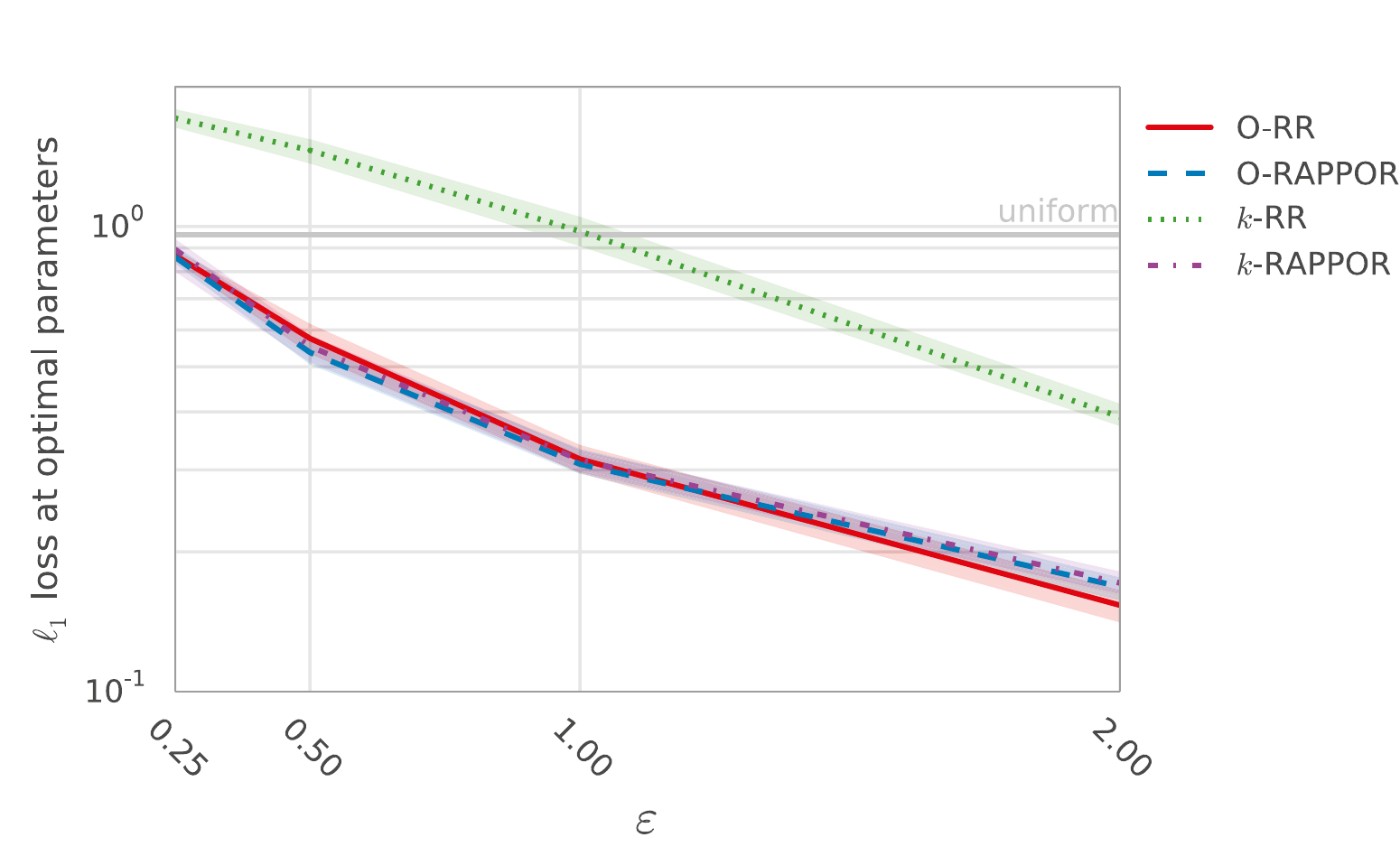}
\caption{Low $\varepsilon$ range.}
\end{subfigure}
\end{tabular}
\caption{
$\ell_1$ loss when decoding decoding a known alphabet
using the \ORR and \ORAPPOR (via permutative perfect hash functions)
for $n=10^6$ users with input drawn from an
alphabet of $S=256$ symbols under a geometric distribution with mean=$S/5$,
as depicted in Figure~\protect\ref{fig:geometric_ground_truth}.
Free parameters are set via grid search over
$\kScalar \in [2, 4, 8, \ldots, 2048, 4096]$,
$\cScalar \in [1, 2, 4, \ldots, 512, 1024]$,
$\hScalar \in [1, 2, 4, 8, 16]$
to minimize the median loss over 50 samples at the given $\varepsilon$
value.  Lines show median $\ell_1$ loss while the (narrow) shaded regions indicate 90\%
confidence intervals (over 50 samples).
Note that the \KRAPPOR and \ORAPPOR lines in
\protect\subref*{fig:orr_vs_rappor:closed} are nearly indistinguishable.
Baselines indicate expected loss from (1) using an empirical estimator
directly on the input $\sVector$ and (2) using the uniform distribution
as the $\hat{\pVector}$ estimate.
}

\label{fig:closed_set_zoom}
\end{figure*}

\clearpage

\begin{figure*}
\centering

\begin{subfigure}[b]{.45\linewidth}
\includegraphics[width=\linewidth]{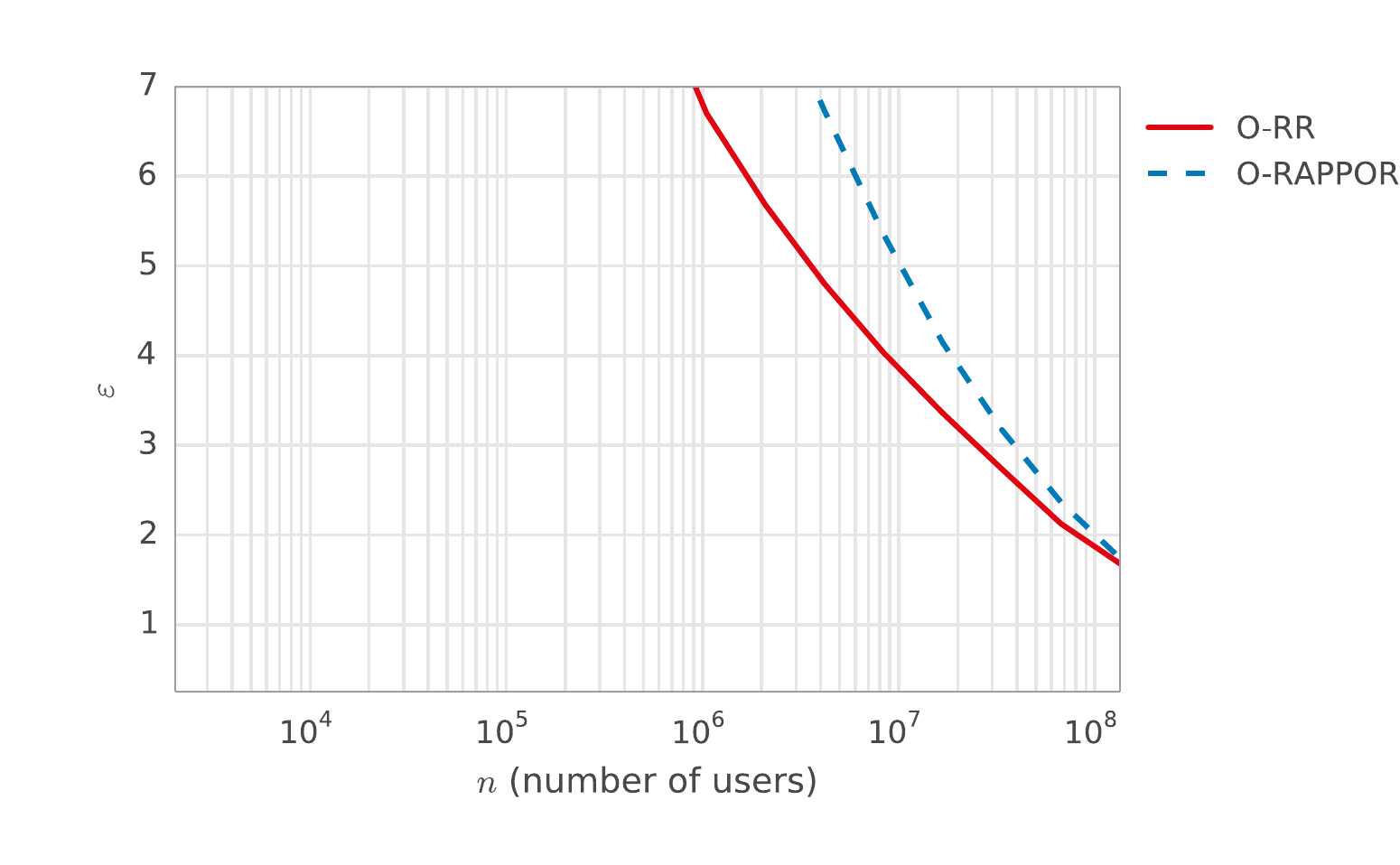}
\caption{$\ell_1=0.02$}
\end{subfigure}
\begin{subfigure}[b]{.45\linewidth}
\includegraphics[width=\linewidth]{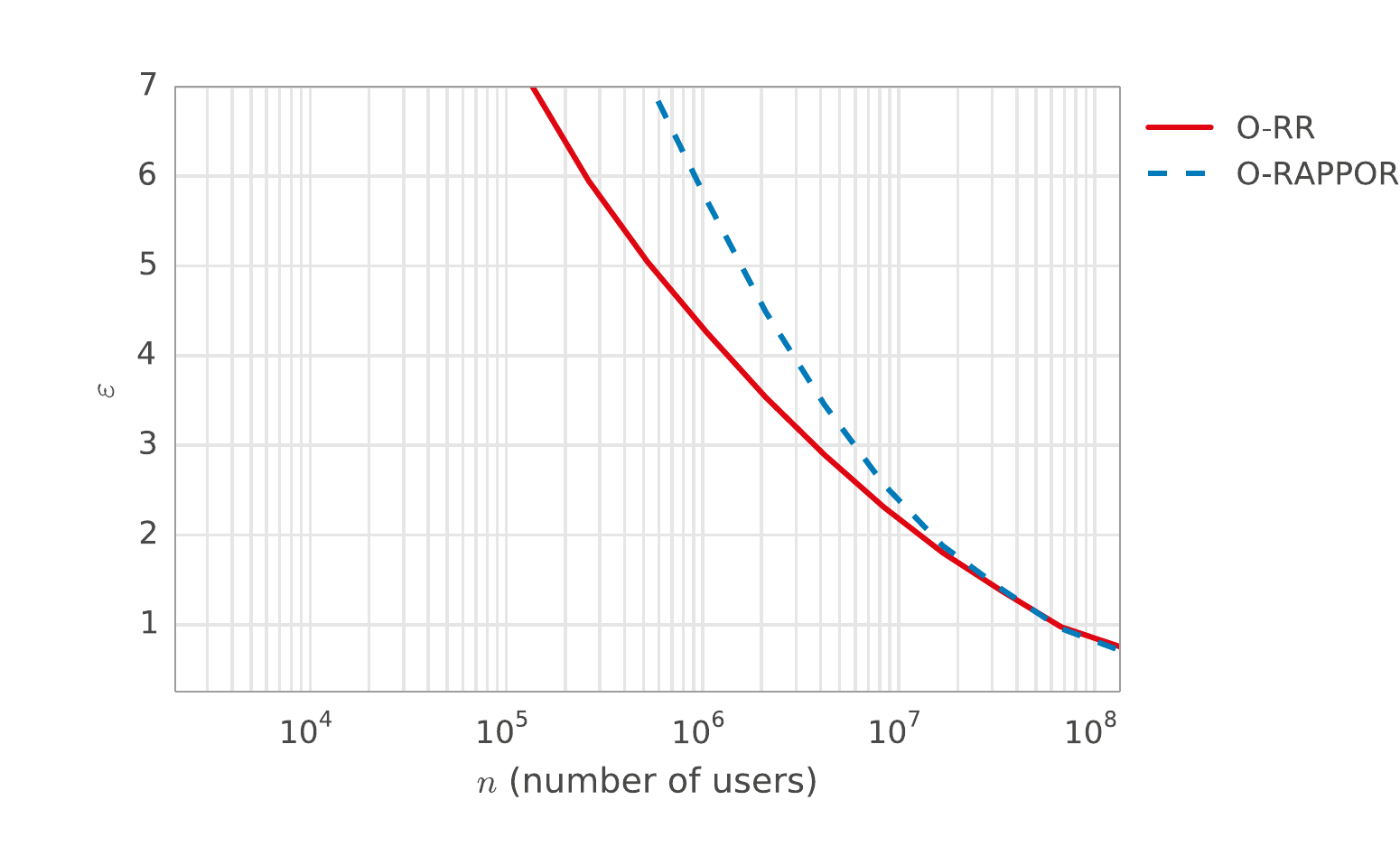}
\caption{$\ell_1=0.05$}
\end{subfigure}

\begin{subfigure}[b]{.45\linewidth}
\includegraphics[width=\linewidth]{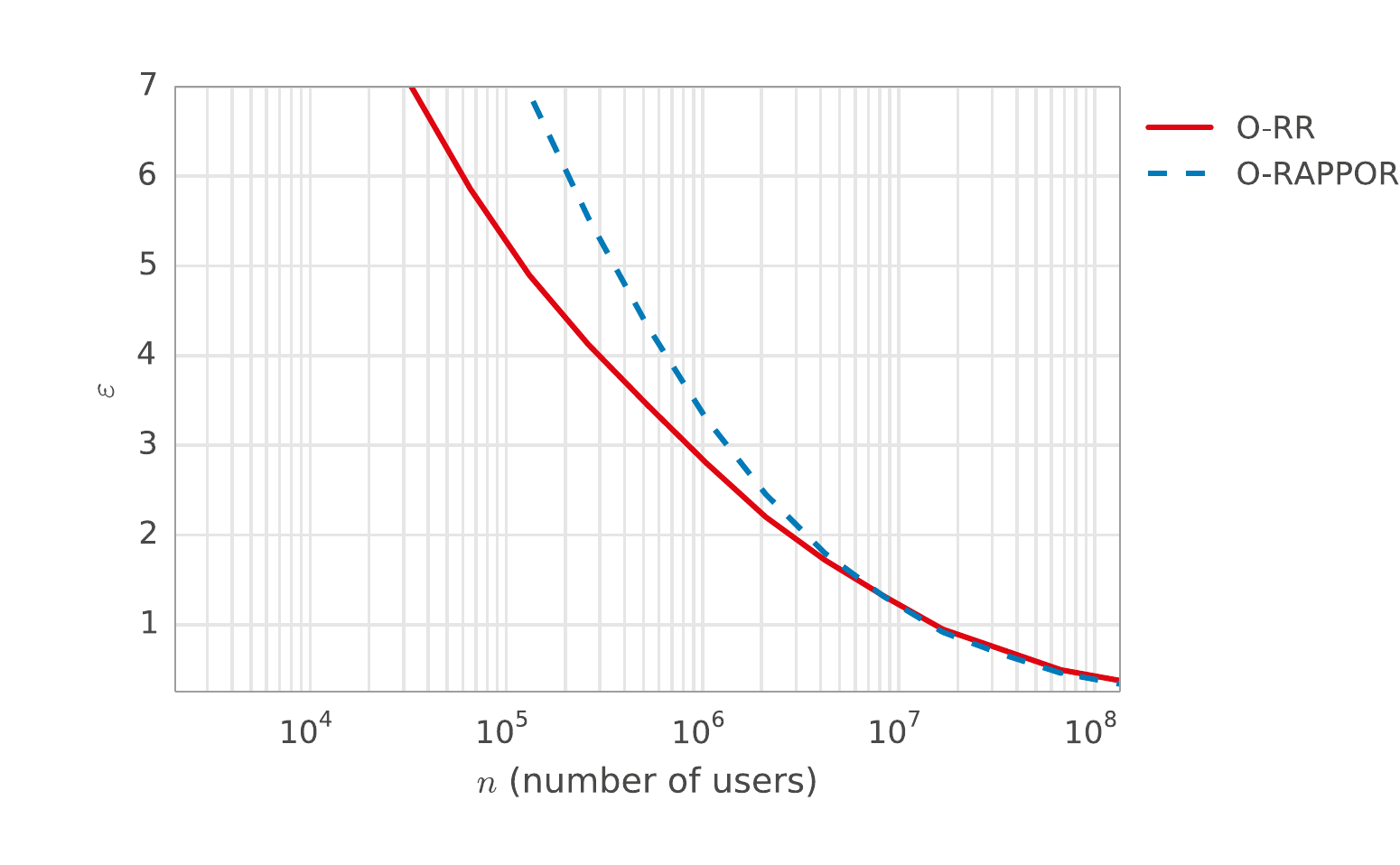}
\caption{$\ell_1=0.10$}
\end{subfigure}
\begin{subfigure}[b]{.45\linewidth}
\includegraphics[width=\linewidth]{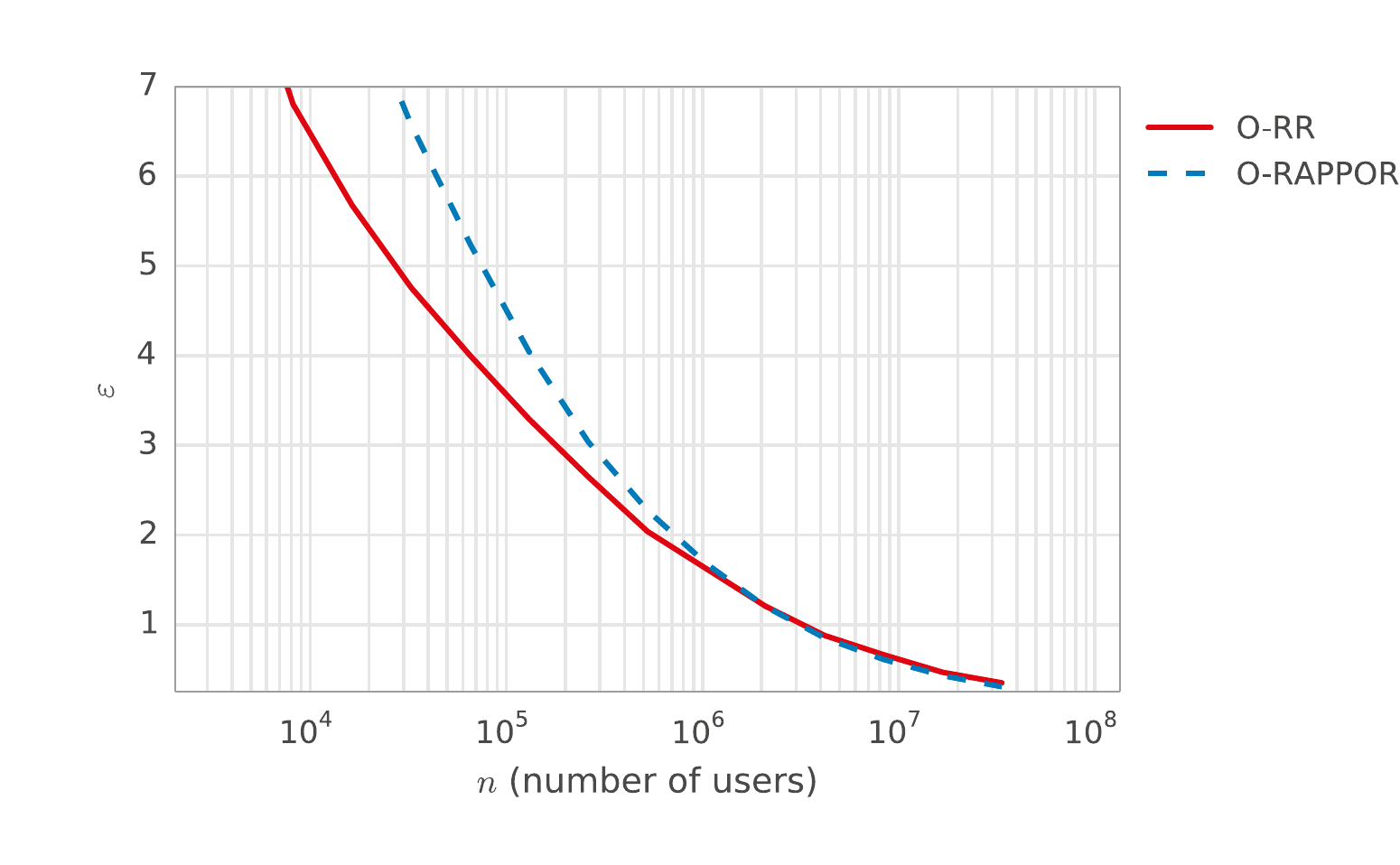}
\caption{$\ell_1=0.20$}
\end{subfigure}

\begin{subfigure}[b]{.45\linewidth}
\includegraphics[width=\linewidth]{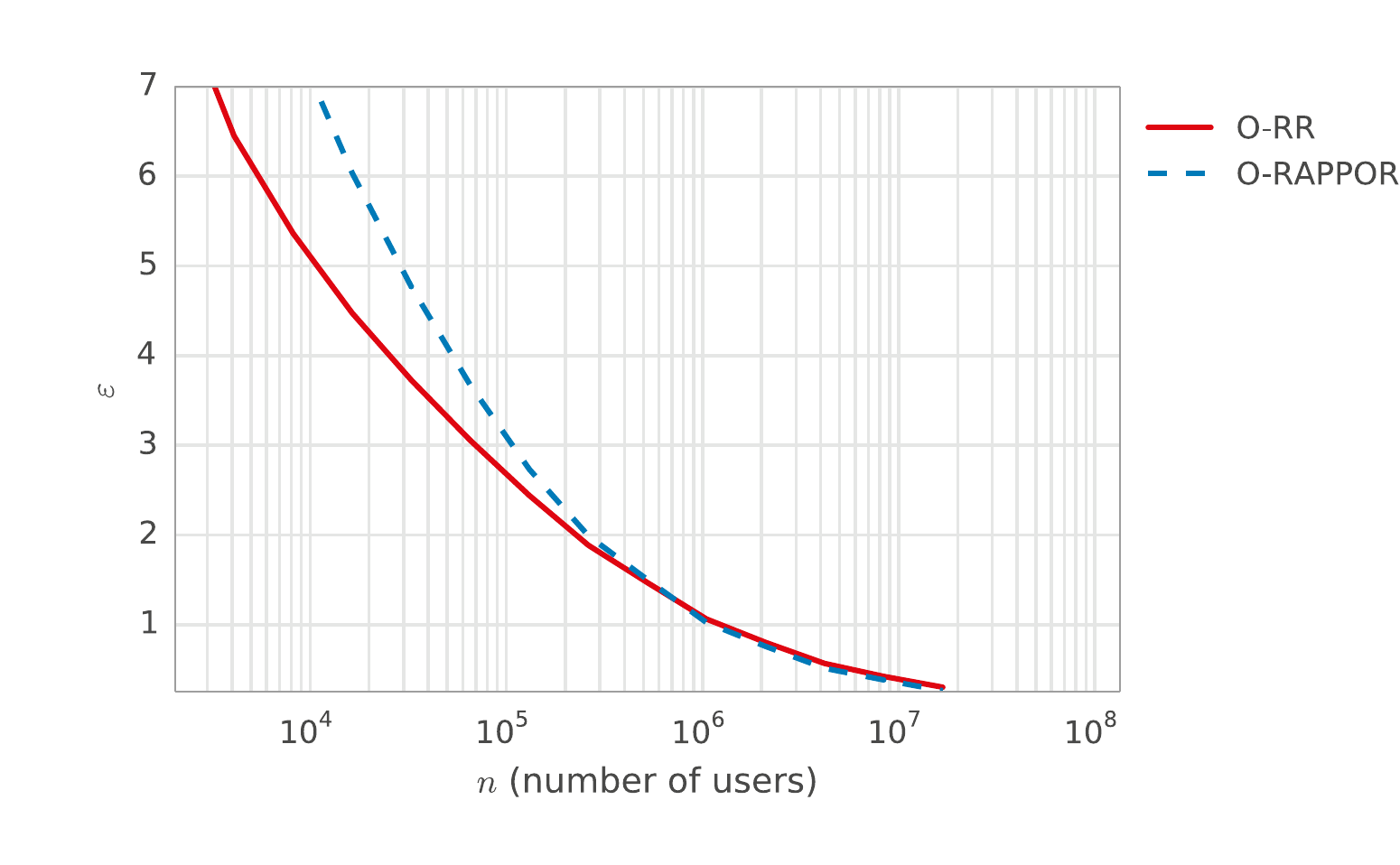}
\caption{$\ell_1=0.30$}
\end{subfigure}

\caption{
Taking $\ell_1$ loss (the utility) and $\nScalar$ (the number of users) as
fixed requirements (as is the case in many practical scenarios), we
approximate the degree of privacy $\varepsilon$ that can be obtained under
\ORR and \ORAPPOR for open alphabets (lower $\varepsilon$ is better).
Input is generated from an
alphabet of $S=256$ symbols under a geometric distribution with mean=$S/5$,
as depicted in Figure~\protect\ref{fig:geometric_ground_truth}.
Free parameters are set via grid search to minimize the median loss over 50
samples at the given $\varepsilon$ and fixed parameter values.
}
\label{fig:privacy_at_n_open}
\end{figure*}

\begin{figure*}
\centering

\begin{subfigure}[b]{.45\linewidth}
\includegraphics[width=\linewidth]{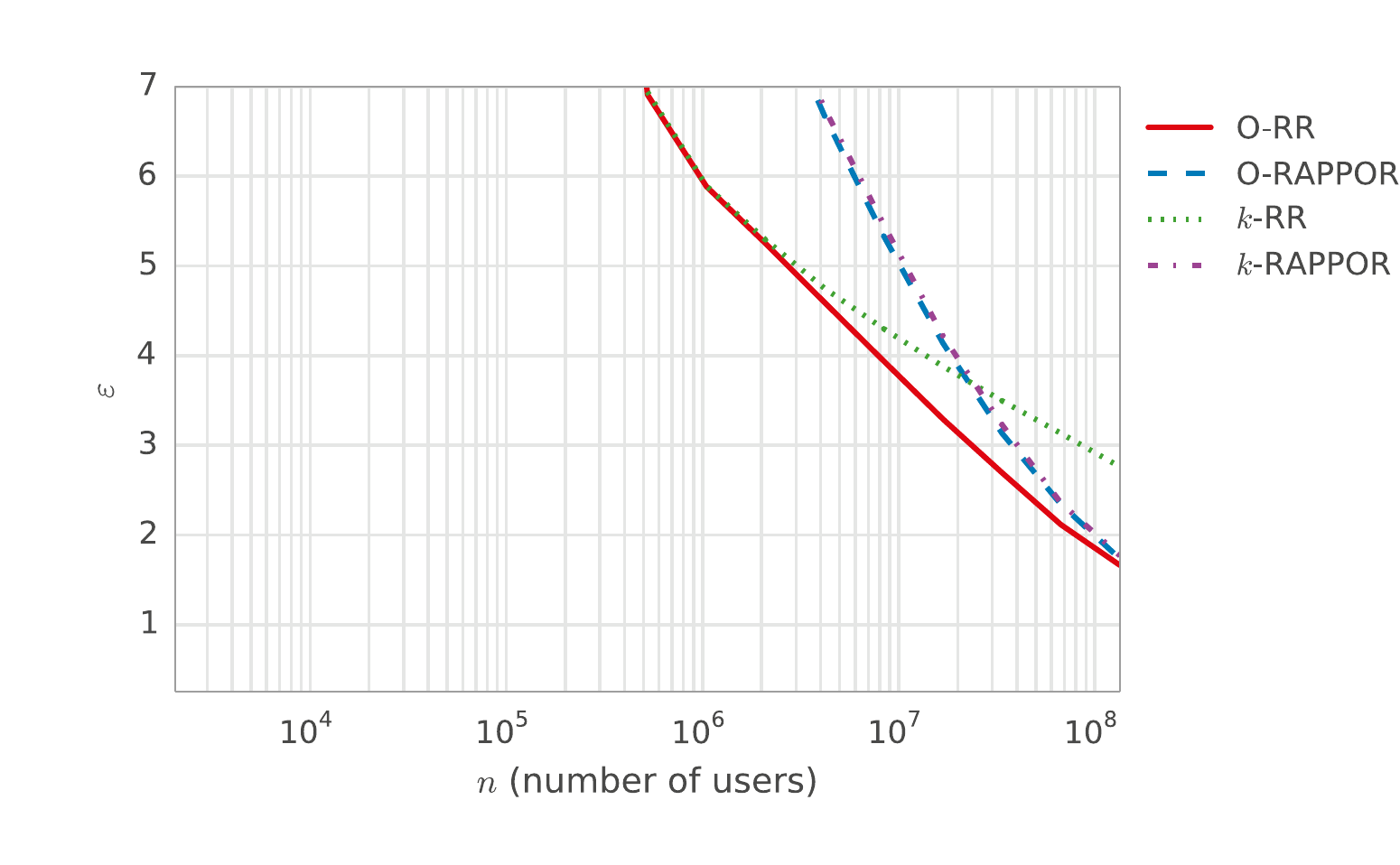}
\caption{$\ell_1=0.02$}
\end{subfigure}
\begin{subfigure}[b]{.45\linewidth}
\includegraphics[width=\linewidth]{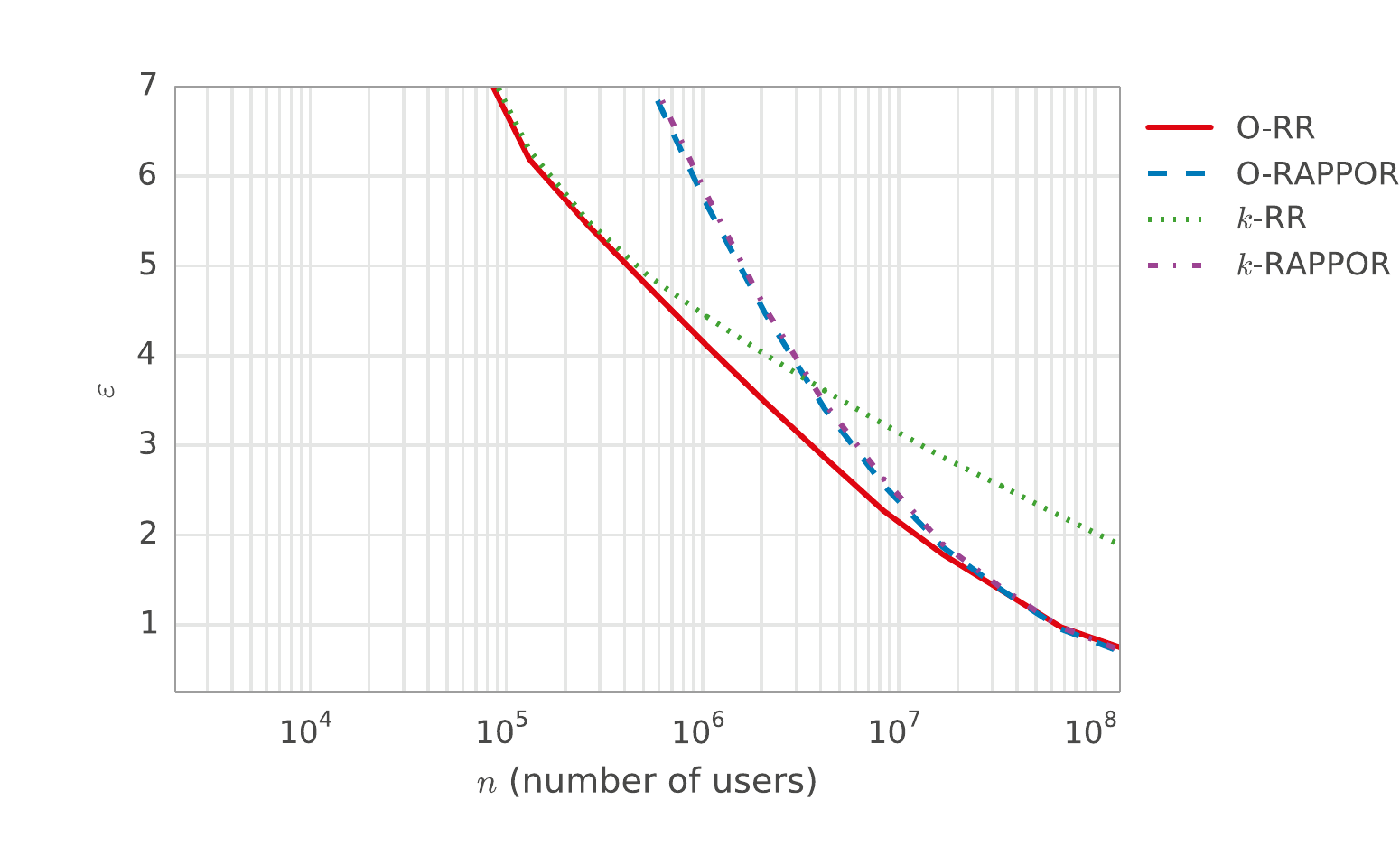}
\caption{$\ell_1=0.05$}
\end{subfigure}

\begin{subfigure}[b]{.45\linewidth}
\includegraphics[width=\linewidth]{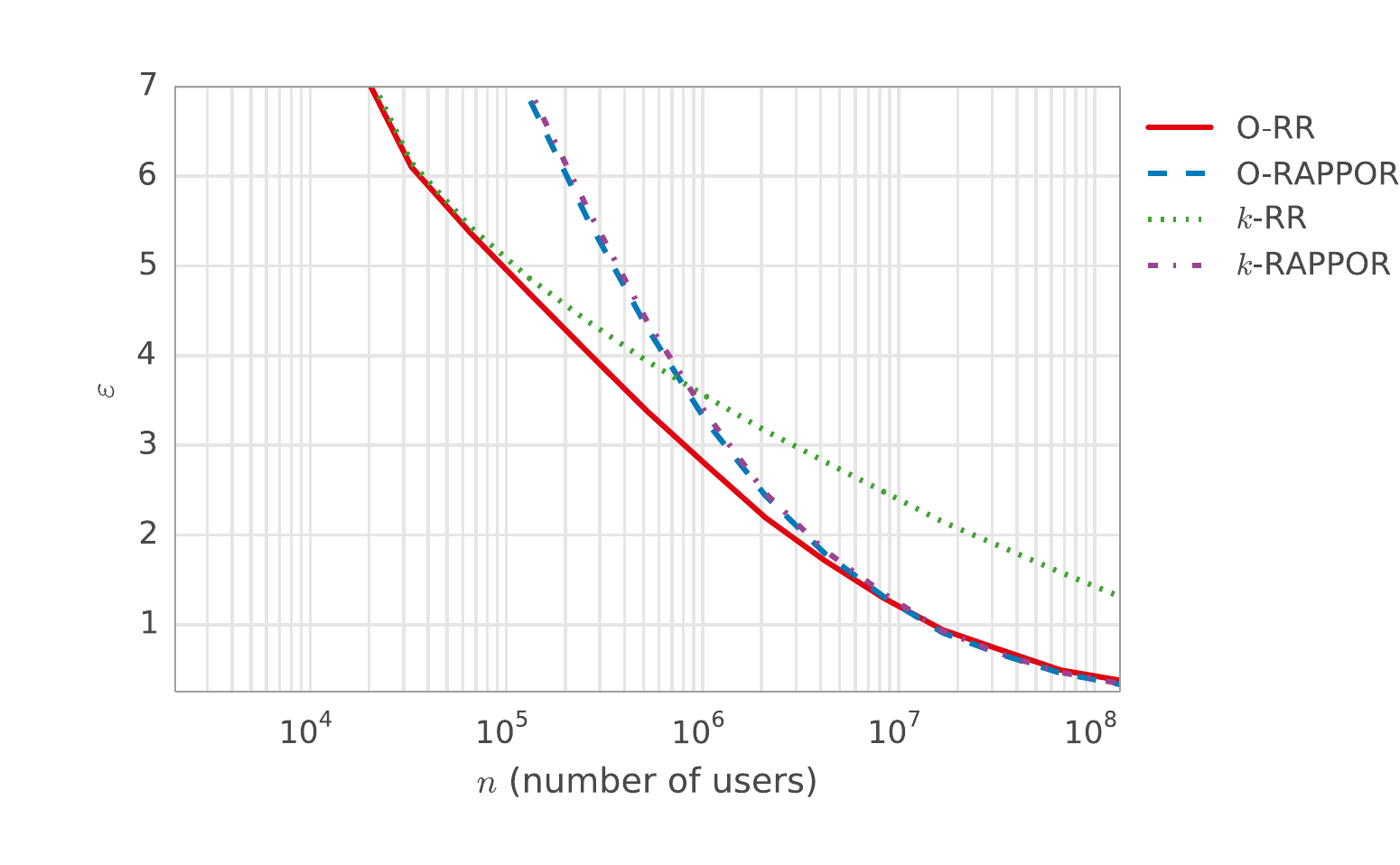}
\caption{$\ell_1=0.10$}
\end{subfigure}
\begin{subfigure}[b]{.45\linewidth}
\includegraphics[width=\linewidth]{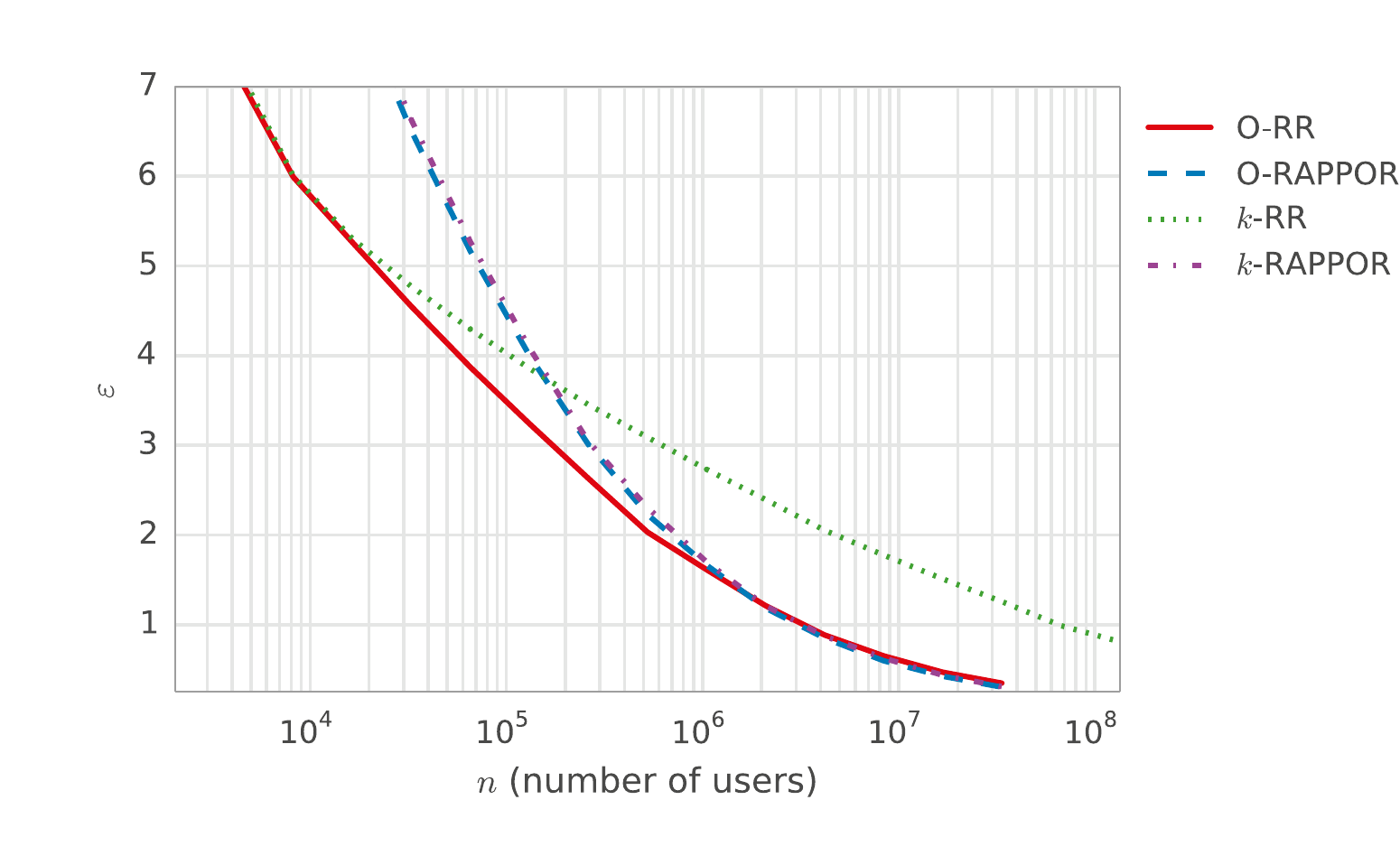}
\caption{$\ell_1=0.20$}
\end{subfigure}

\begin{subfigure}[b]{.45\linewidth}
\includegraphics[width=\linewidth]{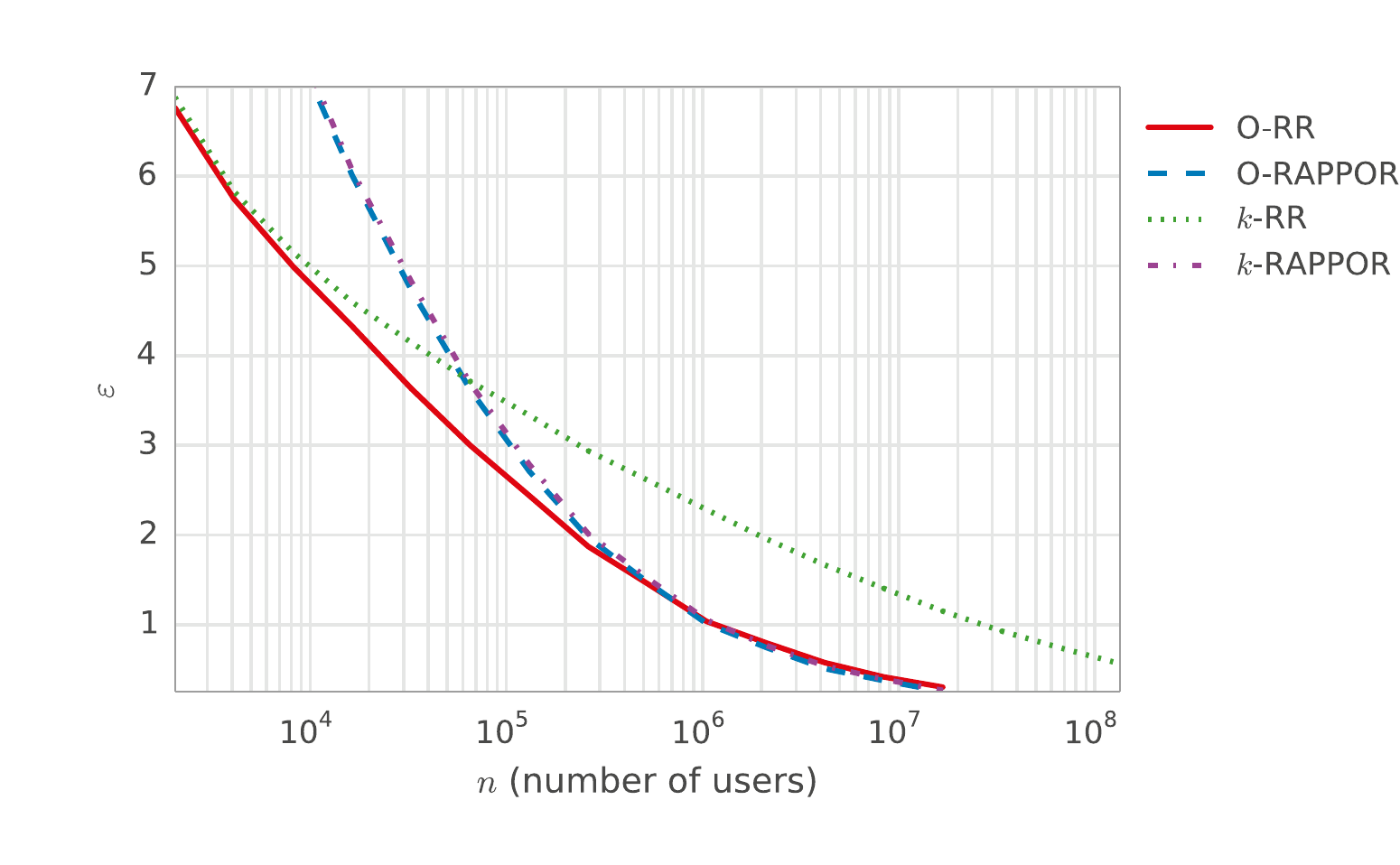}
\caption{$\ell_1=0.30$}
\end{subfigure}

\caption{
Taking $\ell_1$ loss (the utility) and $\nScalar$ (the number of users) as
fixed requirements (as is the case in many practical scenarios), we
approximate the degree of privacy $\varepsilon$ that can be obtained under
\ORR and \ORAPPOR for closed alphabets (lower $\varepsilon$ is better).
Input is generated from an
alphabet of $S=256$ symbols under a geometric distribution with mean=$S/5$,
as depicted in Figure~\protect\ref{fig:geometric_ground_truth}.
Free parameters are set via grid search to minimize the median loss over 50
samples at the given $\varepsilon$ and fixed parameter values.
}
\label{fig:privacy_at_n_closed}
\end{figure*}

\begin{figure*}
\centering

\begin{tabular}{ccc}

\begin{subfigure}[b]{.45\linewidth}
\includegraphics[width=\linewidth]{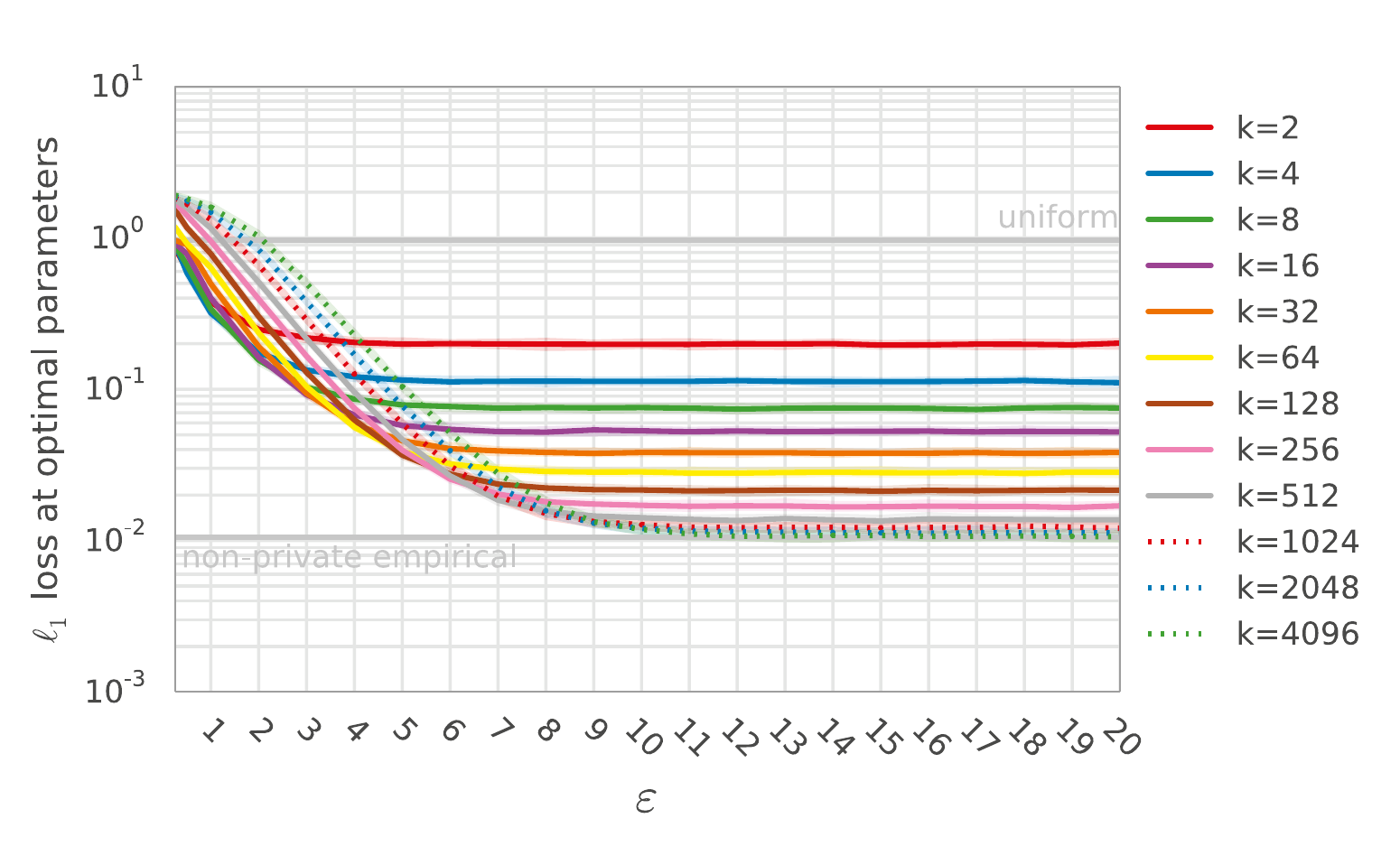}
\caption{\ORR varying $\kScalar$}
\end{subfigure}
&
\begin{subfigure}[b]{.45\linewidth}
\includegraphics[width=\linewidth]{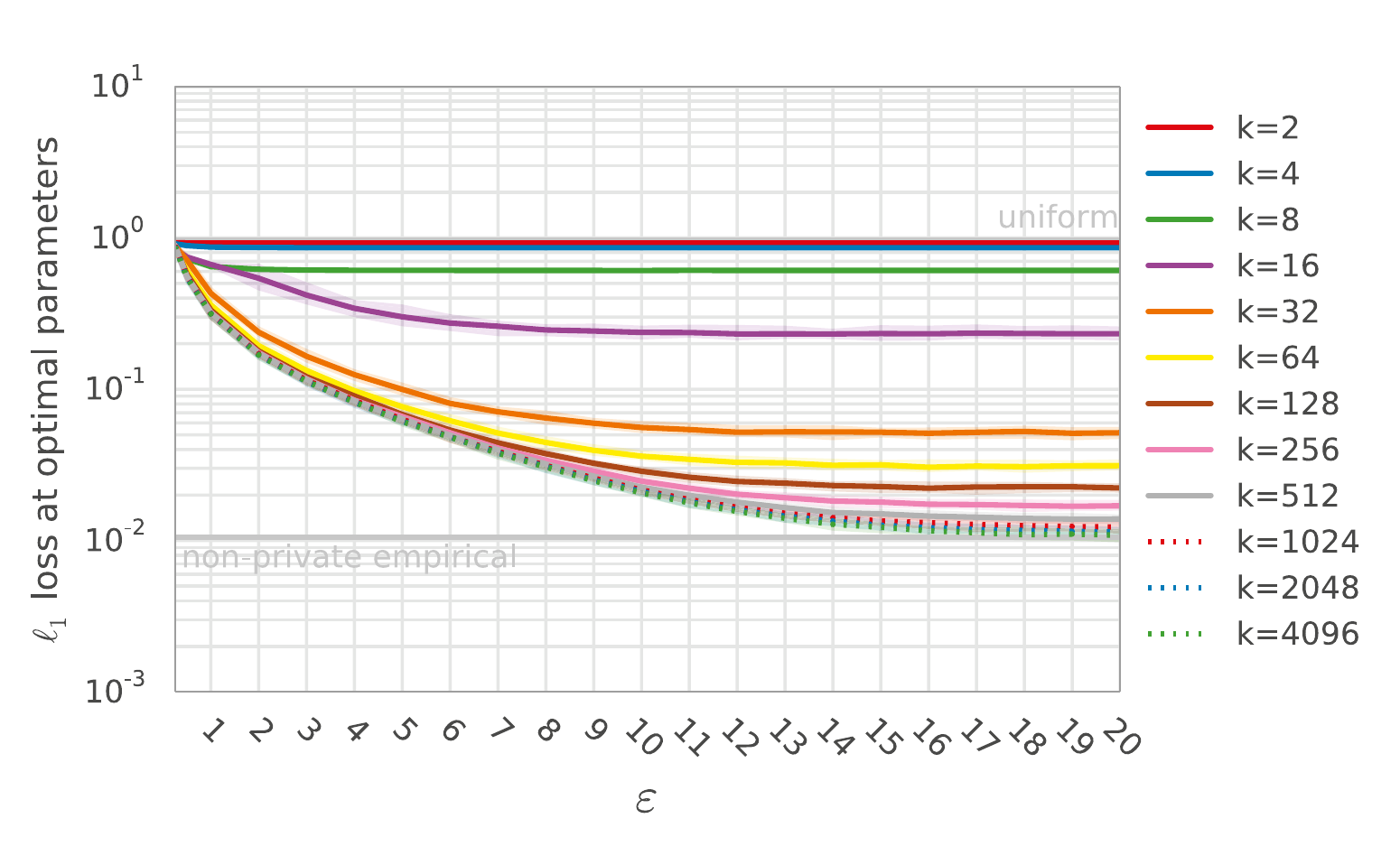}
\caption{\ORAPPOR varying $\kScalar$}
\end{subfigure}
\\
\begin{subfigure}[b]{.45\linewidth}
\includegraphics[width=\linewidth]{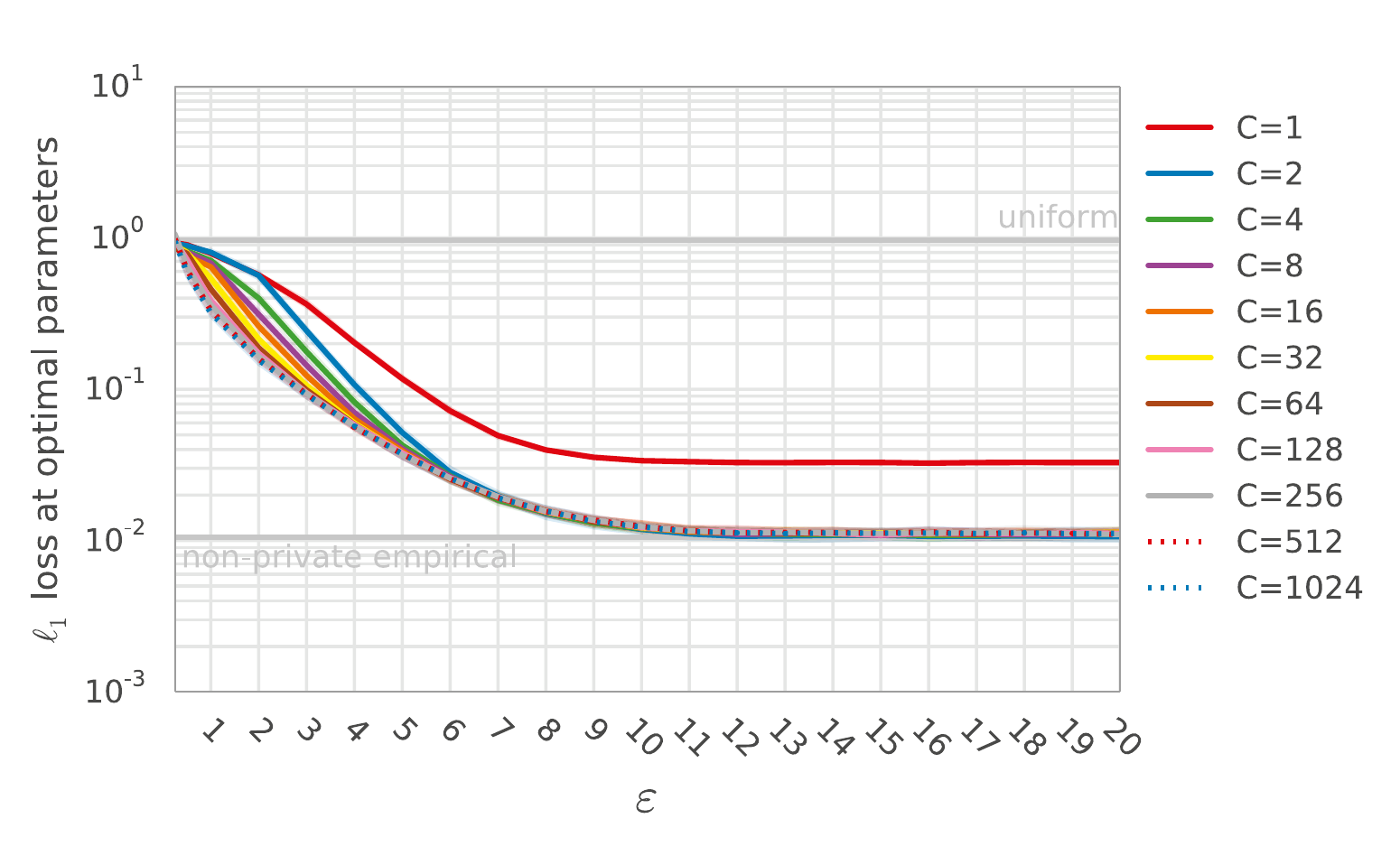}
\caption{\ORR varying $\CScalar$}
\end{subfigure}
&
\begin{subfigure}[b]{.45\linewidth}
\includegraphics[width=\linewidth]{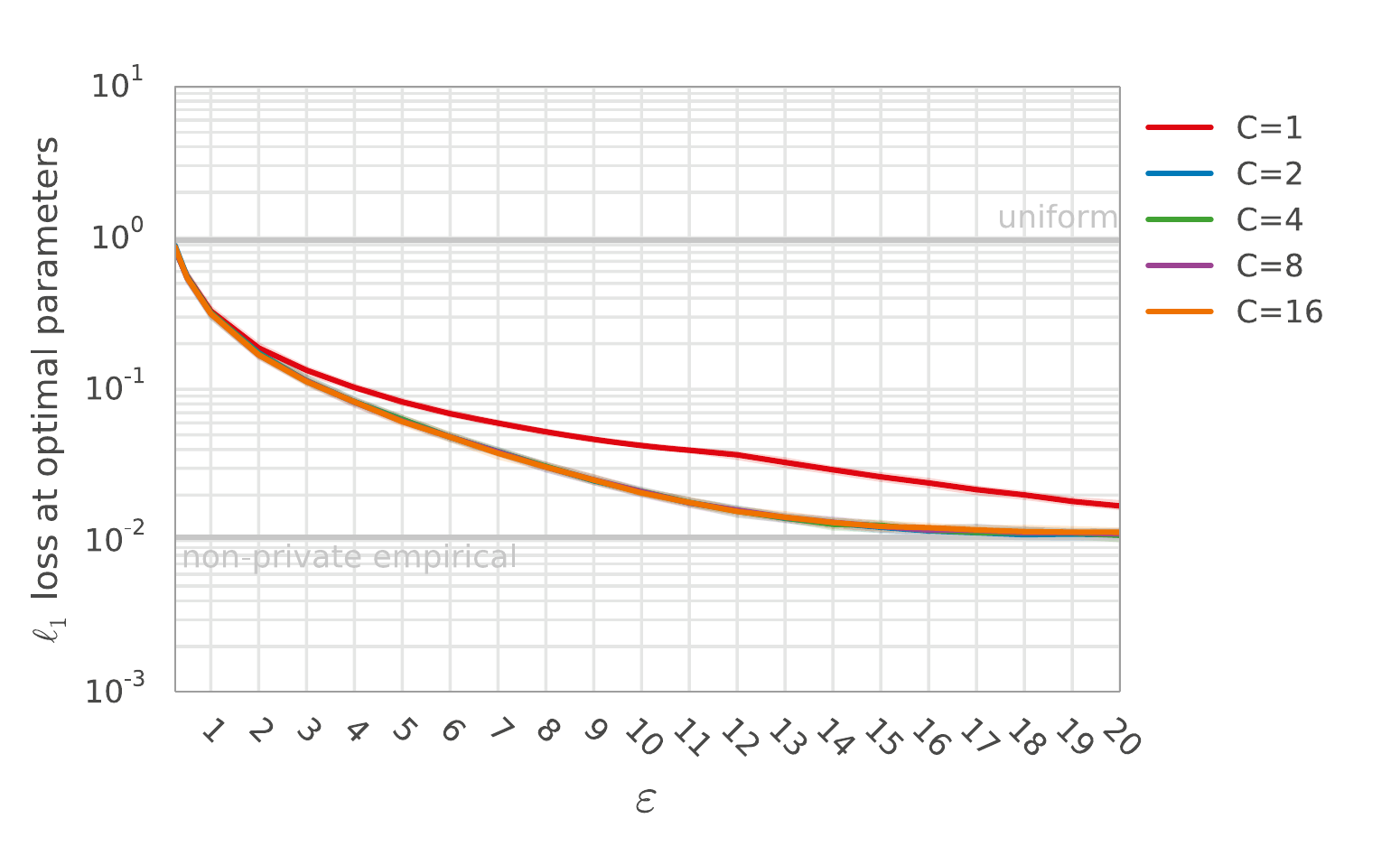}
\caption{\ORAPPOR varying $\CScalar$}
\end{subfigure}
\\
&
\begin{subfigure}[b]{.45\linewidth}
\includegraphics[width=\linewidth]{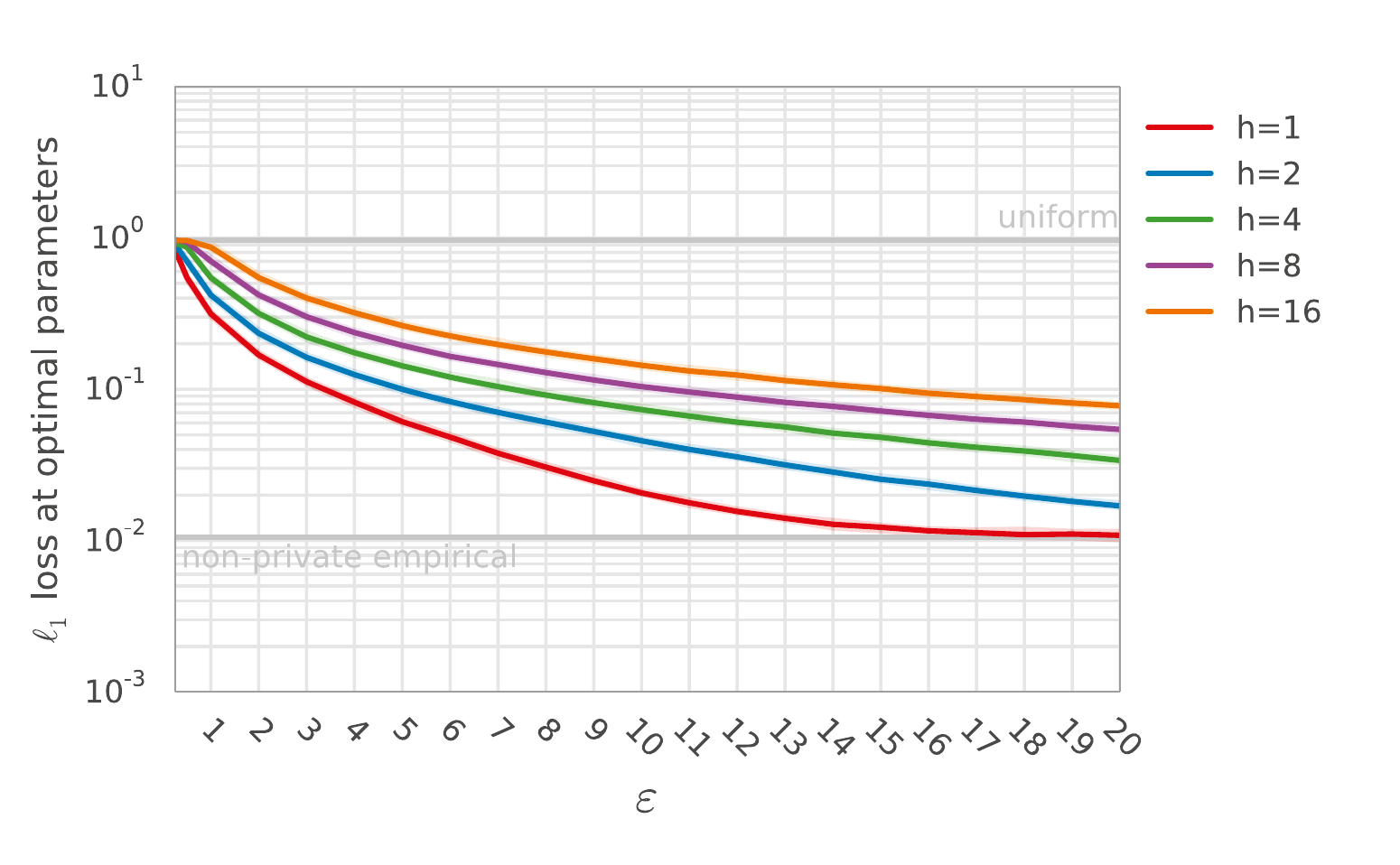}
\caption{\ORAPPOR varying $\hScalar$}
\end{subfigure}

\end{tabular}

\caption{$\ell_1$ loss when decoding open alphabets
using \ORR and \ORAPPOR under various parameter settings,
for $n=10^6$ users with input drawn from an
alphabet of $S=4096$ symbols under a geometric distribution with mean=$S/5$.
Remaining free parameters are set via grid search to minimize the
median loss over 50 samples at the given $\varepsilon$ and fixed parameter
values.  Lines show median $\ell_1$ loss while the (narrow) shaded regions indicate 90\%
confidence intervals (over 50 samples for the optimal parameter settings.)}
\label{fig:open_set_params}
\end{figure*}

\begin{figure*}
\centering
\begin{tabular}{cc}

\begin{subfigure}[b]{.45\linewidth}
\includegraphics[width=\linewidth]{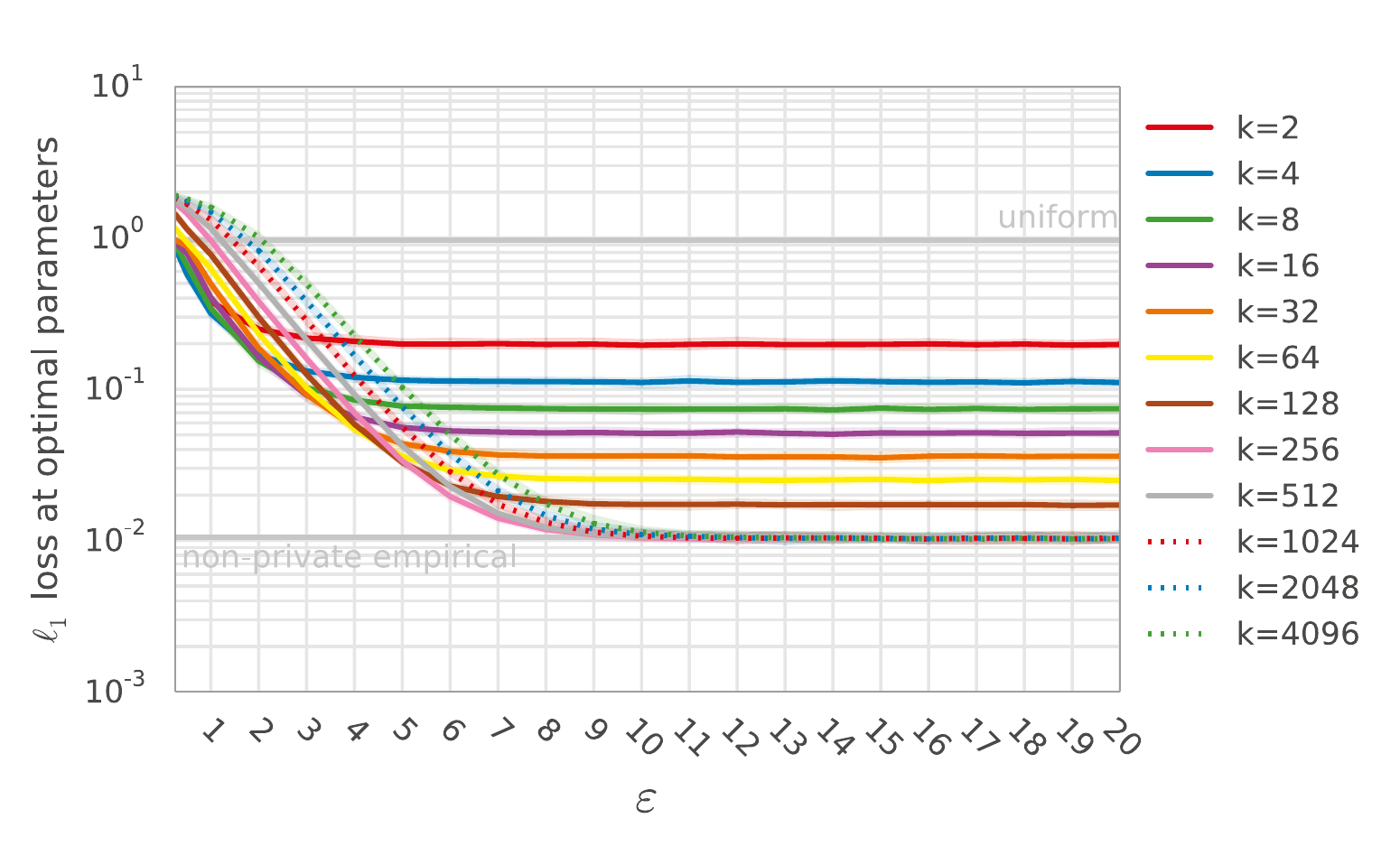}
\caption{\ORR varying $\kScalar$}
\end{subfigure}
&
\begin{subfigure}[b]{.45\linewidth}
\includegraphics[width=\linewidth]{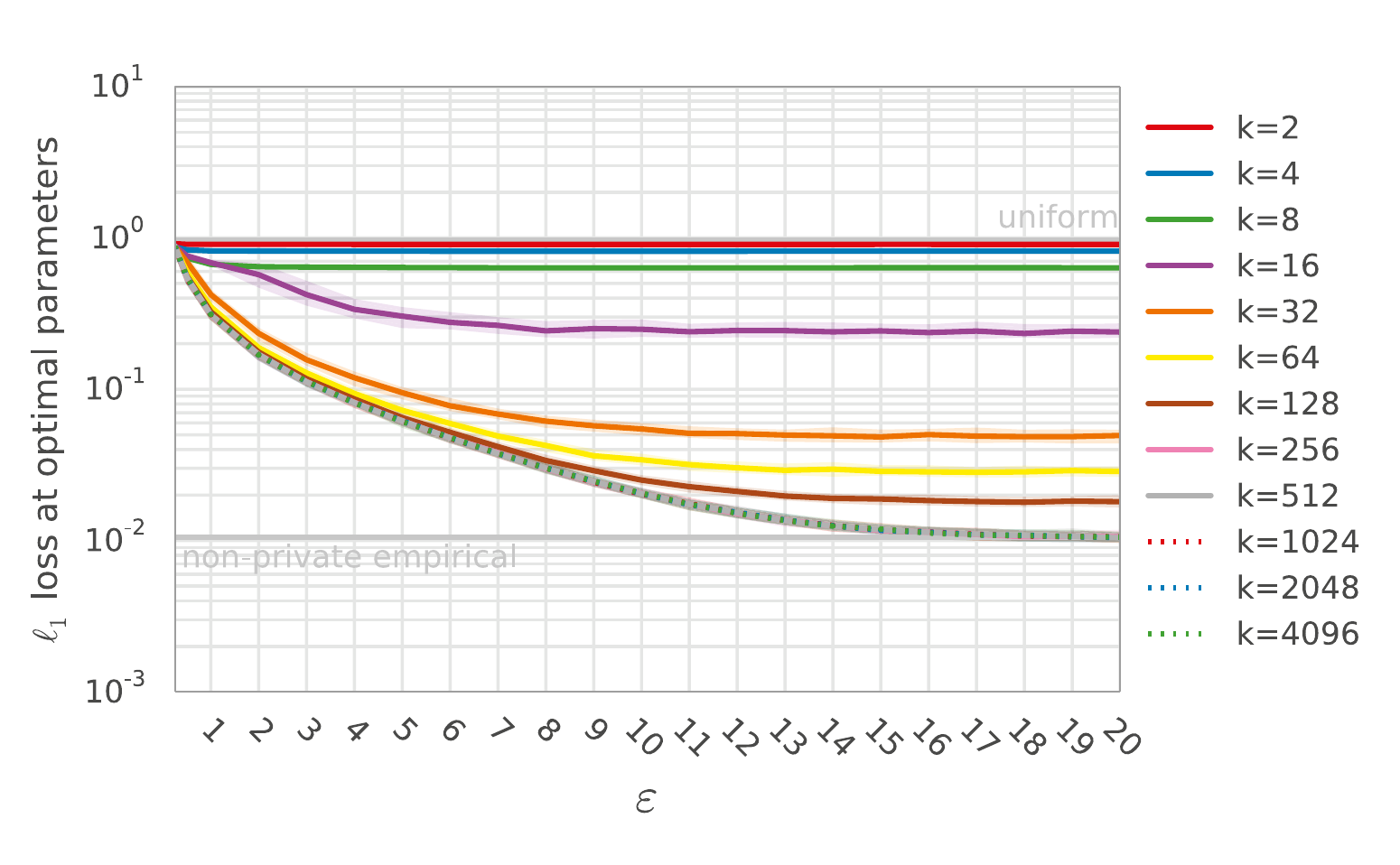}
\caption{\ORAPPOR varying $\kScalar$}
\end{subfigure}
\\
\begin{subfigure}[b]{.45\linewidth}
\includegraphics[width=\linewidth]{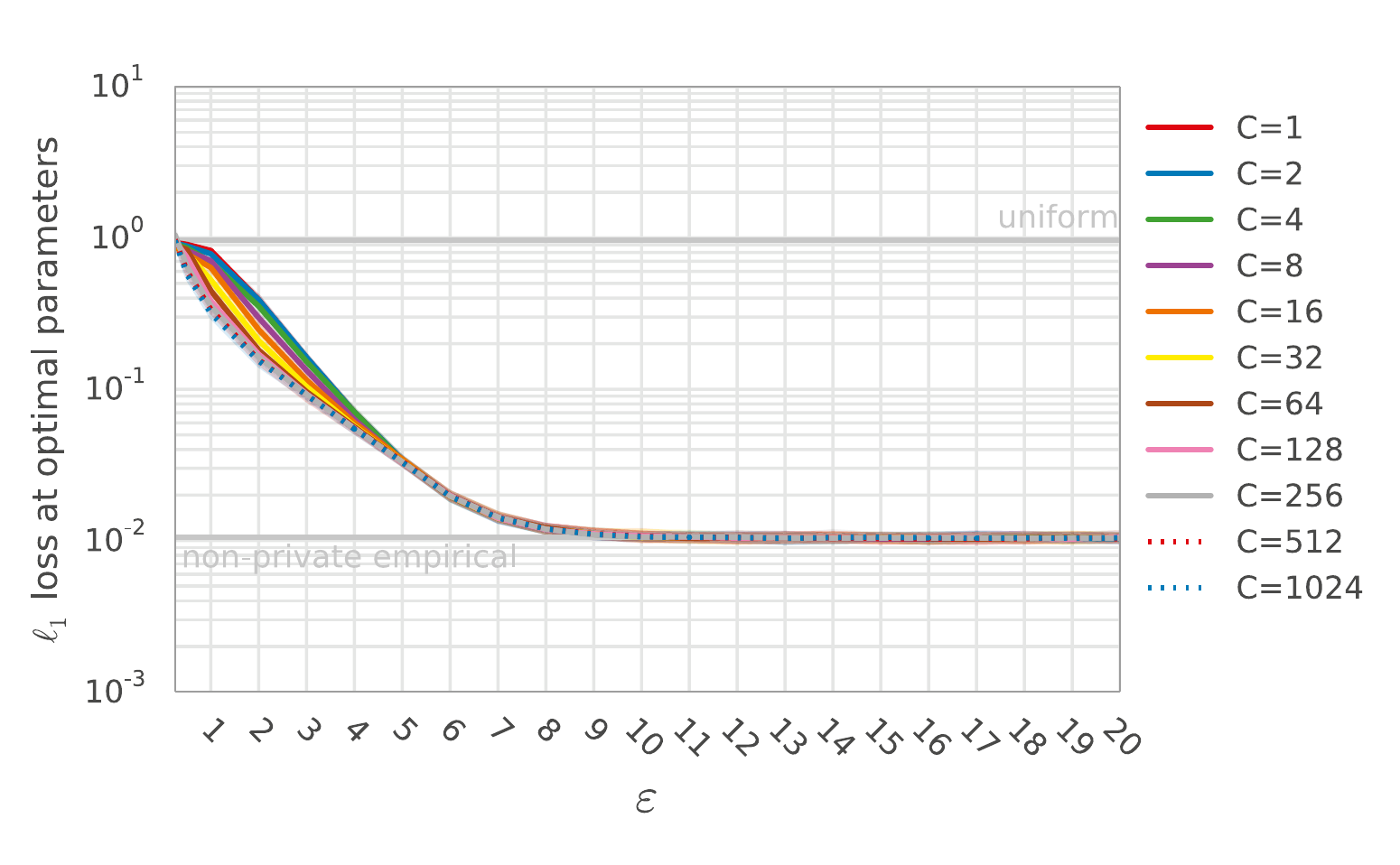}
\caption{\ORR varying $\CScalar$}
\end{subfigure}
&
\begin{subfigure}[b]{.45\linewidth}
\includegraphics[width=\linewidth]{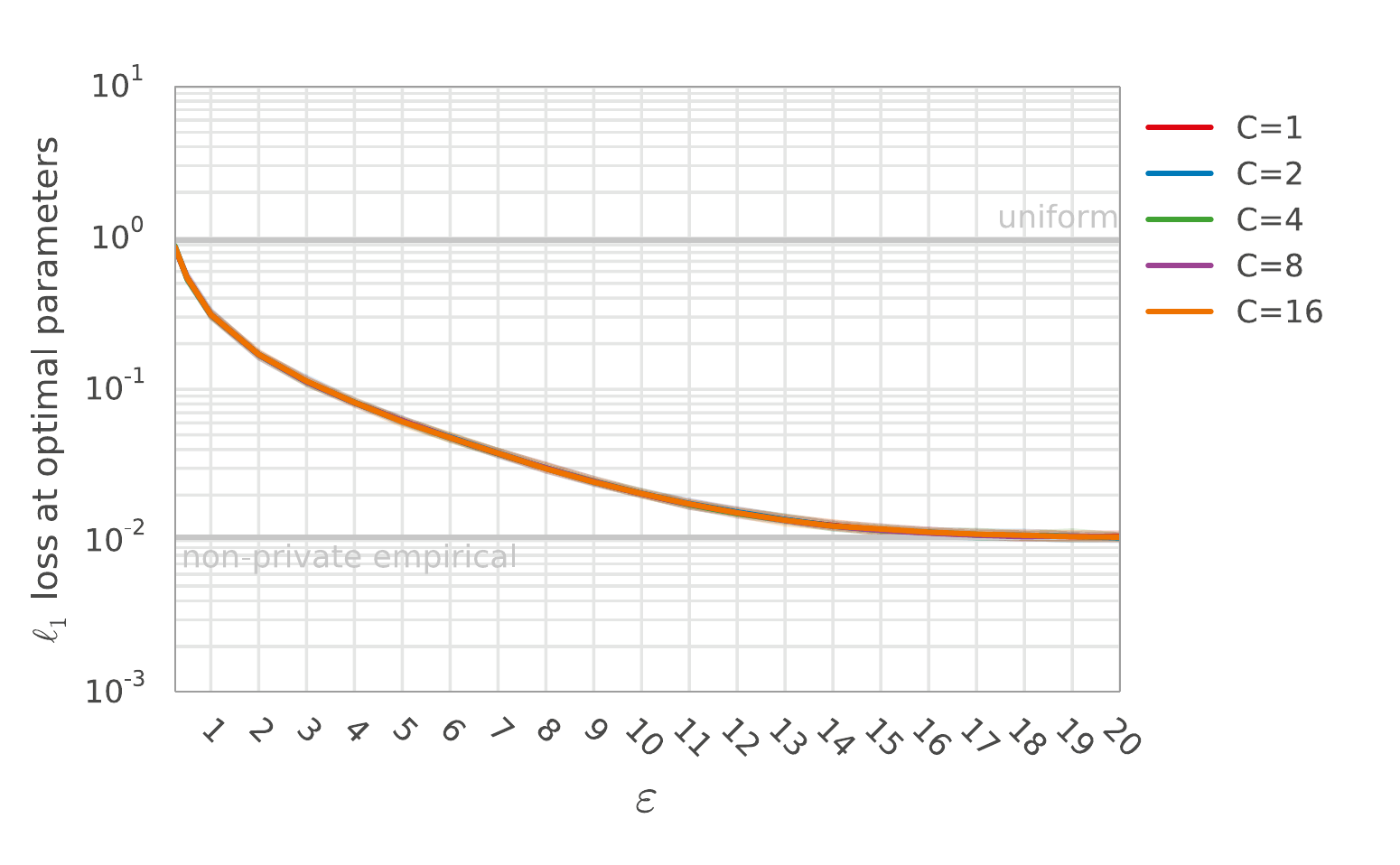}
\caption{\ORAPPOR varying $\CScalar$}
\end{subfigure}
\\
&
\begin{subfigure}[b]{.45\linewidth}
\includegraphics[width=\linewidth]{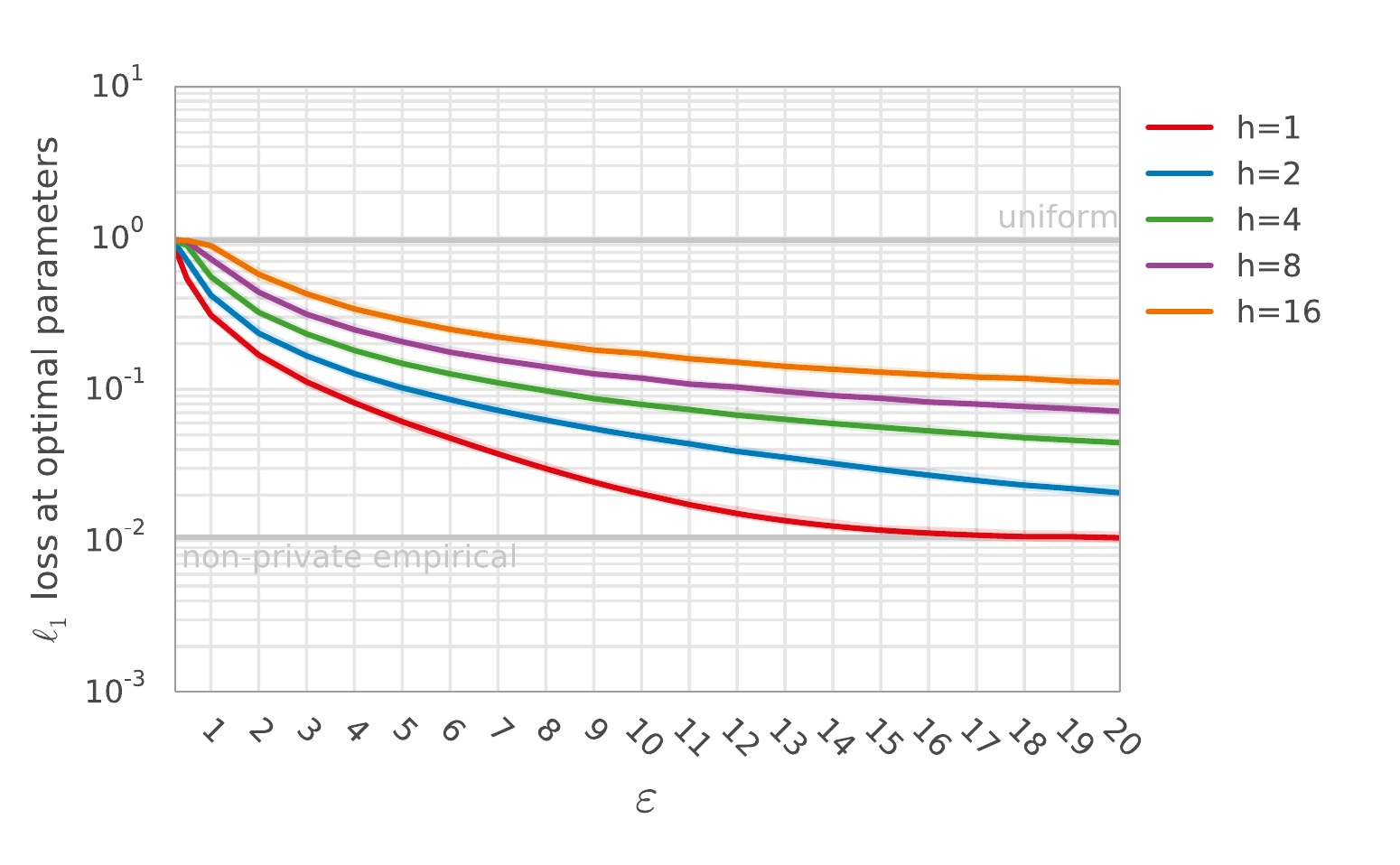}
\caption{\ORAPPOR varying $\hScalar$}
\end{subfigure}

\end{tabular}
\caption{$\ell_1$ loss when decoding closed alphabets using
the \ORR and \ORAPPOR under various parameter
settings,
for $n=10^6$ users with input drawn from an
alphabet of $S=4096$ symbols under a geometric distribution with mean=$S/5$.
Remaining free parameters are set via grid search to minimize the
median loss over 50 samples at the given $\varepsilon$ and fixed parameter
values.  Lines show median $\ell_1$ loss while the (narrow) shaded regions indicate 90\%
confidence intervals (over 50 samples for the optimal parameter settings.)}
\label{fig:closed_set_params}
\end{figure*}

\begin{figure*}
\centering

\begin{tabular}{cc}

\begin{subfigure}[b]{.45\linewidth}
\includegraphics[width=\linewidth]{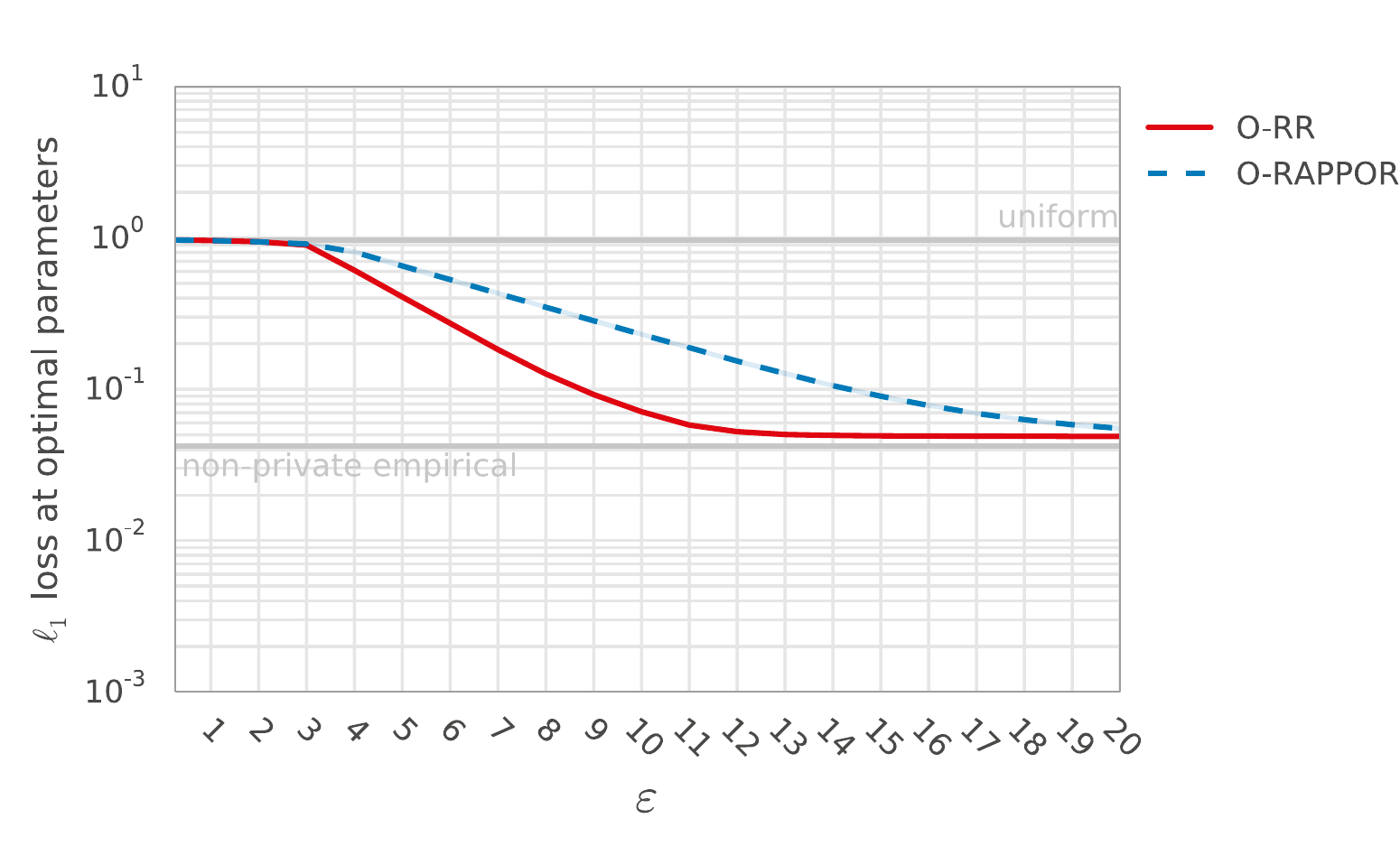}
\caption{$n=10^6$ users}
\end{subfigure}
&
\begin{subfigure}[b]{.45\linewidth}
\includegraphics[width=\linewidth]{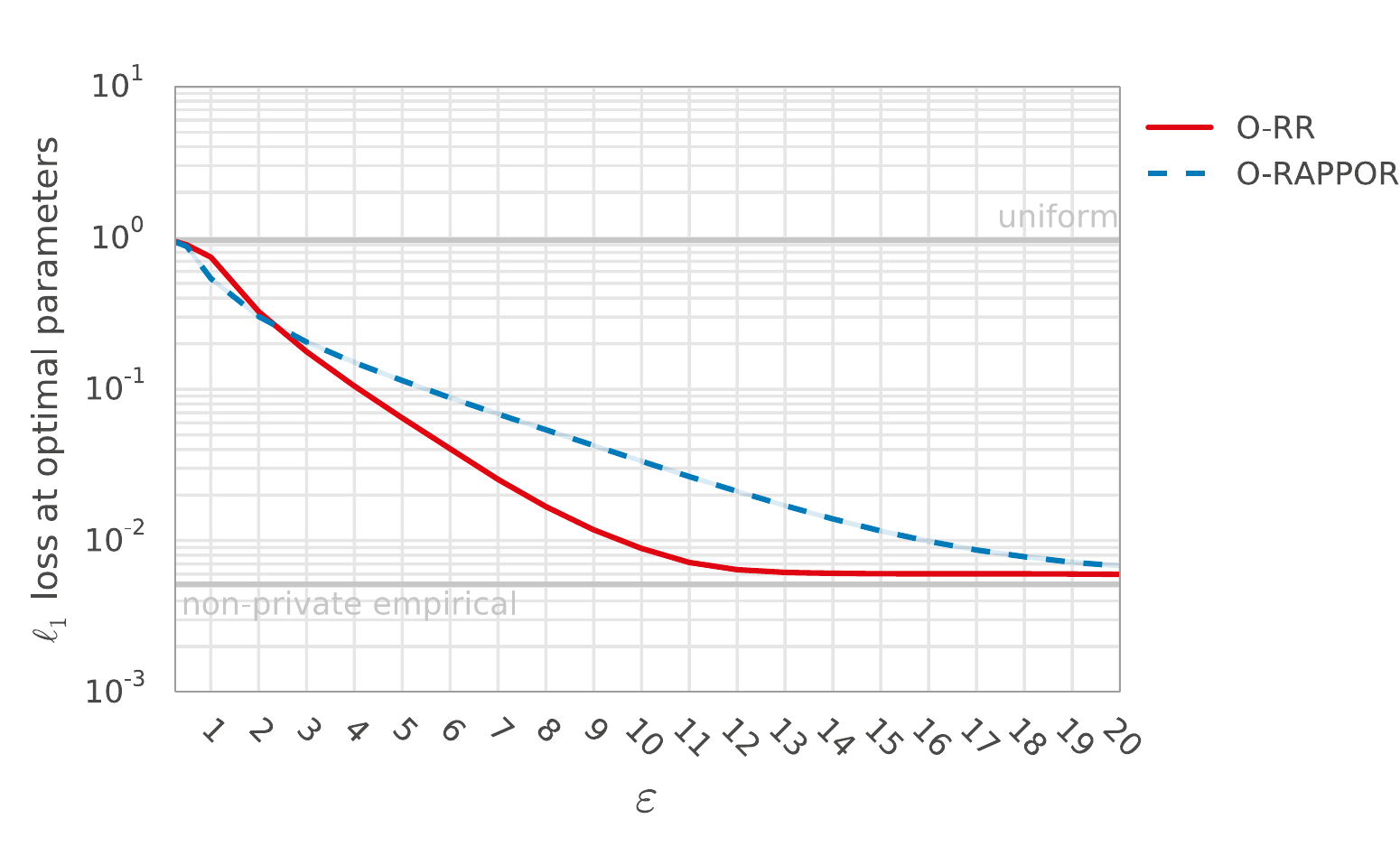}
\caption{$n=10^8$ users}
\end{subfigure}
\end{tabular}
\caption{
$\ell_1$ loss when decoding open alphabets using the \ORR and \ORAPPOR,
with input drawn from an
alphabet of $S=4096$ symbols under a geometric distribution with mean=$S/5$.
Free parameters are set via grid search over
$\kScalar \in [2, 4, 8, \ldots, 8192, 16384]$,
$\cScalar \in [1, 2, 4, \ldots, 512, 1024]$,
$\hScalar \in [1, 2]$
to minimize the median loss over 50 samples at the given $\varepsilon$
value.  Lines show median $\ell_1$ loss while the (narrow) shaded regions indicate 90\%
confidence intervals (over 50 samples).
Baselines indicate expected loss from (1) using an empirical estimator
directly on the input $\sVector$ and (2) using the uniform distribution
as the $\hat{\pVector}$ estimate.
}

\label{fig:open_set_zoom_s4096}
\end{figure*}

\begin{figure*}
\centering


\begin{subfigure}[b]{.45\linewidth}
\includegraphics[width=\linewidth]{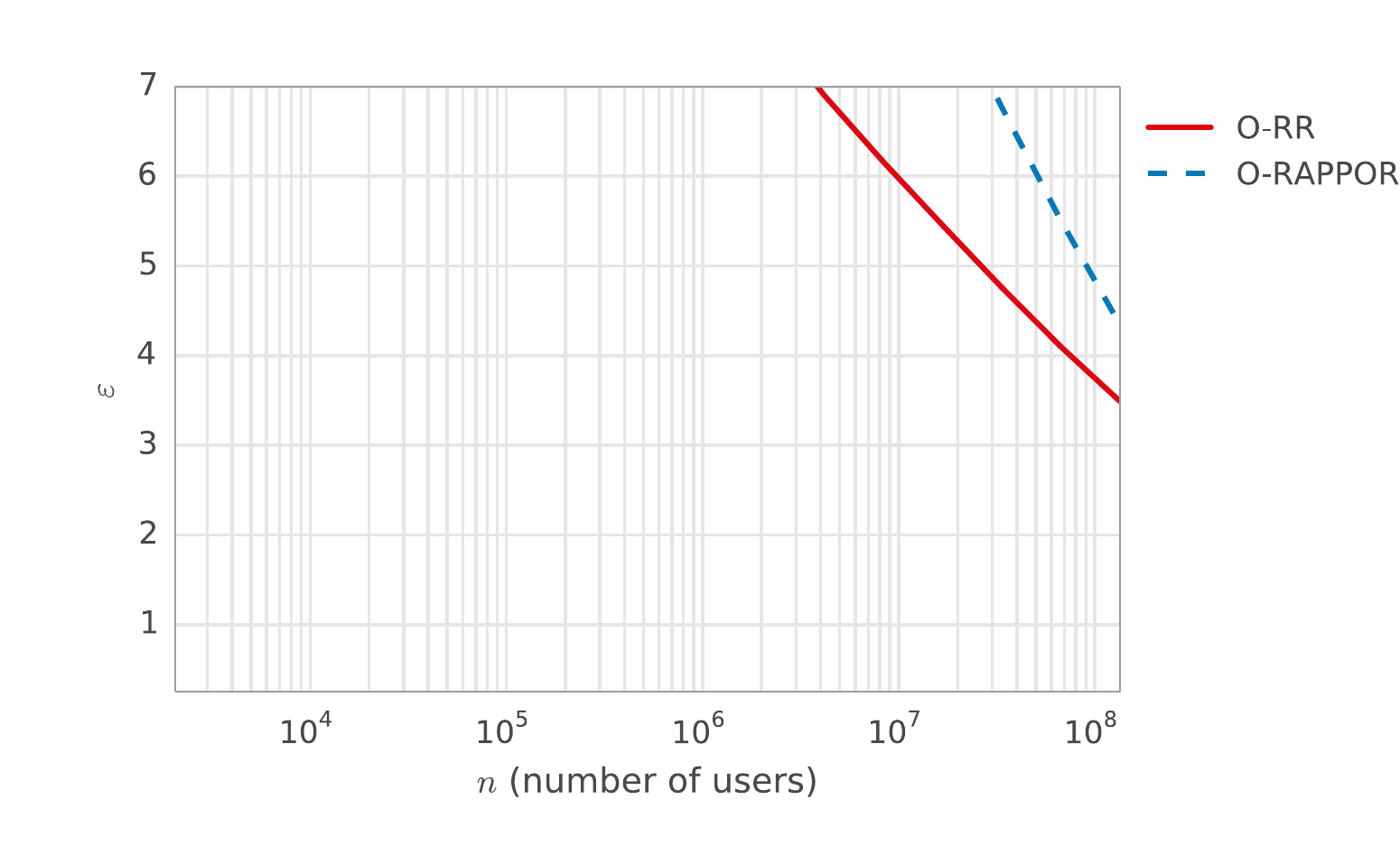}
\caption{$\ell_1=0.10$}
\end{subfigure}
\begin{subfigure}[b]{.45\linewidth}
\includegraphics[width=\linewidth]{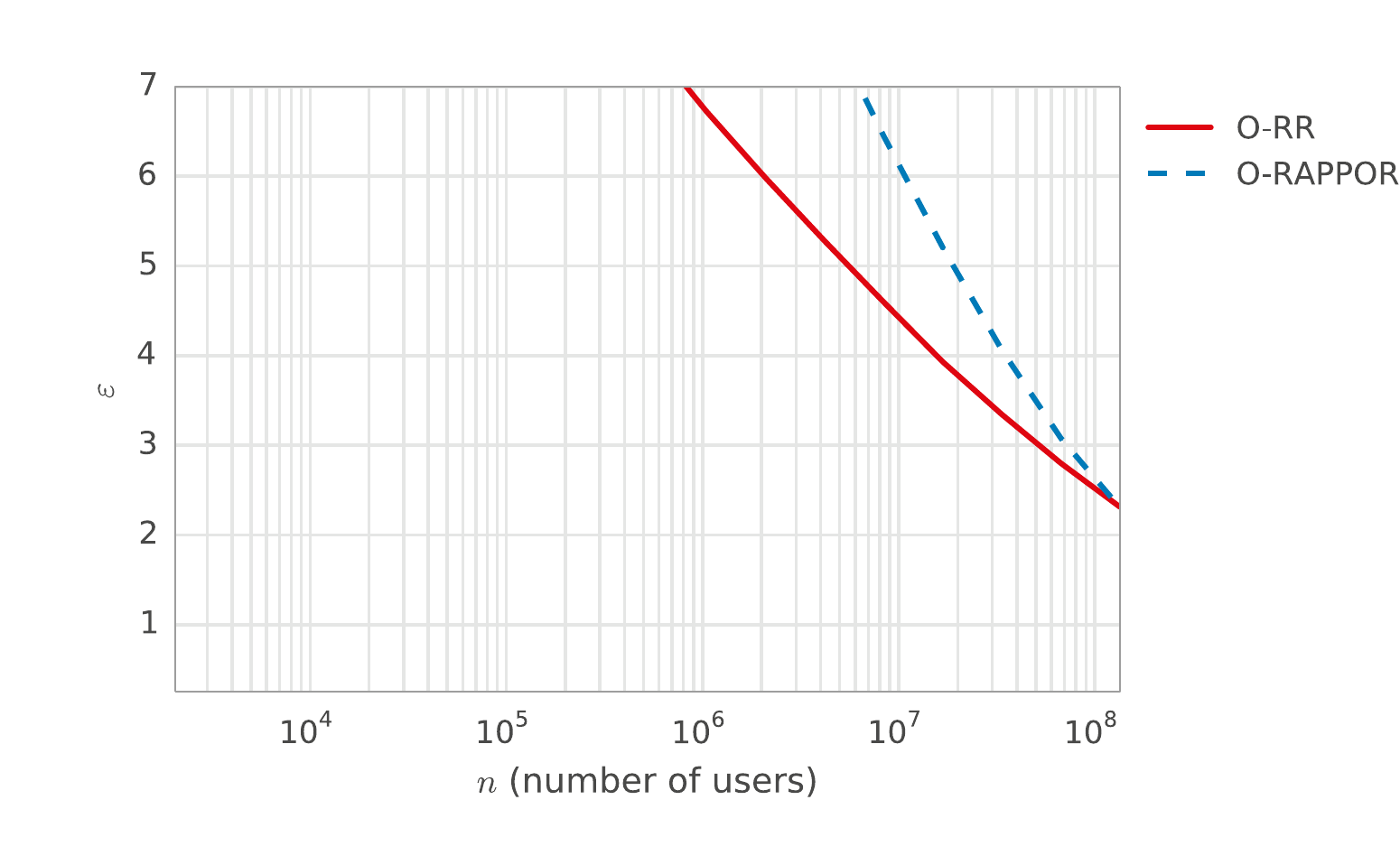}
\caption{$\ell_1=0.20$}
\end{subfigure}

\begin{subfigure}[b]{.45\linewidth}
\includegraphics[width=\linewidth]{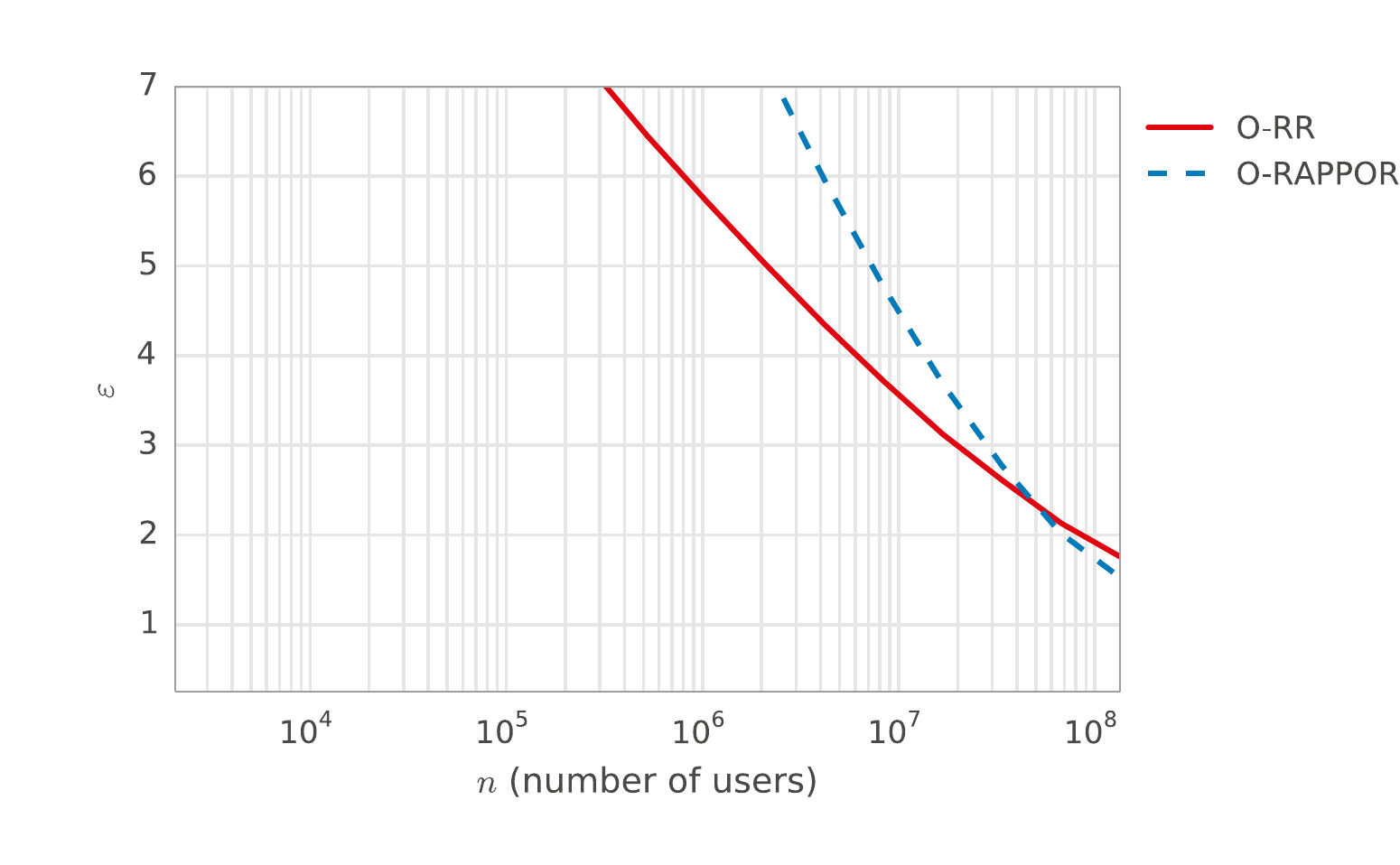}
\caption{$\ell_1=0.30$}
\end{subfigure}
\begin{subfigure}[b]{.45\linewidth}
\includegraphics[width=\linewidth]{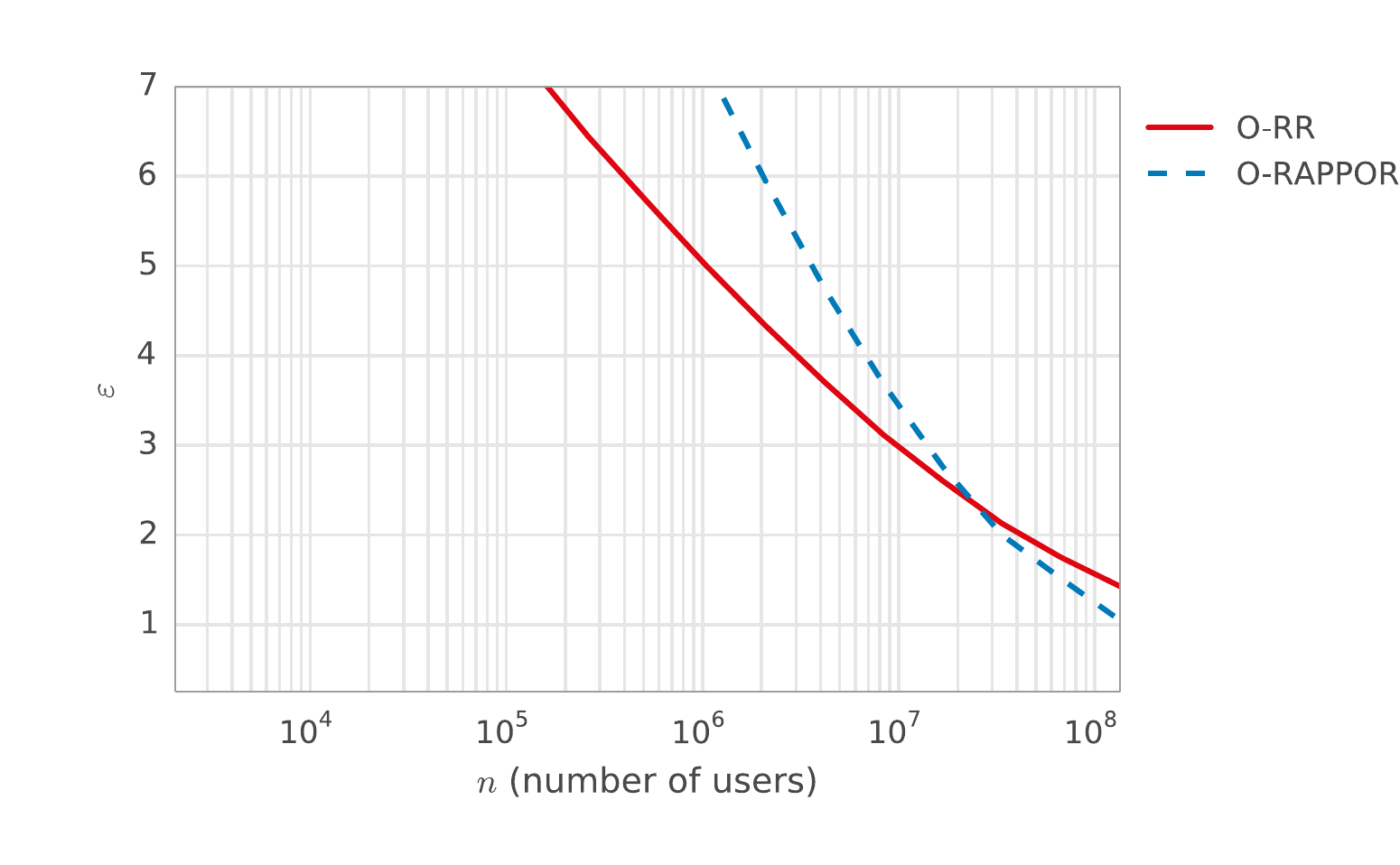}
\caption{$\ell_1=0.40$}
\end{subfigure}

\caption{
Taking $\ell_1$ loss (the utility) and $\nScalar$ (the number of users) as
fixed requirements (as is the case in many practical scenarios), we
approximate the degree of privacy $\varepsilon$ that can be obtained under
\ORR and \ORAPPOR for open alphabets (lower $\varepsilon$ is better).
Input is generated from an
alphabet of $S=4096$ symbols under a geometric distribution with mean=$S/5$,
as depicted in Figure~\protect\ref{fig:geometric_ground_truth}.
Free parameters are set via grid search to minimize the median loss over 50
samples at the given $\varepsilon$ and fixed parameter values.
}
\label{fig:privacy_at_n_open_s4096}
\end{figure*}

\begin{figure*}
\centering

\begin{tabular}{ccc}

\begin{subfigure}[b]{.45\linewidth}
\includegraphics[width=\linewidth]{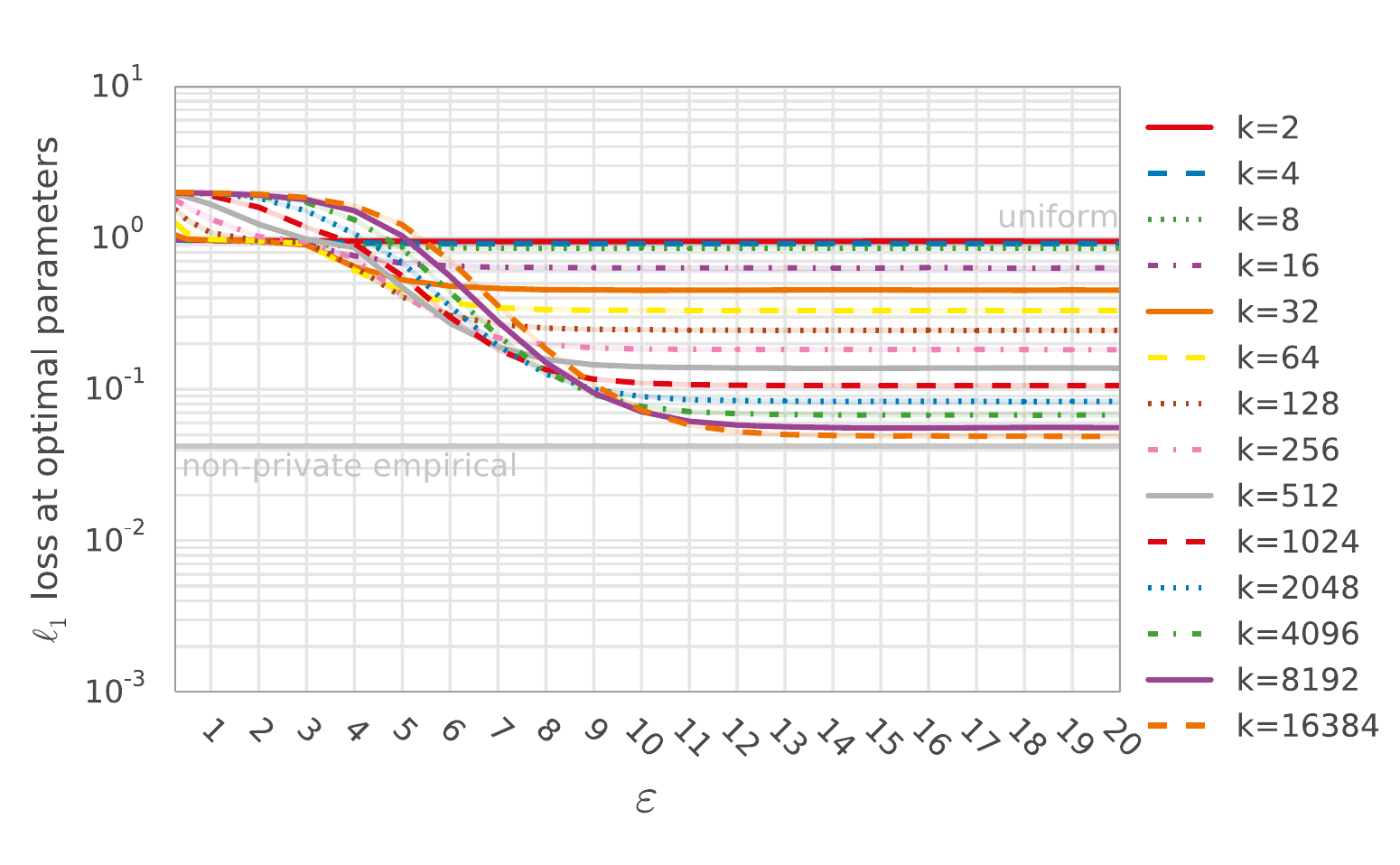}
\caption{\ORR varying $\kScalar$}
\end{subfigure}
&
\begin{subfigure}[b]{.45\linewidth}
\includegraphics[width=\linewidth]{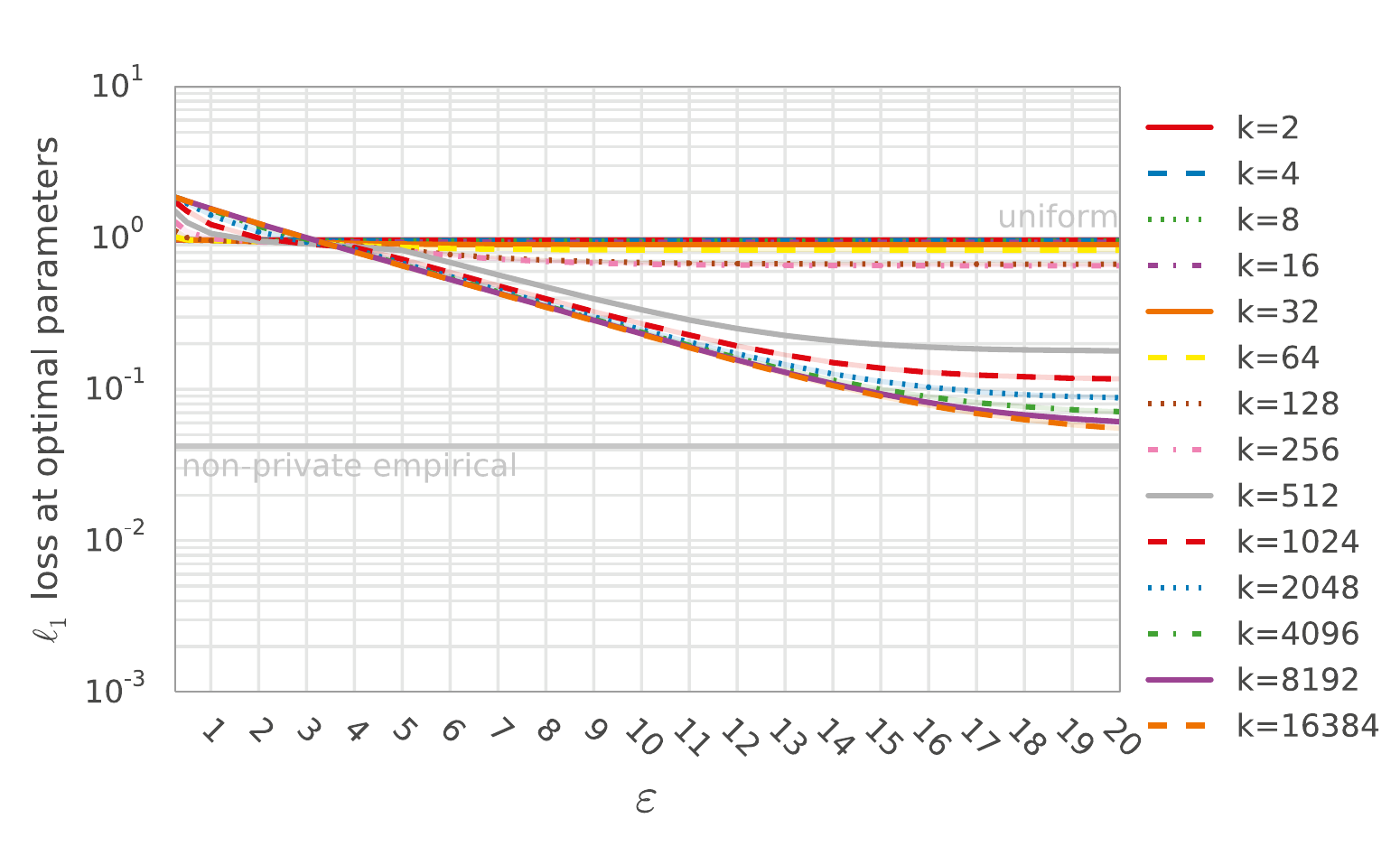}
\caption{\ORAPPOR varying $\kScalar$}
\end{subfigure}
\\
\begin{subfigure}[b]{.45\linewidth}
\includegraphics[width=\linewidth]{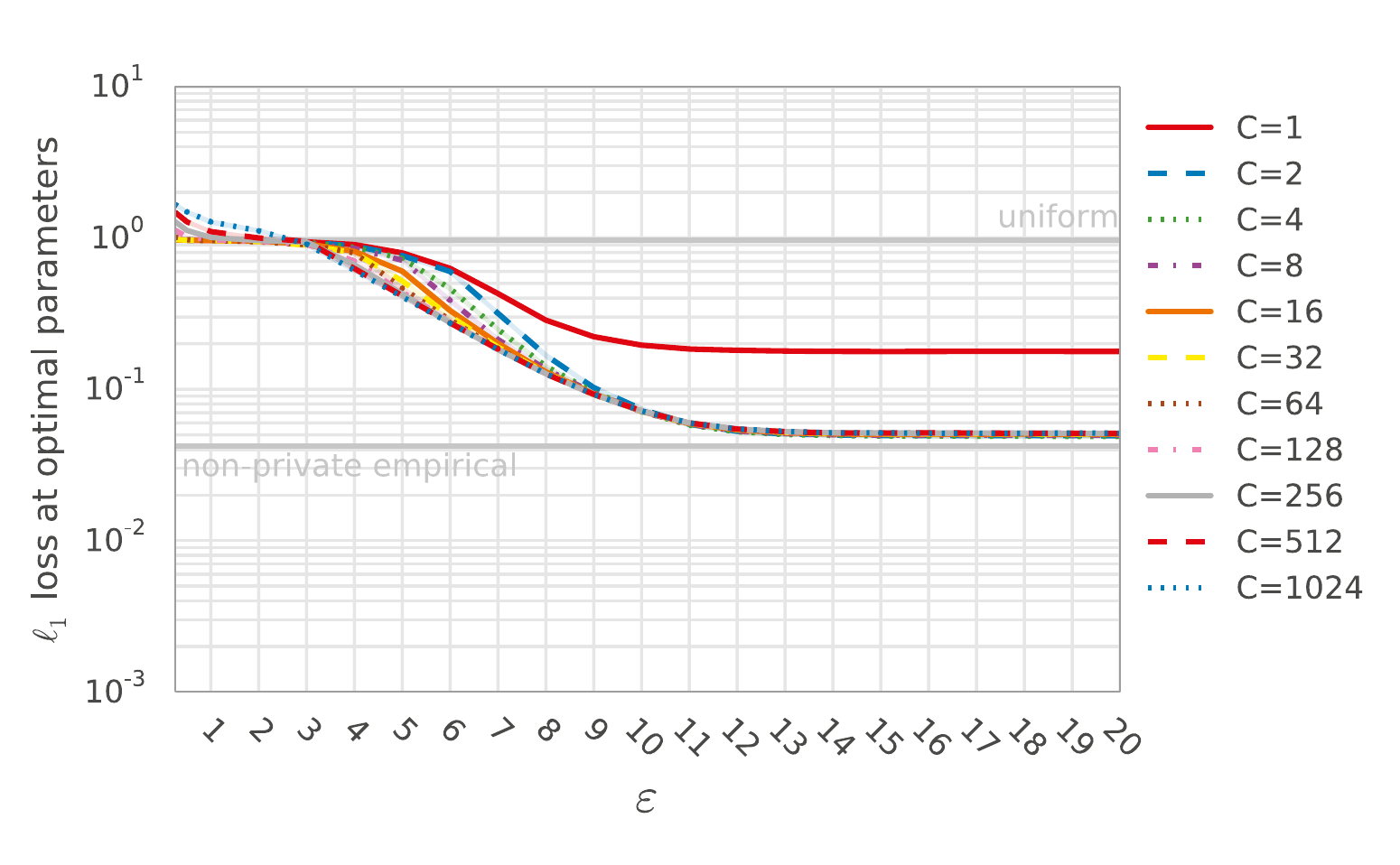}
\caption{\ORR varying $\CScalar$}
\end{subfigure}
&
\begin{subfigure}[b]{.45\linewidth}
\includegraphics[width=\linewidth]{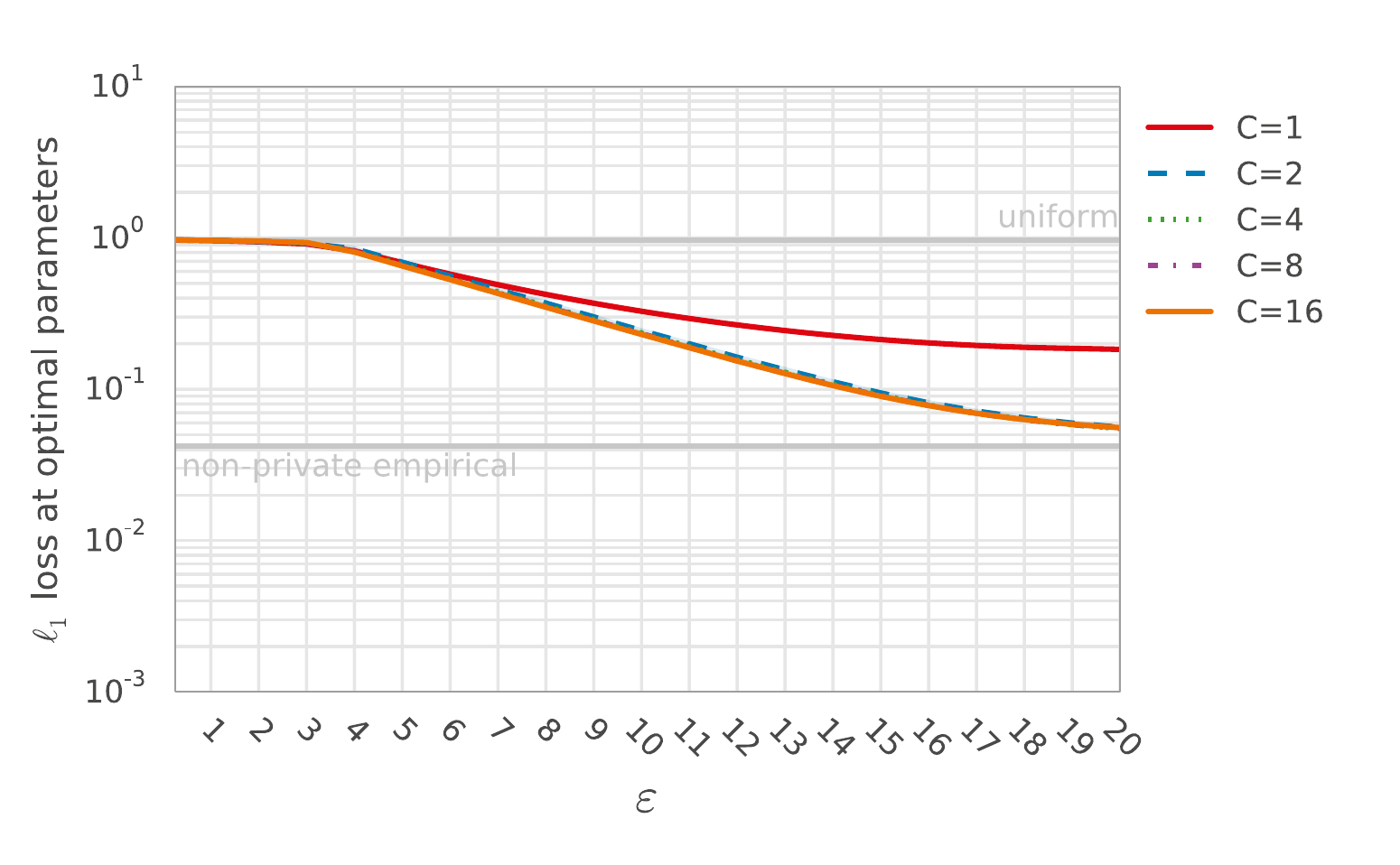}
\caption{\ORAPPOR varying $\CScalar$}
\end{subfigure}
\\
&
\begin{subfigure}[b]{.45\linewidth}
\includegraphics[width=\linewidth]{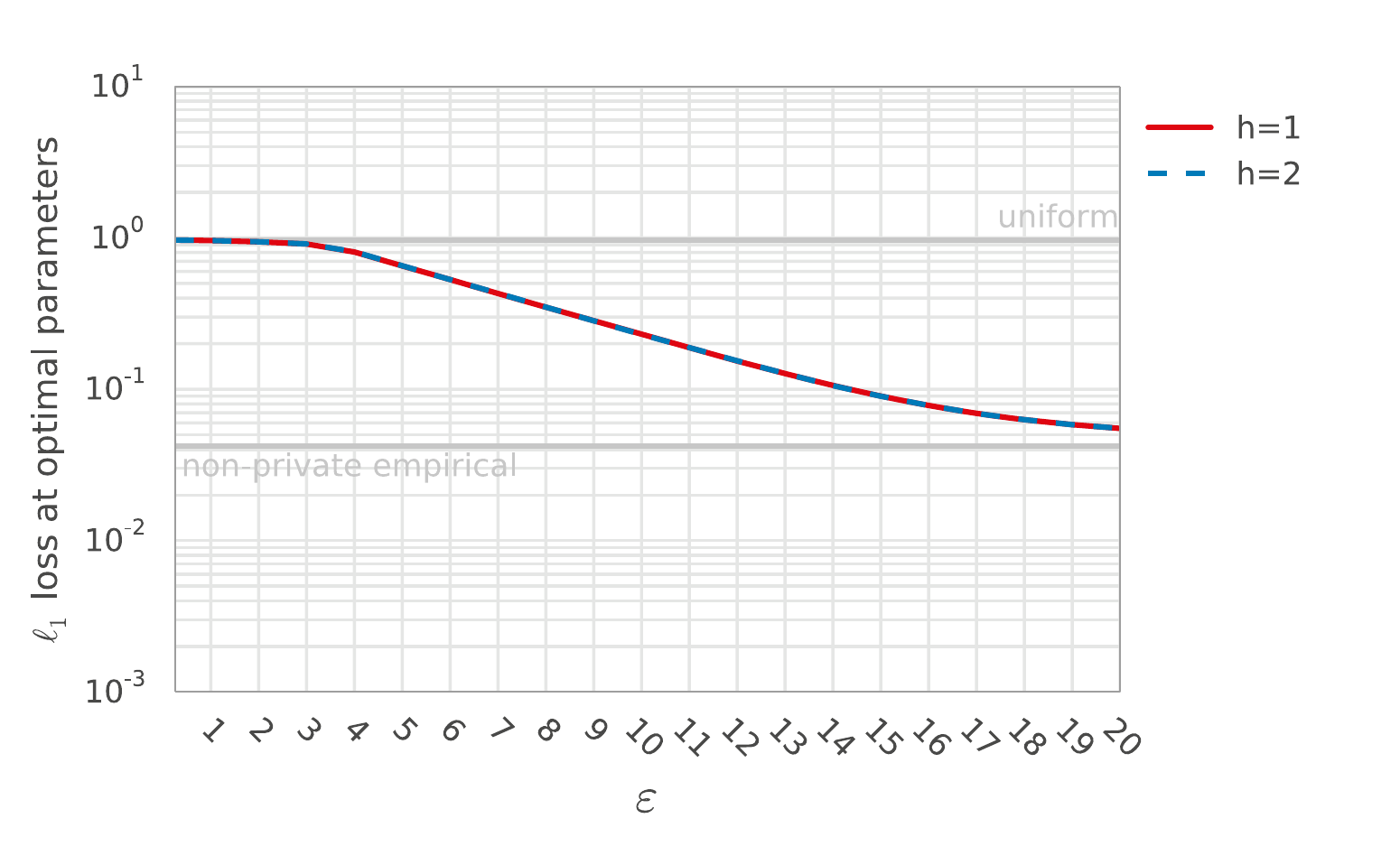}
\caption{\ORAPPOR varying $\hScalar$}
\end{subfigure}

\end{tabular}

\caption{$\ell_1$ loss when decoding open alphabets
using \ORR and \ORAPPOR under various parameter settings,
for $n=10^6$ users with input drawn from an
alphabet of $S=4096$ symbols under a geometric distribution with mean=$S/5$.
Remaining free parameters are set via grid search to minimize the
median loss over 50 samples at the given $\varepsilon$ and fixed parameter
values.  Lines show median $\ell_1$ loss while the (narrow) shaded regions indicate 90\%
confidence intervals (over 50 samples for the optimal parameter settings.)}
\label{fig:open_set_params_s4096}
\end{figure*}

\begin{figure*}
\centering

\begin{tabular}{ccc}

\begin{subfigure}[b]{.45\linewidth}
\includegraphics[width=\linewidth]{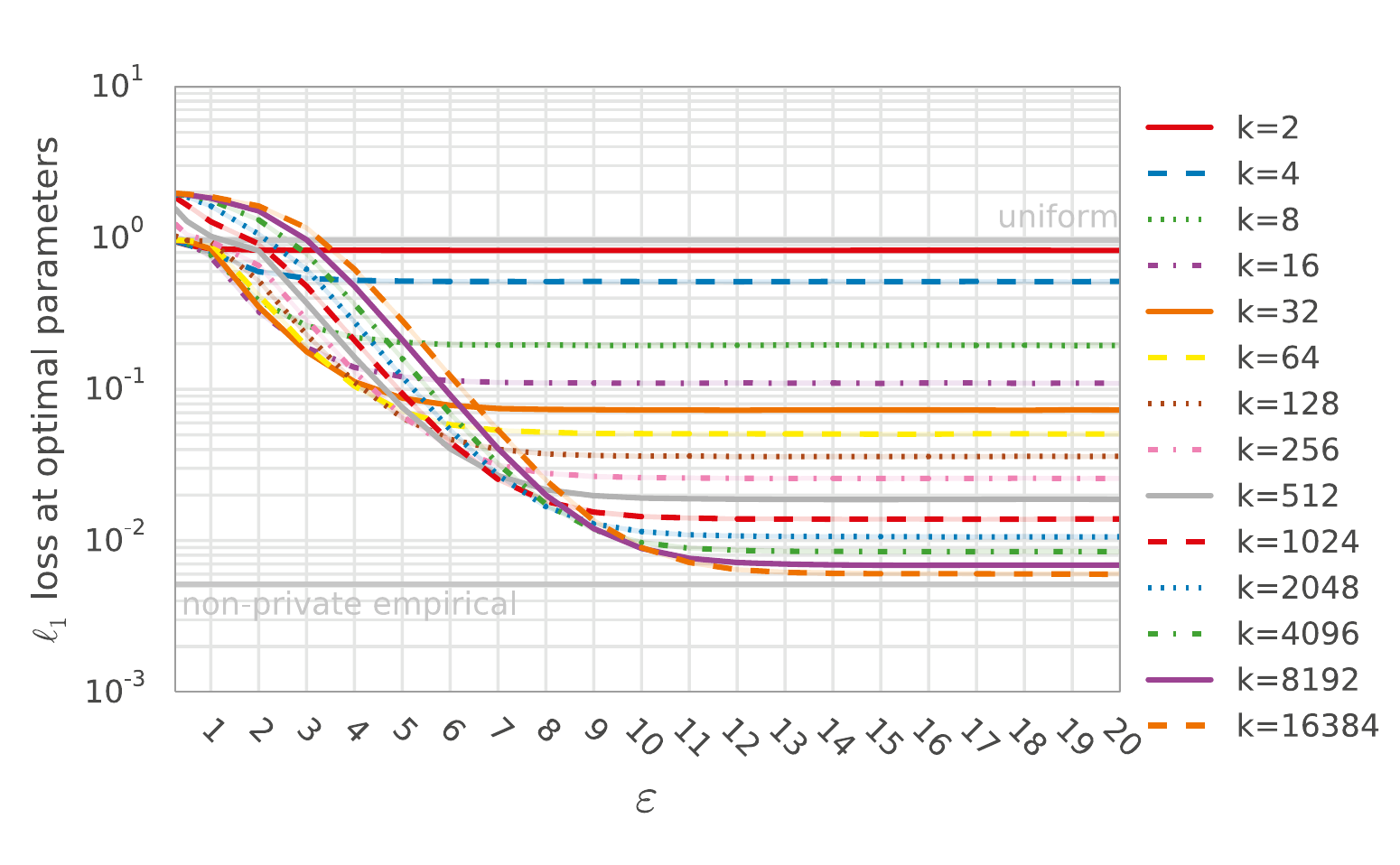}
\caption{\ORR varying $\kScalar$}
\end{subfigure}
&
\begin{subfigure}[b]{.45\linewidth}
\includegraphics[width=\linewidth]{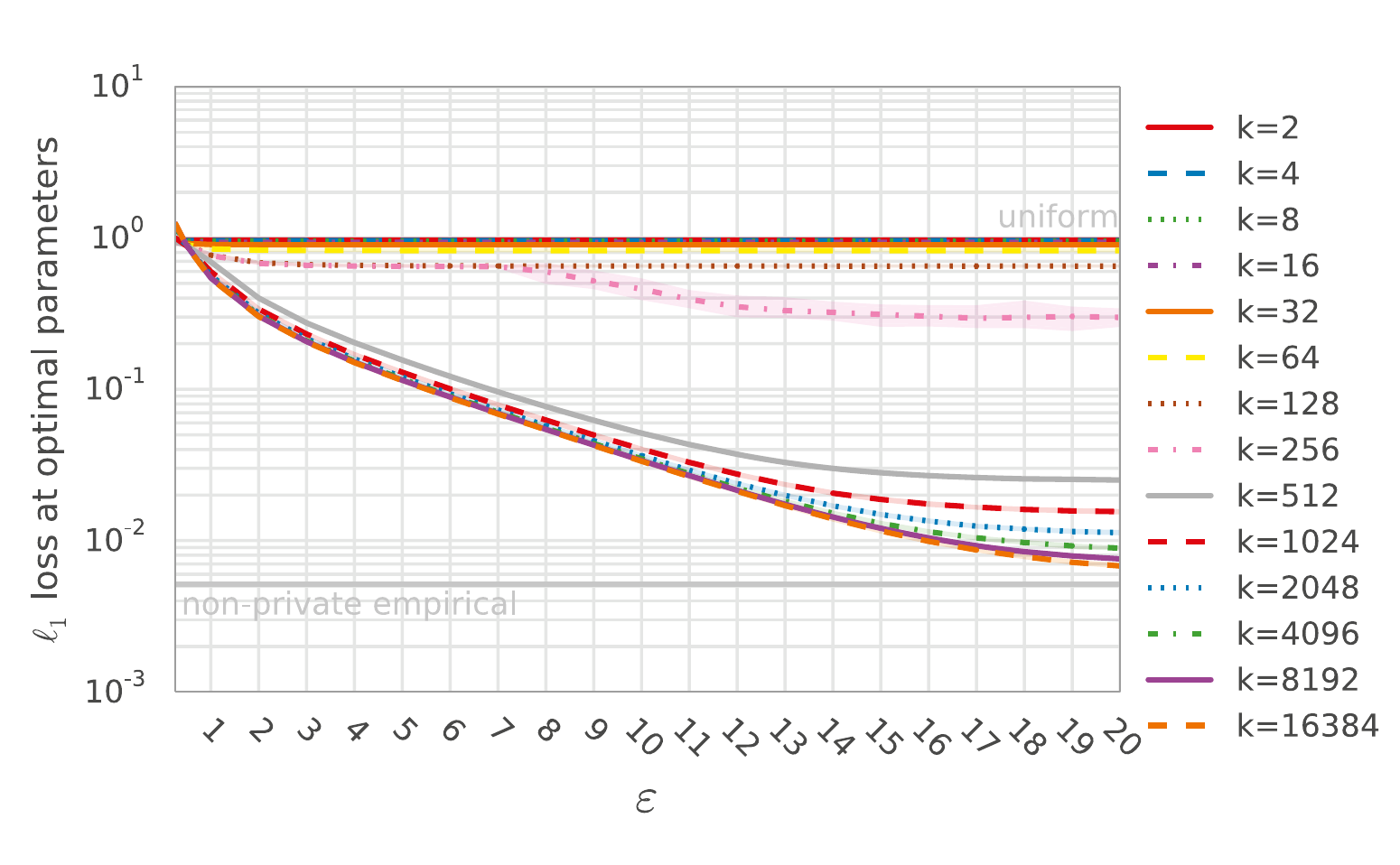}
\caption{\ORAPPOR varying $\kScalar$}
\end{subfigure}
\\
\begin{subfigure}[b]{.45\linewidth}
\includegraphics[width=\linewidth]{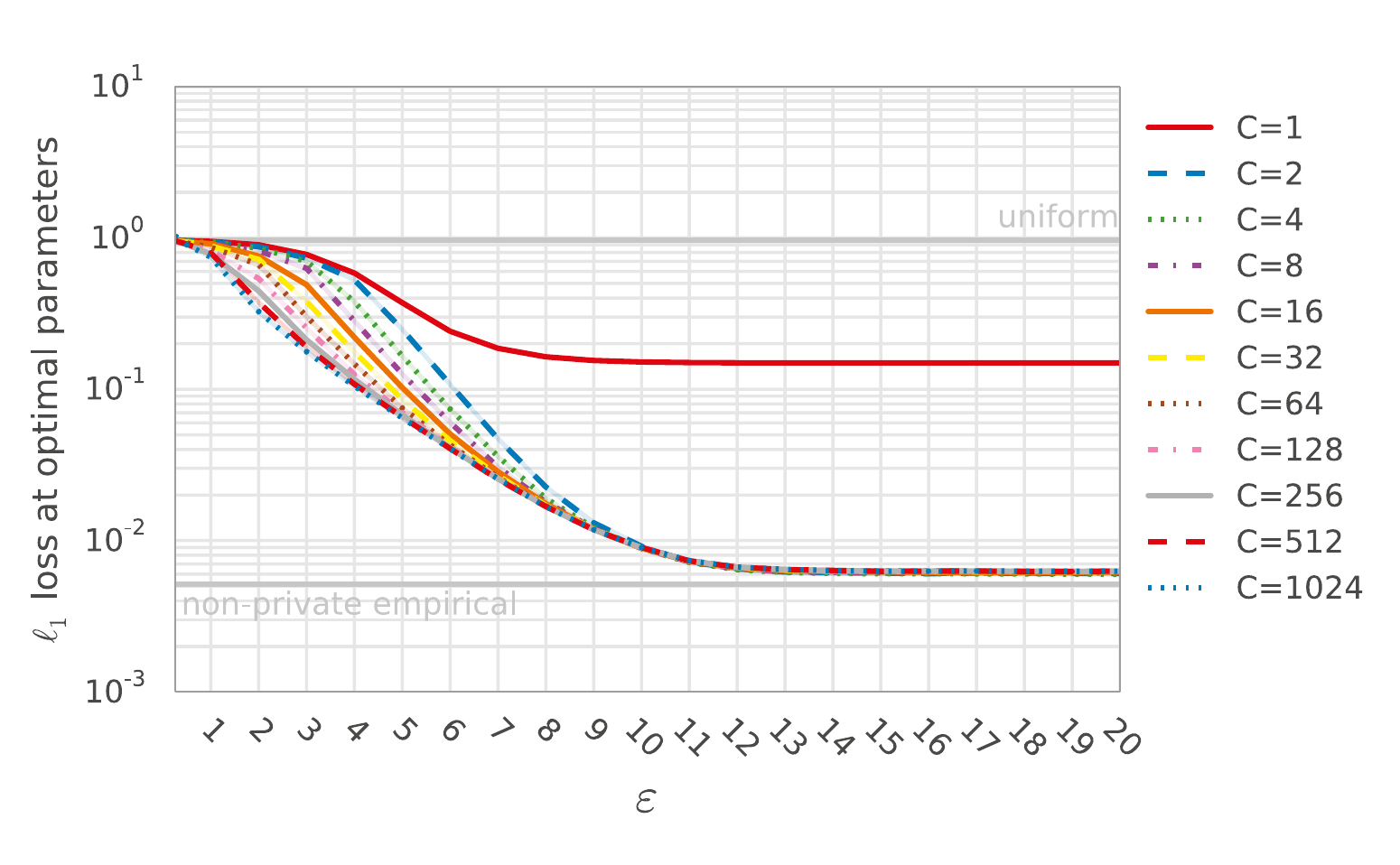}
\caption{\ORR varying $\CScalar$}
\end{subfigure}
&
\begin{subfigure}[b]{.45\linewidth}
\includegraphics[width=\linewidth]{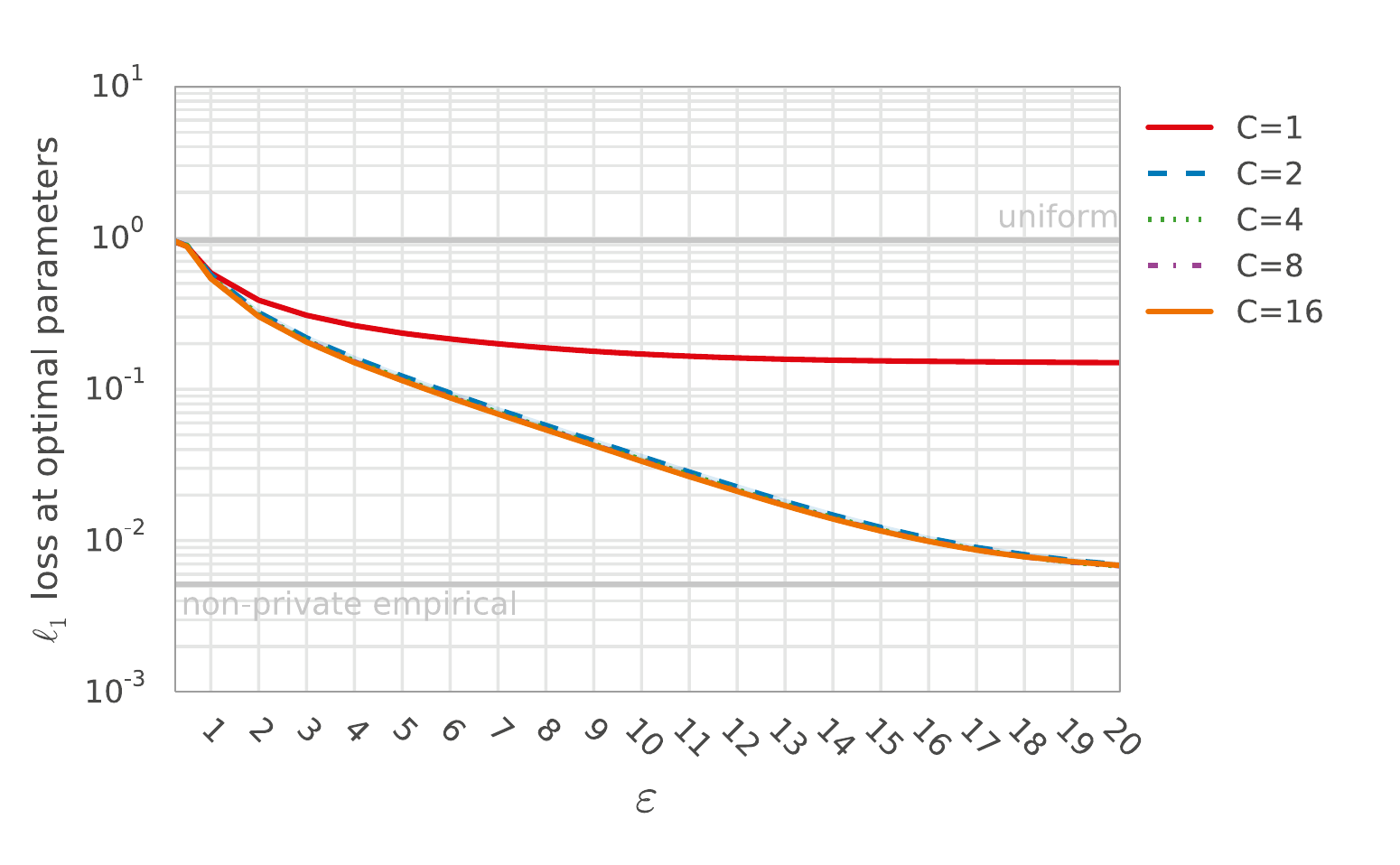}
\caption{\ORAPPOR varying $\CScalar$}
\end{subfigure}
\\
&
\begin{subfigure}[b]{.45\linewidth}
\includegraphics[width=\linewidth]{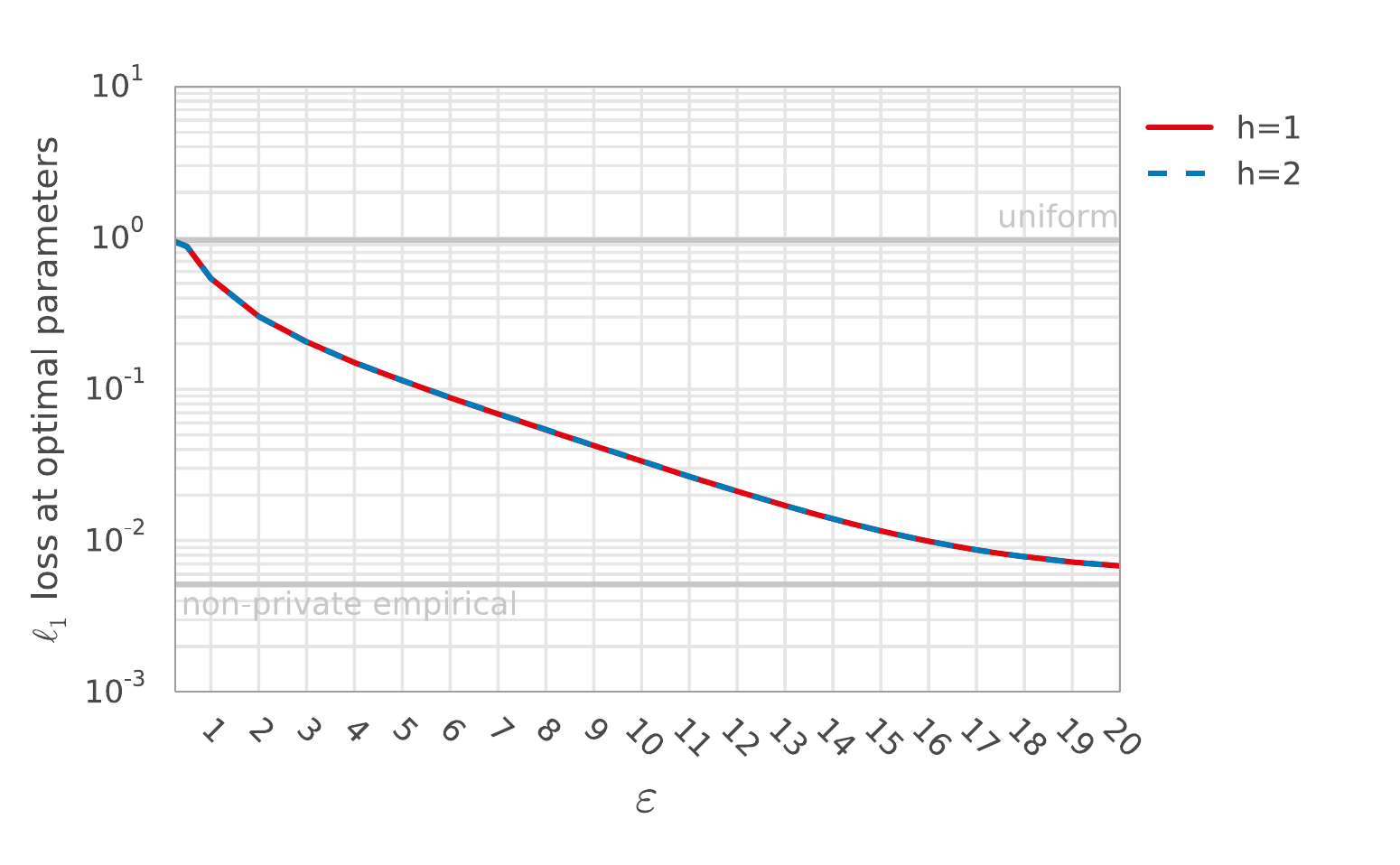}
\caption{\ORAPPOR varying $\hScalar$}
\end{subfigure}

\end{tabular}

\caption{$\ell_1$ loss when decoding open alphabets
using \ORR and \ORAPPOR under various parameter settings,
for $n=10^8$ users with input drawn from an
alphabet of $S=4096$ symbols under a geometric distribution with mean=$S/5$.
Remaining free parameters are set via grid search to minimize the
median loss over 50 samples at the given $\varepsilon$ and fixed parameter
values.  Lines show median $\ell_1$ loss while the (narrow) shaded regions indicate 90\%
confidence intervals (over 50 samples for the optimal parameter settings.)}
\label{fig:open_set_params_s4096_u1e8}
\end{figure*}

\end{document}